%% file: main.tex
\documentclass{article}

\usepackage[english]{babel}

\usepackage[scale=0.8]{geometry}

\usepackage{graphicx} 
\usepackage{url}

\usepackage{nicefrac}
\usepackage{amssymb,latexsym,amsfonts,amsmath,amsthm}
\usepackage{color}
\usepackage{enumerate}
\usepackage{enumitem}
\usepackage{algorithm}
\usepackage{algorithmic}
\usepackage{bm}
\usepackage{subfigure} 
\usepackage{mathtools}
\usepackage[titletoc,title]{appendix}
\usepackage[T1]{fontenc}
\usepackage{dsfont}
\usepackage{bbm}

\usepackage{url}
\usepackage[font=small,labelfont=bf]{caption}
\usepackage{hyperref}[] 
\usepackage{cleveref}
\usepackage{thm-restate}
\usepackage{placeins}
\usepackage{natbib}

\usepackage{minitoc}

\let\hat\widehat

\theoremstyle{plain}
\newtheorem{theorem}{Theorem}[section]
\newtheorem{proposition}[theorem]{Proposition}
\newtheorem{lemma}[theorem]{Lemma}
\newtheorem{corollary}[theorem]{Corollary}
\theoremstyle{definition}
\newtheorem{definition}[theorem]{Definition}
\newtheorem{assumption}[theorem]{Assumption}
\theoremstyle{remark}

\newtheorem*{theorem*}{Theorem}

\usepackage[textsize=tiny]{todonotes}

\usepackage{enumitem}
\setlist{leftmargin=2mm}

\newcommand{\R}{\mathbb{R}}
\newcommand{\E}{\mathbb{E}}
\newcommand{\cZ}{\mathcal{Z}}

\newcommand{\cS}{\mathcal{S}}

\newcommand{\LP}{\mathrm{LP}}

\newcommand{\cH}{\mathcal{H}}
\newcommand{\cF}{\mathcal{F}}
\newcommand{\calF}{\mathcal{F}}
\newcommand{\calH}{\mathcal{H}}

\newcommand{\ind}{\mathbb{I}}

\DeclareMathOperator*{\Median}{Median}

\newcommand{\calG}{\mathcal{G}}
\newcommand{\ellB}{\ell_B}
\newcommand{\Rhat}{\widehat{R}}
\newcommand{\Rad}{\mathfrak{R}}

\newcommand{\Prob}{\mathbb{P}} 
\newcommand{\Var}{\mathrm{Var}}

\usepackage[textsize=tiny]{todonotes}

\title{On the Generalization and Robustness in Conditional Value-at-Risk}
\author{Dinesh Karthik Mulumudi \\
Dept. of Mathematics \\
IISER, Pune.\\
\texttt{dineshkarthikforml@gmail.com }
\and
Piyushi Manupriya  \\
Dept. of Computer Science and Automation \\
Indian Institute of Sciecne, Bangalore.\\
\texttt{piyushim@iisc.ac.in}
\and
Gholamali Aminian\\
The Alan Turing Institute \\
Greater London, England, UK.\\
\texttt{gaminian@turing.ac.uk}
  \and
  Anant Raj \\
  Dept. of Computer Science and Automation \\
  Indian Institute of Sciecne, Bangalore. \\
  \texttt{anantraj@iisc.ac.in} 
}

\begin{document}

\maketitle
\doparttoc
\faketableofcontents

\begin{abstract}
Conditional Value-at-Risk (CVaR) is a widely used risk-sensitive objective for learning under rare but high-impact losses, yet its statistical behavior under heavy-tailed data remains poorly understood.
Unlike expectation-based risk, CVaR depends on an endogenous, data-dependent quantile, which couples tail averaging with threshold estimation and fundamentally alters both generalization and robustness properties. In this work, we develop a learning-theoretic analysis of CVaR-based empirical risk minimization under heavy-tailed and contaminated data.
We establish sharp, high-probability generalization and excess risk bounds under minimal moment assumptions, covering fixed hypotheses, finite and infinite classes, and extending to $\beta$-mixing dependent data; we further show that these rates are minimax optimal.
To capture the intrinsic quantile sensitivity of CVaR, we derive a uniform Bahadur-Kiefer type expansion that isolates a threshold-driven error term absent in mean-risk ERM and essential in heavy-tailed regimes.
We complement these results with robustness guarantees by proposing a truncated median-of-means CVaR estimator that achieves optimal rates under adversarial contamination.
Finally, we show that CVaR decisions themselves can be intrinsically unstable under heavy tails, establishing a fundamental limitation on decision robustness even when the population optimum is well separated. Together, our results provide a principled characterization of when CVaR learning generalizes and is robust, and when instability is unavoidable due to tail scarcity.
\end{abstract}

\section{Introduction}

Statistical Learning Theory (SLT) provides the standard framework for analyzing generalization in learning algorithms. Its central paradigm, Empirical Risk Minimization (ERM), selects a hypothesis $h\in\mathcal H$ by minimizing empirical risk, yielding generalization under suitable complexity control \cite{Shalev-Shwartz_Ben-David_2014}. Classical ERM, however, is risk-neutral and optimizes expected loss, which can be inadequate in high-stakes domains—such as finance, healthcare, and safety-critical learning—where rare but extreme losses dominate performance \cite{doi:10.1287/mnsc.18.7.356,6855488,pmlr-v89-cardoso19a}. 

These limitations have motivated the study of \emph{risk-sensitive} learning objectives that emphasize tail behavior, including Entropic Risk, Mean-Variance, and Conditional Value-at-Risk (CVaR).
From a learning-theoretic standpoint, several works have begun to investigate generalization guarantees for such objectives.
For example, generalization bounds for tilted empirical risk minimization under finite-moment assumptions were derived in \cite{aminian2025generalization}, while \cite{lee2021learningboundsrisksensitivelearning} studied CVaR and related risk measures under bounded-loss and light-tailed assumptions.
However, these settings do not fully capture the regimes in which CVaR is most relevant in practice, where data are often heavy-tailed and may exhibit temporal dependence.
This gap motivates a systematic study of CVaR generalization beyond bounded or light-tailed models, which we address in this work by developing guarantees under heavy-tailed losses for both i.i.d.\ and dependent data.

Another key challenge specific to CVaR-based learning is structural.
CVaR depends on an \emph{endogenous, data-dependent threshold}, the Value-at-Risk (VaR), which couples tail averages with empirical quantile estimation.
Unlike smooth risk functionals, this coupling creates an intrinsic interaction between tail fluctuations and threshold estimation error.
As a consequence, standard generalization analyses that treat the objective as a fixed functional of the loss distribution are insufficient.
Controlling CVaR generalization therefore requires explicitly accounting for how fluctuations of the empirical quantile propagate into fluctuations of the tail risk.
To address this issue, we develop a refined empirical-process analysis based on \emph{uniform Bahadur-Kiefer–type expansions}
\cite{bahadur1966note,kiefer1967bahadur}, adapted to the CVaR setting to capture the effects of this endogenous threshold.

Robustness presents a closely related challenge and arises at multiple levels.
At the \emph{functional} level, although CVaR is often viewed as robust due to its emphasis on tail outcomes, its dependence on an empirical quantile can amplify sensitivity to perturbations near the tail.
At the \emph{estimator} level, robustness under heavy-tailed data and adversarial contamination has been extensively studied through trimmed means and median-of-means estimators
\cite{lugosi2019mean,lugosi2019risk,lugosi2021robust},
yielding sharp guarantees for estimation primitives and robust ERM, but not directly for tail-based risk functionals such as CVaR.
At the \emph{decision} level, the stability of CVaR-optimal solutions under distributional perturbations remains comparatively less understood.

In this work, we study both the \emph{generalization} and \emph{robustness} properties of CVaR-based learning under heavy-tailed and contaminated data.
On the generalization side, we derive non-asymptotic excess risk guarantees under minimal moment assumptions and complement them with refined empirical-process tools that explicitly account for quantile sensitivity.
On the robustness side, we analyze the stability of CVaR objectives and solutions under distributional perturbations.
Together, these results clarify the statistical regimes in which empirical CVaR reliably approximates its population counterpart, as well as settings in which instability is intrinsic.

\vspace{2mm}
\noindent\textbf{Contributions:} We make following contributions in this paper: 
\begin{itemize}
\item \textbf{Sharp generalization theory for heavy-tailed CVaR (Section~\ref{sec:iid} and \ref{sec:depen}).}
We derive high-probability generalization and excess risk bounds for empirical CVaR minimization under minimal $(\lambda+1)$-moment assumptions ($0<\lambda\le1$), covering fixed hypotheses and finite and infinite classes via VC/pseudo-dimension and Rademacher complexity. The bounds explicitly characterize the dependence on $\alpha$, tail exponent, hypothesis complexity, and sample size, extend CVaR theory beyond bounded and sub-Gaussian regimes, and are shown to be minimax optimal. We further extend the theory to $\beta$-mixing dependent data with matching lower bounds.

\item \textbf{Uniform Bahadur-Kiefer expansions for CVaR (Section~\ref{sec:bahadur}).}
We establish the first uniform Bahadur-Kiefer–type expansion for CVaR that explicitly accounts for its endogenous and data-dependent threshold. The result captures the nonstandard coupling between empirical quantile fluctuations and tail averaging, yielding a sharp decomposition of CVaR error into a classical empirical process term and an additional threshold-driven component governed by local tail geometry. This identifies a second-order source of generalization error absent in mean-risk ERM and is essential in heavy-tailed regimes.

\item \textbf{Robust CVaR-ERM under adversarial contamination (Section~\ref{sec:estimator}).}
We propose a robust CVaR-ERM estimator based on truncation and median-of-means aggregation applied to the Rockafellar-Uryasev lift, jointly robustifying estimation over the decision variable and the endogenous threshold. We prove non-asymptotic excess risk guarantees under heavy-tailed losses and oblivious adversarial contamination, with rates that optimally decompose into a statistical term and an unavoidable contamination-dependent term. This provides the first learning-theoretic robustness guarantees for CVaR-ERM without bounded-loss assumptions.

\item \textbf{Fundamental limits of decision robustness (Section~\ref{sec:decision}).}
We analyze the stability of CVaR decisions themselves and show that, unlike mean-risk ERM, CVaR minimization can be intrinsically unstable under heavy-tailed losses. Even with a unique and well-separated population minimizer, a single observation can flip the empirical CVaR-ERM decision with polynomially small but unavoidable probability. This establishes a sharp impossibility result for decision robustness in tail-scarce regimes.

\end{itemize}

\subsection{Related Work}
\vspace{-2mm}
Optimized certainty equivalents include expectation, CVaR, entropic, and mean-variance risks. 
Early generalization bounds were obtained under bounded-loss assumptions \cite{lee2021learningboundsrisksensitivelearning}, and later extended to unbounded losses via tilted ERM under bounded $(1+\lambda)$-moment conditions, achieving rates $\mathcal{O}(n^{-\lambda/(1+\lambda)})$ \cite{li2021tiltedempiricalriskminimization}.For CVaR, concentration inequalities under light- and heavy-tailed losses were established in \cite{kolla2019concentration,prashanth2019concentration} under a strictly increasing CDF assumption, but learning-theoretic generalization and robustness guarantees in the genuinely heavy-tailed regime remain largely open.

Our fixed-hypothesis analysis builds on truncation-based techniques \cite{brownlees2015erm,prashanth2019concentration}, with extensions to finite and infinite hypothesis classes via standard learning-theoretic tools \cite{Shalev-Shwartz_Ben-David_2014}, and matching minimax lower bounds obtained through classical constructions \cite{10.5555/1522486}. 
Related work on CVaR and spectral risk learning under heavy tails \cite{holland2021learning,holland2022spectral} is restricted to finite-variance i.i.d.\ settings.
Results on robust mean estimation under heavy tails and contamination \cite{laforgue2021generalization,de2025optimality} focus on estimation primitives rather than tail-based risks, while learning under heavy-tailed dependence has primarily addressed expected-risk objectives \cite{roy2021empirical,shen2026sgd}. Minimax guarantees for quantile-based risks largely assume light-tailed regimes \cite{el2024minimax}.

The asymptotic behavior of sample quantiles is classically characterized by the Bahadur and Bahadur-Kiefer representations.
Bahadur~\cite{bahadur1966note} and Kiefer~\cite{kiefer1967bahadur} established linear and uniform expansions that form a foundation of empirical process theory \cite{van2000asymptotic}, but rely on smoothness and light-tail assumptions.
These results do not directly apply to CVaR, where the quantile is endogenous and coupled with tail averages.
Our work develops a uniform Bahadur-Kiefer expansion tailored to CVaR, explicitly accounting for the random active set induced by the data-dependent threshold and allowing for heavy-tailed losses.

On the robustness side, robust ERM under weak moment assumptions has been studied in \cite{10.1093/imaiai/iaab004}.
Trimmed-mean and median-of-means estimators \cite{lugosi2019mean,lugosi2019risk} admit finite-sample and minimax-optimal guarantees \cite{oliveira2025finite,oliveira2025trimmed,lugosi2021robust}, but do not directly address tail-based risk functionals such as CVaR.

\section{Background and Problem Formulation}
\subsection{Conditonal Value at Risk (CVaR)}
    Let $L$ be a real-valued random variable representing the loss of a decision or portfolio. 
    For a tail probability $\alpha \in (0,1)$, the \emph{Value-at-Risk} at level $\alpha$ is defined as the upper $(1-\alpha)$-quantile of $L$:
    \[
    \mathrm{VaR}_{\alpha}(L)
    \;:=\;
    \inf\{t \in \mathbb{R} : \mathbb{P}(L \le t) \ge 1-\alpha\}.
    \]
    Equivalently, $ \mathbb{P}\!\left(L > \mathrm{VaR}_{\alpha}(L)\right) \le \alpha.$ The \emph{Conditional Value-at-Risk} at level $\alpha$ is defined as the expected loss in the $\alpha$-tail of the distribution:
    \[
    \mathrm{CVaR}_{\alpha}(L)
    \;:=\;
    \mathbb{E}\!\left[L \mid L \ge \mathrm{VaR}_{\alpha}(L)\right],
    \]
    whenever the conditional expectation is well-defined. More generally, CVaR admits the variational representation \cite{Rockafellar2000OptimizationOC}:
    \[
    \mathrm{CVaR}_{\alpha}(L)
    \;=\;
    \inf_{t \in \mathbb{R}}
    \left\{
    t + \frac{1}{\alpha}\,\mathbb{E}\big[(L - t)_+\big]
    \right\},
    \]
    which remains valid without continuity assumptions and is particularly convenient for learning and optimization.

    Unlike VaR, CVaR is a coherent risk measure: it is monotone, translation invariant, positively homogeneous, and subadditive. In particular, CVaR is convex in the underlying loss distribution, which makes it amenable to empirical risk minimization and generalization analysis.

\subsection{Setup, Notation and Assumptions} \label{sec:notations}
	Let $\mathcal{H}$ be a hypothesis class and let $\ell(h, x) \geq 0$ be a non-negative loss of hypothesis $h \in \mathcal{H}$ on example $x$. Let $Z_1, \ldots, Z_n$ be i.i.d. samples drawn from a distribution $P$. We define the Rockerfellar-Uryasev (RU) population and empirical objective as follows:
    \begin{align*}
        \phi_P(h,\theta) = \theta + \frac{1}{\alpha} \mathbb{E}_P[(\ell(h,Z) - \theta)_+] ~~~\text{and} ~~ \phi_{P_n}(h,\theta) = \theta + \frac{1}{n\alpha} \sum_{i=1}^n [(\ell(h,Z_i) - \theta)_+],
    \end{align*}
    where $P_n = \frac{1}{n}\sum_{i=1}^n \delta_{Z_i}$ and $\{Z_1,\cdots Z_n\}$ are i.i.d samples from $P$. Define the population Conditional Value-at-Risk (CVaR) at level $\alpha \in (0,1)$ by
	\begin{equation}
    	R_\alpha^P(h) = \inf_{\theta \in \mathbb{R}} \left\{ \theta + \frac{1}{\alpha} \mathbb{E} \left[ \left( \ell(h, Z_1) - \theta \right)_{+} \right] \right\} 
	\end{equation}
	and its empirical counterpart
    \begin{equation}
    \widehat{R}_\alpha^P(h) = \inf_{\theta \in \mathbb{R}} \left\{ \theta + \frac{1}{\alpha n} \sum_{i=1}^n \left( \ell(h, Z_i) - \theta \right)_{+} \right\}    
    \end{equation}
	where $\left( e \right)_{+} \equiv \max \left\{ e, 0 \right\}$. \cite{rockafellar2000conditional}
    
We omit the superscript denoting the underlying distribution whenever it is clear from context. Our objective is to learn the population CVaR minimizer $h^* \in \arg\min_{h \in \mathcal{H}} R_\alpha(h)$, and we denote by $\hat h \in \arg\min_{h \in \mathcal{H}} \widehat R_\alpha(h)$
the empirical CVaR minimizer.

    \begin{assumption}[Bounded \((1+\lambda)\)-th Moment]\label{asm:moment}
		There exists a constant \( M \in \mathbb{R}^+ \) and \( \lambda \in (0,1) \) such that the loss function \( (h, Z) \mapsto \ell(h, Z) \) satisfies:
	$\mathbb{E} \left[ \left( \ell(h, Z) \right)^{1+\lambda} \right] \leq M$, 
		uniformly for all \( h \in \mathcal{H} \), where \( \mathcal{Z} = \mathcal{X} \times \mathcal{Y} \) is the instance space.
	\end{assumption}
    
    \begin{assumption}[$\beta$-mixing dependence]
    \label{ass:beta-mixing}
    The process $\{Z_i\}_{i\in\mathbb{Z}}$ is strictly stationary and $\beta$-mixing with exponentially decaying coefficients: there exist constants $c_0,c_1,\gamma>0$ such that
    \[
    \beta(k)
    :=
    \sup_{t\in\mathbb{Z}}
    \beta\bigl(\sigma(Z_{-\infty}^t),\sigma(Z_{t+k}^\infty)\bigr)
    \le c_0 e^{-c_1 k^\gamma},
    \qquad k\ge1.
    \]
    \end{assumption}
    
    \begin{assumption}[Hypothesis class complexity]
    \label{ass:complexity_D}
    The hypothesis class $\mathcal{H}$ has finite pseudo-dimension:
    \[
    \mathrm{Pdim}(\mathcal{H}) \le d < \infty.
    \]
    \end{assumption}
    	\begin{assumption}[Statistical Complexity of Truncated Class]
		\label{ass:complexity}
		For any truncation level $B > 0$, define the truncated function class:
		$\calF_B = \left\{ \min(\phi(\cdot; h, \theta), B) : h \in \calH, \theta \in \Theta \right\}.$
		We assume that for any probability measure $Q$, the uniform covering number satisfies:
		\[
		\log \mathcal{N}(\calF_B, \|\cdot\|_{L_2(Q)}, u) \le d \log\left(\frac{C_0 B}{u}\right), \quad \forall u \in (0, B].
		\]
	\end{assumption}

\section{Generalization in CVaR}

This section establishes non-asymptotic generalization guarantees for Conditional Value-at-Risk (CVaR) under heavy-tailed losses, ranging from fixed-hypothesis concentration to sharp excess risk bounds for infinite classes, with matching minimax lower bounds and extensions to dependent data. We further develop a random active-set theory for CVaR and derive uniform Bahadur–Kiefer expansions that explicitly capture the role of the endogenous threshold.

The core difficulty is structural: CVaR depends on tail behavior through a data-dependent quantile, requiring joint control of tail fluctuations and threshold instability.

\subsection{Generalization with i.i.d. Data}\label{sec:iid}

We begin with the i.i.d. setting. The first step is to understand how the empirical CVaR concentrates around its population counterpart for a fixed hypothesis. 

The key technical ingredient is a concentration inequality for heavy-tailed random variables with finite $(1+\lambda)$-moment, as developed in \cite{brownlees2015erm, prashanth2019concentration}. Applying these tools to the Rockafellar-Uryasev variational representation of CVaR yields the following fixed-hypothesis deviation bound.

\begin{theorem}\label{thm:fixed and finite}
Under Assumption \ref{asm:moment} and for any fixed $h \in \cH$ and $\epsilon > 0$ with probability at least $1 - \delta$ we,
\begin{equation} \label{prop fixed lower}
\widehat{R}_{\alpha}(h) - R_{\alpha}(h) \leq \frac{2}{\alpha} M^{\frac{1}{1+\lambda}} \left( \frac{\log(2/\delta)}{n} \right)^{\frac{\lambda}{1+\lambda}}.
\end{equation}
For a finite $\mathcal{H}$, we extend the bound to all hypotheses using the union bound.
\begin{equation} \label{finite Excess Risk Bound}
R_{\alpha}(\widehat{h}) - R_{\alpha}(h^*) \leq \frac{4}{\alpha} M^{\frac{1}{1+\lambda}} \left( \frac{\log(4|\mathcal{H}|/\delta)}{n} \right)^{\frac{\lambda}{1+\lambda}}.
\end{equation}
\end{theorem}

While the finite-class bound follows directly from a union bound, it does not capture the true statistical complexity of learning when $\mathcal{H}$ is infinite. We therefore turn to localized complexity arguments that yield dimension-dependent rates without incurring unnecessary logarithmic penalties.

For a fixed natural number $n-1$, consider the space $\{0,1\}^{n}$ endowed with the Hamming metric.  Let $N$ denote its packing number.

\begin{theorem}[Localized VC Bounds for Heavy-Tailed CVaR] \label{thm:vc}
Assume $\mathcal H$ has pseudo-dimension $d<\infty$ and Assumption~\ref{asm:moment} holds. Then for any $\delta\in(0,1)$, with probability at least $1-\delta$,
\begin{equation}\label{eq:inf iid}
R_\alpha(\widehat h)-R_\alpha(h^*) \;\le\; \frac{C_\lambda}{\alpha}\, M^{\frac{1}{1+\lambda}} \left( \frac{d\log n+\log(1/\delta)}{n} \right)^{\frac{\lambda}{1+\lambda}}
\end{equation}
where $C_\lambda>0$ depends only on $\lambda$. Moreover, this rate is minimax optimal. There exist universal constants $c,C>0$ (depending only on $\lambda$) such that for any $\alpha\in(0,1)$, $M>0$, $\lambda\in(0,1]$, and $n\ge C\,(\log N)/\alpha$,
\small
\begin{align}
&\inf_A \sup_{P\in\mathcal P(M,\lambda)} \mathbb E\!\left[ R_\alpha(A(S);P)-R_\alpha(h_P^*;P) \right]   \;\ge\; c\,\frac{M^{1/(1+\lambda)}}{\alpha} \left(\frac{\log N}{n}\right)^{\frac{\lambda}{1+\lambda}}. \label{eq:minimax lower}
\end{align}
\normalsize
\end{theorem}

\noindent \emph{\textbf{Discussion:}}
This theorem shows that empirical CVaR minimization achieves the optimal tradeoff between model complexity, sample size, and tail heaviness. Importantly, the rate matches the minimax lower bound up to constants, demonstrating that no improvement is possible without stronger assumptions.

\paragraph{Truncated Empirical CVaR Estimation.}
Although empirical CVaR minimization is statistically optimal, finite-sample instability can arise from extreme losses. We therefore introduce a truncated CVaR estimator that limits the influence of large observations while preserving optimal rates and remaining amenable to first-order optimization. Truncation simultaneously stabilizes estimation and enables standard complexity-based analysis, with the induced bias controlled under finite-moment assumptions.

\begin{theorem}[High-Probability Generalization Bound]\label{thm:inf_bound}
Suppose Assumption \ref{asm:moment} holds. Let $\Rad_n(\ell \circ \calH)$ be the Rademacher complexity of the loss class. Fix $\delta \in (0,1)$. If we set the truncation level $B = (Mn)^{\frac{1}{1+\lambda}}$, then with probability at least $1-\delta$:
\begin{equation}
R_\alpha(\hat{h}_B) - R_\alpha(h^*) \;\leq\; \frac{4}{\alpha} \Rad_n(\ell \circ \calH) + \frac{C_{\lambda, \delta}}{\alpha} M^{\frac{1}{1+\lambda}} \left( \frac{1}{n} \right)^{\frac{\lambda}{1+\lambda}},
\end{equation}
where the constant $C_{\lambda, \delta}$ is given by:
\begin{equation}
C_{\lambda, \delta} = \frac{1}{\lambda} + \sqrt{8\log(1/\delta)} + \frac{2}{3}\log(1/\delta).
\end{equation}
\end{theorem}

\paragraph{Takeaway:}
\emph{Truncated empirical CVaR achieves optimal high-probability excess risk bounds while offering improved robustness to heavy-tailed noise.}

\subsection{Generalization with Dependent Data} \label{sec:depen}

We now extend the analysis to dependent observations. Let $\{Z_i\}_{i\in\mathbb{Z}}$ be a strictly stationary stochastic process, and suppose only a finite segment $Z_1,\dots,Z_n$ is observed. Dependence reduces the effective sample size and requires refined concentration arguments.

\begin{theorem}[Upper Bound for CVaR under $\beta$-Mixing Heavy-Tailed Data]\label{thm:upper-bound_D}
Under Assumptions~\ref{asm:moment}-\ref{ass:C_local_stability} Then for any $\delta\in(0,1)$, with probability at least $1-\delta - n^{-O(1)}$,
\[
\sup_{h\in\mathcal{H}} \bigl| \widehat{R}_\alpha(h) - R_\alpha(h) \bigr| \le \frac{C_{\lambda,\gamma}}{\alpha} \, M^{\frac{1}{1+\lambda}} \Bigl( \frac{d\log n + \log(1/\delta)}{n/\log n} \Bigr)^{\frac{\lambda}{1+\lambda}},
\]
and consequently, for the empirical CVaR minimizer $\hat{h}$,
\[
R_\alpha(\hat{h}) - R_\alpha(h^*) \le \frac{2C_{\lambda,\gamma}}{\alpha} \, M^{\frac{1}{1+\lambda}} \Bigl( \frac{d\log n + \log(1/\delta)}{n/\log n} \Bigr)^{\frac{\lambda}{1+\lambda}},
\]
where $C_{\lambda,\gamma}>0$ depends only on $\lambda,\gamma$ and the mixing constants.
\\
We also establish \textbf{Minimax Lower Bound.}
\label{thm:lower-bound_D}
Let $\mathcal{P}(\lambda,M,\beta)$ be the class of strictly stationary $\beta$-mixing distributions with $\beta(k)\le c_0 e^{-c_1 k^\gamma}$ and $\sup_{h\in\mathcal{H}}\mathbb{E}[\ell(h,Z)^{1+\lambda}]\le M$. Assume $\mathrm{VCdim}(\mathcal{H})\ge d\ge 1$. Then for any $\alpha\in(0,1)$ and $n$ satisfying $n \ge C_0 d/\alpha$,
\[
\inf_{\hat{h}}\sup_{P\in\mathcal{P}} \mathbb{E}_P\bigl[ R_\alpha(\hat{h}) - R_\alpha(h^*_P) \bigr]
\ge \frac{c}{\alpha} \, M^{\frac{1}{1+\lambda}} \Bigl( \frac{d}{n/\log n} \Bigr)^{\frac{\lambda}{1+\lambda}},
\]
where $C_0,c>0$ depend only on $\lambda,\gamma$ and the mixing constants.
\end{theorem}

\emph{Discussion.}
Dependence introduces an unavoidable logarithmic degradation in the effective sample size, but the fundamental tail-dependent rate remains unchanged. The matching lower bound shows this loss is information-theoretically necessary.

\subsection{Random active-set theory and uniform Bahadur - Kiefer expansions for CVaR} \label{sec:bahadur}
While the preceding results bound uniform CVaR deviations under heavy tails, they do not explain how these fluctuations affect the empirical minimizer. The challenge is structural: CVaR depends on an \emph{endogenous} quantile, coupling threshold estimation with tail averaging. We make this coupling explicit via uniform Bahadur–Kiefer expansions for CVaR-ERM, revealing two distinct error sources—one from tail moments and another from the local geometry at the CVaR threshold.
For, fixed $\alpha\in(0,1)$. For $h\in\cH$ and $\theta\in\R$, we write $X_h:=\ell(h,Z)$ and $X_{h,i}:=\ell(h,Z_i)$ and we define the population RU threshold
\begin{equation}
	\theta^\star(h)
	:=\inf\{\theta\in\R:\ P(X_h>\theta)\le \alpha\}, \notag
\end{equation}
and let the empirical RU threshold be the \emph{minimal} empirical minimizer
\small
\begin{equation}
	\hat\theta_n(h)
	:=\inf\{\theta\in\R:\ P_n(X_h>\theta)\le \alpha\},
	\quad
	P_n:=\frac1n\sum_{i=1}^n\delta_{Z_i}. \notag
\end{equation} \normalsize
By convexity of $\theta\mapsto \widehat\Phi_n(h,\theta)$, $\hat\theta_n(h)$ is always a minimizer of
$\widehat\Phi_n(h,\cdot)$. We also assume that for all $h\in\mathcal H$, $P(X_h>\theta^\star(h))=\alpha$ (equivalently, $P(X_h=\theta^\star(h))=0$). However, before stating the final theorem statement, we would like to state some assumptions:
\begin{assumption}[Uniform tail-indicator deviation]
	\label{ass:B3_E1}
	There exist constants $V\ge 1$ and $C_0>0$ such that for all $\delta\in(0,1)$, with probability at least $1-\delta$,
	\begin{align*}
		\sup_{h\in\mathcal H}\sup_{t\in\R}\Big|P_n(X_h>\theta)-P(X_h>\theta)\Big|
		\ \le\
		\varepsilon_n(\delta)  
		:=C_0\sqrt{\frac{V\log(en)+\log(2/\delta)}{n}}.
	\end{align*}
\end{assumption}

\begin{assumption}[Two-sided local quantile margin]
	\label{ass:B3_margin_local_}
	There exist constants $\kappa\ge 1$, $c_-,c_+>0$, and $u_0>0$ such that for all $h\in\mathcal H$ and all $u\in[0,u_0]$,
	\begin{equation}
		c_-u^\kappa
		\le
		P(X_h>\theta^\star(h)-u)-P(X_h>\theta^\star(h))
		\le
		c_+u^\kappa,
		\label{eq:B3_margin_left}
	\end{equation}
	\begin{equation}
		c_-u^\kappa
		\le
		P(X_h>\theta^\star(h))-P(X_h>\theta^\star(h)+u)
		\le
		c_+u^\kappa.
		\label{eq:B3_margin_right}
	\end{equation}
\end{assumption}

\begin{assumption}[Uniform heavy-tail deviation for the hinge at $t^\star$]
	\label{ass:B3_E2}
	There exists a function $\eta_n(\delta)$ such that for all $\delta\in(0,1)$, with probability at least $1-\delta$,
	\begin{equation}
		\sup_{h\in\mathcal H}\Big|
		(P_n-P)\big[(X_h-\theta^\star(h))_+\big]
		\Big|\ \le\ \eta_n(\delta).
		\label{eq:B3_E2}
	\end{equation}
\end{assumption}

\paragraph{Discussion on Assumptions:} Assumptions~\ref{ass:B3_E1}–\ref{ass:B3_E2} are mild and hold for broad model classes. Uniform control of tail indicators holds when the class ${(z,\theta) \mapsto!\mathbf 1{\ell(h,z)>\theta}}$ has finite VC dimension or polynomial entropy, covering linear models with Lipschitz losses, GLMs, and bounded-norm neural networks, yielding $\varepsilon_n(\delta)!\asymp!\sqrt{(\log n)/n}$. Regular quantiles require only local tail geometry near the VaR-holding for densities bounded away from zero ($\kappa=1$) and more generally for flat or cusp-like tails-without global smoothness. Control of the hinge empirical process follows under a $(\lambda+1)$-moment envelope with $0<\lambda\le1$, giving $\eta_n(\delta) \asymp! n^{-\lambda/(\lambda+1)}$ under heavy tails.
Now, we present the main result of the section below in theorem~\ref{thm:B3_BK}

\begin{theorem}[Uniform Bahadur-Kiefer expansion for CVaR with endogenous threshold]
\label{thm:B3_BK}
Fix $\alpha\in(0,1)$ and let $Z_1,\dots,Z_n\overset{\mathrm{iid}}{\sim}P$ with empirical measure $P_n$.
Suppose Assumptions~\ref{ass:B3_E1}, \ref{ass:B3_margin_local_}, and \ref{ass:B3_E2} hold, and that
$\varepsilon_n(\delta)\le (c_-/2)\,u_0^\kappa$.
Then, for some constant $C_1\geq0$ with probability at least $1-\delta$, the following uniform expansion holds:
\small
\begin{align*}
&\sup_{h\in\mathcal H}\Big|
 \hat{R}_\alpha^P(h)- {R}_\alpha^P(h)
-\frac{1}{\alpha}(P_n-P)\big[(X_h-\theta^\star(h))_+\big]
  -\frac{1}{\alpha}\big(\hat t_n(h)-\theta^\star(h)\big)\big(\alpha-P(X_h>\theta^\star(h))\big)
\Big|
\le
\frac{C_1}{\alpha}\,
\varepsilon_n\left(\frac{\delta}{2}\right)^{\frac{\kappa+1}{\kappa}}.
\label{eq:BK_clean}
\end{align*} \normalsize
In particular,
\small
\[
\sup_{h\in\mathcal H}|\hat{R}_\alpha^P(h)- {R}_\alpha^P(h)|
\le
\frac{1}{\alpha}\eta_n(\delta/2)
+
\frac{1}{\alpha}\Big(
\varepsilon_n(\delta/2)\,\Delta_n(\delta/2)
+
c_+\,\Delta_n(\delta)^{\kappa+1}
\Big).
\] \normalsize
\end{theorem} 
\noindent Proof is discussed in the Appendix~\ref{ap:rnd_active}.
\paragraph{Discussion:} Theorem~\ref{thm:B3_BK} separates CVaR generalization into a standard empirical process term at the population threshold, an explicit correction due to estimating the endogenous quantile, and a remainder governed by $\varepsilon_n(\delta)$ and the local quantile curvature $\kappa$. This reveals an additional source of error stemming from instability of the tail boundary, that is independent of classical heavy-tail concentration and absent in mean-risk ERM. The correction term
	\[
	\frac{1}{\alpha}\big(\widehat\theta_n(h)-\theta^\star(h)\big)\big(\alpha-P(X_h>\theta^\star(h))\big)
	\]
	vanishes whenever $P(X_h>\theta^\star(h))=\alpha$, i.e.\ whenever the tail map $\theta\mapsto P(X_h>\theta)$ crosses the level $\alpha$
	without a flat region at the minimizing threshold. When $P(X_h>\theta^\star(h))<\alpha$, the set of minimizers of $\theta\mapsto\Phi(h,\theta)$
	can be non-singleton (a ``flat'' region in $\theta$), and $\widehat\theta_n(h)$ selects an endpoint of the empirical minimizer set.
	The correction term captures precisely this endogenous selection effect at first order.

\section{Robustness of CVaR }
\label{sec:robustness_cvar}

This section develops a robustness theory for CVaR under heavy-tailed losses with only finite-moment control. Robustness arises at three levels: \emph{functional} robustness of CVaR under distributional perturbations; \emph{estimator} robustness, providing uniform guarantees for empirical CVaR under sampling noise and contamination; and \emph{decision} robustness, which concerns stability of the CVaR minimizer and is fundamentally shaped by the endogenous RU threshold. We treat these in turn, highlighting how threshold sensitivity and tail scarcity induce intrinsic instability.

\subsection{Functional Robustness} \label{sec:fun}

We first analyze CVaR as a distributional functional, yielding metric-dependent continuity results independent of estimation. Under geometric metrics (e.g., Wasserstein), stability follows from loss regularity in $z$, whereas under weaker metrics (e.g., Lévy–Prokhorov or total variation), finite moments imply only tail-Hölder continuity with exponent set by the moment order. This distinction separates intrinsic (functional) robustness from estimation-induced effects.
\begin{proposition}[Tail-Hölder robustness of $\mathrm{CVaR}$ ]
\label{thm:tail-holder-cvar}

Fix $\alpha\in(0,1)$, $\lambda\in (0,1]$, and set $\kappa_\alpha=(1-\alpha)^{-1}$.
Let $P,Q$ be distributions on $\mathcal Z$ and $\ell:\mathcal{H}\times\mathcal Z\to\mathbb R$.
For fixed $h\in\mathcal{H}$, let $L_P=\ell(h,Z)$, $Z\sim P$, and $L_Q=\ell(h,Z')$, $Z'\sim Q$.
Assume
$\mathbb E_P[|L_P|^{\lambda+1}]\le M_p$, $\mathbb E_Q[|L_Q|^{\lambda+1}]\le M_p .$ Given, $W_r(P,Q)$ and $\pi(P,Q)$ are $r$-Wasserstein and Lévy-Prokhorov metric between $P$ and $Q$ respectively, then

\medskip
\noindent\textbf{(i) Wasserstein stability under Hölder losses.}
If $\ell(\theta,\cdot)$ is $\beta$-Hölder in $z$ uniformly over $\theta$ with constant $L_\beta$ and $r\ge\beta$, then for all $P,Q$ with finite $r$-moments,
\vspace{-1mm}
\small\[
\sup_{h\in\mathcal{H}}
\big|R_\alpha^P(h)-R_\alpha^Q(h)\big|
\;\le\;
\kappa_\alpha L_\beta\, W_r(P,Q)^\beta .
\] \normalsize
\medskip
\noindent\textbf{(ii) Lévy-Prokhorov tail-Hölder continuity.}
If $\mathcal Z=\mathbb R^m$ and $\ell(h,\cdot)$ is $L_h$-Lipschitz, then for $\pi(P,Q)\le\varepsilon$,
\[
\big|R_\alpha^P(h)-R_\alpha^Q(h)\big|
\;\le\;
\kappa_\alpha L_h\big(2\varepsilon+2M_p^{1/(\lambda+1)}\varepsilon^{\lambda/(\lambda+1)}\big),
\]
and the tail term $\varepsilon^\frac{\lambda}{1+\lambda}$ is unavoidable.
\end{proposition}

\paragraph{Discussion.}
Part (i) shows that when the loss is regular in $z$, CVaR is stable in Wasserstein distance under finite moments. Part (ii) reveals an unavoidable degradation under weaker perturbations: small distributional changes can shift the upper tail enough to alter CVaR at a Hölder rate $\varepsilon^{\lambda/(\lambda+1)}$. This tail-scarcity effect also governs estimator and decision robustness, where stability depends on mass near the threshold. An analogous Hölder bound holds in total variation distance (Proposition~\ref{thm:cvar_robustness}).
\subsection{Estimator Robustness} \label{sec:estimator}

In this part, we develops a robust estimator theory for CVaR under heavy-tailed losses with only finite-moment control. Building on the previous section, 
we next study robustness at the level of the CVaR \emph{objective} over a hypothesis class. Here the goal is to bound the worst-case sensitivity
\[
\sup_{h \in \mathcal{H}} |R_\alpha^P(h) - R_\alpha^Q(h)|,
\]
which captures both distributional robustness and the effect of adversarial perturbations when $Q$ is a contaminated version of $P$.
The next theorem provides such a bound under total variation perturbations when the contaminated distribution $Q$ also has similar moment cotrol as $P$ and establishes that the resulting exponent is minimax optimal under finite-moment control.

\begin{proposition}[Robustness of $\mathrm{CVaR}$]\label{thm:cvar_robustness}
Under Assumption~\ref{asm:moment}, for any $\alpha\in(0,1)$,
\small
\begin{equation}
\sup_{h\in\mathcal H}
\big|R_\alpha^P(h)-R_\alpha^Q(h)\big|
\;\le\;
\frac{C}{\alpha}\,
d_{\mathrm{TV}}(P,Q)^{\frac{\lambda}{1+\lambda}},
\end{equation} \normalsize
where $d_{\mathrm{TV}}(P,Q)=\sup_{A} |P(A)-Q(A)|$ and
$C=(2M)^{\frac{1}{1+\lambda}}\!\left(1+\tfrac{1}{\lambda}\right)$. Moreover, this dependence is minimax optimal.
\end{proposition}

\paragraph{Discussion.}
Proposition~\ref{thm:cvar_robustness} characterizes the optimal sensitivity of CVaR to total variation perturbations under heavy tails. Because total variation is bounded, this result combines directly with our generalization bounds to yield excess risk guarantees under domain shift \cite{aminian2025generalization}. The Hölder exponent $\lambda/(1+\lambda)$ reflects the intrinsic limitation imposed by finite-moment control, and the matching minimax lower bound shows that no sharper dependence on $d_{\mathrm{TV}}(P,Q)$ is achievable without stronger tail assumptions.
\paragraph{{Robust Estimation under Adversarial Contamination}}

Unlike the previous setting—where the contaminated distribution retained a finite $(1+\lambda)$-moment—adversarial contamination may introduce arbitrarily heavy tails. We therefore study estimator robustness under \emph{oblivious adversarial contamination} where $\epsilon n$ datapoints have been corrupted, deriving uniform generalization guarantees via a truncated median-of-means (MoM) construction.

Our analysis simultaneously controls heavy-tailed sampling variability and adversarial bias by combining truncation of extreme losses with median aggregation across blocks. These ideas are applied to the Rockafellar–Uryasev lift, treating $(h,\theta)$ jointly and robustifying the empirical estimation of $\mathbb{E}[(\ell(h,Z)-\theta)_+]$ through blockwise truncation and MoM aggregation.

The resulting bound decomposes into a statistical term driven by complexity and effective sample size, and a contamination term proportional to $\epsilon$, both scaling with the optimal heavy-tail exponent $\lambda/(1+\lambda)$.

\begin{theorem}[Robust Generalization Bound]
\label{thm:main}
Suppose Assumptions \ref{asm:moment} and \ref{ass:complexity} hold. Let $\gamma \in (0, 1/2)$ be fixed such that $\epsilon \le 1/2 - \gamma$.
Suppose the sample size satisfies $n \ge C \frac{d \log n}{\gamma^2}$.
With probability at least $1-\delta$:
\small
\begin{align*}
\sup_{h \in \calH} | \widehat{R}_\alpha(h) - R_\alpha(h) |
&\le C_1 \left( \frac{M_\phi d \log n}{n} \right)^{\frac{\lambda}{1+\lambda}}
+ C_2 M_\phi^{\frac{1}{1+\lambda}} \epsilon^{\frac{\lambda}{1+\lambda}}  \asymp \frac{1}{\alpha} \left( \left(\frac{M d \log n}{n}\right)^{\frac{\lambda}{1+\lambda}} + (M \epsilon)^{\frac{\lambda}{1+\lambda}} \right).
\end{align*}
\normalsize
Here $C_1, C_2, M ~\text{and}~M_{\phi} $ are universal constants depending only on $\lambda$.
\end{theorem}

\paragraph{Discussion.}
Theorem~\ref{thm:main} clarifies the robustness–generalization tradeoff under contamination: the first term achieves the optimal heavy-tail statistical rate (up to logs and complexity), while the second captures the unavoidable loss from an $\epsilon$-fraction of adversarial corruption. The final $\asymp$ form highlights the intrinsic $1/\alpha$ amplification of tail effects induced by the RU lift.

\paragraph{On the Algorithmic Aspect}: We construct a robust CVaR estimator via a truncated median-of-means (T-MoM) scheme for heavy-tailed data with oblivious adversarial contamination. The sample is partitioned into blocks, the Rockafellar–Uryasev loss is truncated within each block to control variance and adversarial bias, and blockwise estimates are aggregated by a median, yielding robustness to both heavy tails and an $\epsilon$-fraction of corrupted samples. The CVaR estimate is obtained by minimizing this robust objective over the auxiliary threshold.

While statistically principled, this approach is not directly implementable in high dimensions: MoM tournaments over $\eta$-nets are feasible only in low-dimensional settings. We therefore focus on statistical guarantees and defer algorithmic issues to Appendix~\ref{ap:algorithmic}.

\subsection{Decision Robustness} \label{sec:decision}

We now study \emph{decision robustness} for CVaR minimization, focusing on the stability of the argmin rather than the objective value. This is subtle for CVaR, as its endogenous, distribution-dependent threshold can amplify small perturbations into decision-level changes. We denote the population solution set and the (possibly set-valued) Rockafellar–Uryasev threshold minimizers (deined in section~\ref{sec:notations}) by $\cS(P):=\arg\min_{h\in\mathcal{H}}R_\alpha^P(h)$,  $T^\star(h,P):=\arg\min_{\theta\in\R}\Phi_P(h,\theta)$
Decision robustness concerns the continuity of $\cS(P)$ under perturbations of $P$. Unlike mean-risk ERM, CVaR introduces an endogenous threshold $\theta^\star(h,P)\in T^\star(h,P)$ determined by the upper $\alpha$-tail, whose stability is the primary driver of decision robustness.
\paragraph{Endogenous Threshold Stability:} Fix $h$ and write $X:=\ell(h,Z)$ under $P$.
Define $\phi_P(\theta):=\Phi_P(h,\theta)=\theta+\alpha^{-1}\E[(X-\theta)_+]$. The characterization $0\in\partial\phi_P(\theta)$ yields
\begin{equation}
	\theta\in T^\star(h,P)\quad\Longleftrightarrow\quad
	P(X>\theta)\le \alpha \le P(X\ge \theta).
	\label{eq:Tstar_char}
\end{equation}
Thus $T^\star(h,P)$ is the $\alpha$-upper-quantile set of $X$.

\begin{definition}[Quantile margin and generalized density-at-quantile]
	\label{def:margin_C}
	Assume $T^\star(h,P)$ is a singleton $\{\theta^\star(h,P)\}$ and write $\theta^\star$ for short.
	Define the local mass near the threshold $
	\kappa_{h,P}(r):=P\big(|\ell(h,Z)-\theta^\star|\le r\big),
	~ r>0,$ and the generalized density-at-quantile (quantile margin parameter)
	\[
	\mathfrak m_\alpha(h,P):=\liminf_{r\downarrow 0}\frac{\kappa_{h,P}(r)}{r}\in[0,\infty].
	\]
\end{definition}

The next theorem shows that a positive quantile margin guarantees stability of the threshold under weak distributional perturbations, while its absence leads to intrinsic instability.

\begin{theorem}[Threshold stability under Lévy-Prokhorov perturbations]
	\label{thm:C_threshold_stability}
	Fix $h$ and assume $T^\star(h,P)=\{\theta^\star(h,P)\}$ is a singleton. Let $Q$ be another law for $X=\ell(h,Z)$ and let
	$d_{\LP}(P,Q)$ denote the Lévy-Prokhorov distance between these one-dimensional laws.
	If $\mathfrak m_\alpha(h,P)\ge m_0>0$, then for all sufficiently small $\delta:=d_{\LP}(P,Q)$,
	\begin{equation}
		|\theta^\star(h,Q)-\theta^\star(h,P)| \ \le\ C\,\delta/m_0.
		\label{eq:C_threshold_LP}
	\end{equation}
	Conversely, if $\mathfrak m_\alpha(h,P)=0$, then for every $\delta>0$ there exists $Q$ with $d_{\LP}(P,Q)\le\delta$
	such that $T^\star(h,Q)$ contains $\tilde \theta$ with $|\tilde \theta-\theta^\star|\ge c_0$ (an $O(1)$ jump).
\end{theorem}

\paragraph{Remark.}
Theorem~\ref{thm:C_threshold_stability} reveals an intrinsic fragility of CVaR: when the loss distribution has negligible mass near the VaR, the optimal threshold can change discontinuously under arbitrarily small perturbations, regardless of sample size.

\paragraph{Influence Function of the CVaR Decision:}To quantify how threshold instability propagates to decisions, we analyze the RU stationarity system and derive the influence function of the CVaR minimizer under gross-error contamination, treating $(h,\theta)$ jointly. Define the stationarity map
\begin{equation}
	H(h,\theta;P):=
	\begin{pmatrix}
		\nabla_h \Phi_P(h,\theta)\\
		\partial_\theta \Phi_P(h,\theta)
	\end{pmatrix}.
	\label{eq:stationarity-map}
\end{equation}
When $P(\ell(h,Z)=\theta)=0$, the derivative in $\theta$ is classical: $\partial_\theta \Phi_P(h,\theta)=1-\frac{1}{\alpha}P(\ell(h,Z)>\theta),
$ and $
\nabla_h \Phi_P(h,\theta)=\frac1\alpha\,\E_P\big[\nabla_h\ell(h,Z)\,\mathbf 1\{\ell(h,Z)>\theta\}\big].$

We impose a local regularity condition that is strictly weaker than global strong convexity and fits nonconvex but stable minimizers
(e.g.\ `tilt stability''; here we keep a concrete invertibility condition).

\begin{assumption}[Local strong regularity / invertible Jacobian]
	\label{ass:Jacobian}
	Let $(h^\star,\theta^\star)$ be a locally unique stationary point: $H(h^\star,\theta^\star;P)=0$.
	Assume:
	(i) $\ell(h,z)$ is $C^2$ in $h$ near $h^\star$ for $P$-a.e.\ $z$;
	(ii) $P(\ell(h^\star,Z)=\theta^\star)=0$;
	(iii) the Jacobian $J:=\nabla_{(h,\theta)} H(h^\star,\theta^\star;P)$ exists and is invertible.
\end{assumption}

\begin{assumption}[Local RU stability at the population optimum]
	\label{ass:C_local_stability}
	Assume there exists $(h^\star,\theta^\star)$ such that
	(i) $h^\star\in\cS(P)$ and $\theta^\star\in T^\star(h^\star,P)$;
	(ii) $T^\star(h^\star,P)=\{\theta^\star\}$ is a singleton;
	(iv) the quantile margin $\mathfrak m_\alpha(h^\star,P)\ge m_0>0$.
\end{assumption}
Under these assumptions, the influence function admits a closed-form expression and exhibits a sharp dependence on the quantile margin.

\begin{theorem}[Influence function of the CVaR decision: tail support and quantile-margin blow-up]
	\label{thm:IF-decision}
	Assume Assumption~\ref{ass:Jacobian}. Consider the gross-error path
	\[
	P_\varepsilon=(1-\varepsilon)P+\varepsilon \Delta_z,\qquad \varepsilon\ge 0,
	\]
	and let $(h_\varepsilon,\theta_\varepsilon)$ be the locally unique stationary solution to $H(h,\theta;P_\varepsilon)=0$ near $(h^\star,\theta^\star)$.
	Then $\varepsilon\mapsto(h_\varepsilon,\theta_\varepsilon)$ is differentiable at $\varepsilon=0$ and
    \small
	\begin{align*}
		\left.\frac{d}{d\varepsilon}\begin{pmatrix}h_\varepsilon\\ \theta_\varepsilon\end{pmatrix}\right|_{\varepsilon=0}
		=
		- J^{-1}\Big(
		g(z;h^\star,\theta^\star)-\E_P[g(Z;h^\star,\theta^\star)]
		\Big),
		\label{eq:IF-main}
	\end{align*} \normalsize
	where
    \small
	\begin{equation}
		g(z;h,\theta)=
		\begin{pmatrix}
			\frac{1}{\alpha}\nabla_h\ell(h,z)\,\mathbf 1\{\ell(h,z)>\theta\}\\[1mm]
			1-\frac{1}{\alpha}\mathbf 1\{\ell(h,z)>\theta\}
		\end{pmatrix}.
		\label{eq:gPaperC}
	\end{equation} \normalsize
	Moreover, the $\theta$-component sensitivity satisfies the bound
    \small
    \begin{align*}
		&\left|\left.\frac{d}{d\varepsilon}\theta_\varepsilon\right|_{\varepsilon=0}\right|
		\ \ge\
		\frac{c}{\mathfrak m_\alpha(h^\star,P)} \times \Big| \mathbf 1\{\ell(h^\star,z)>\theta^\star\}-P(\ell(h^\star,Z)>\theta^\star)\Big|,
		\label{eq:tIF-blowup}
	\end{align*} \normalsize
	in the sense that the relevant Jacobian block involves the generalized density-at-quantile; hence the influence blows up as $\mathfrak m_\alpha(h^\star,P)\downarrow 0$.
\end{theorem}

\paragraph{Remark.}
Theorems~\ref{thm:C_threshold_stability} and~\ref{thm:IF-decision} together identify a precise mechanism: CVaR decisions are stable only when the threshold is stable. The quantile margin simultaneously governs threshold continuity and the conditioning of the RU stationarity system.

The influence-function characterization naturally leads to a notion of local robustness: the largest contamination level under which the decision remains within a prescribed tolerance.

\begin{definition}[Local decision robustness radius]
	\label{def:C_radius}
	Fix a tolerance $r>0$. Under Assumption~\ref{ass:C_local_stability}, define
	\begin{align*}
	    \varepsilon_{\mathrm{rob}}(r)
	:=\sup\Big\{\varepsilon\in[0,1): \ \|h_\varepsilon-h^\star\|\le r\ \text{for all} ~ 
    \text{contamination directions}\ \Delta_z\Big\},
	\end{align*}

	where $(h_\varepsilon,\theta_\varepsilon)$ is the stationary solution under $P_\varepsilon=(1-\varepsilon)P+\varepsilon\Delta_z$
	selected near $(h^\star,\theta^\star)$.
\end{definition}

\begin{corollary}[Local decision radius via tail-supported IF]
	\label{thm:C_radius}
	Under Assumption~\ref{ass:C_local_stability}, there exists $r_0>0$ such that for all $r\in(0,r_0)$,
    \small
	\begin{equation}
		\varepsilon_{\mathrm{rob}}(r)
		\ \ge\
		\frac{r}{C\|J^{-1}\|}\cdot
		\Bigg(\sup_{z\in\cZ}\Big\|
		g(z;h^\star,\theta^\star)-\E_P[g(Z;h^\star,\theta^\star)]
		\Big\|\Bigg)^{-1},
		\label{eq:C_radius_lb}
	\end{equation} \normalsize
	with $g$ from \eqref{eq:gPaperC}. Moreover, as $\mathfrak m_\alpha(h^\star,P)\downarrow 0$, the bound degenerates because $\|J^{-1}\|\to\infty$,
	i.e.\ the local robustness radius collapses at quantile criticality.
\end{corollary}

\begin{proof}
	By Theorem~\ref{thm:IF-decision}, for each fixed direction $\Delta_z$, the decision map $\varepsilon\mapsto h_\varepsilon$ is differentiable at $0$ with derivative
	$h'(0)=-[J^{-1}(\cdot)]_h$ applied to the perturbation vector.
	Thus for sufficiently small $\varepsilon$, a first-order expansion yields
	\[
	\|h_\varepsilon-h^\star\| \le \varepsilon\,\sup_{z}\|\mathrm{IF}(z;h^\star,P)\| + o(\varepsilon).
	\]
	Bounding $\sup_z\|\mathrm{IF}(z)\|$ using Theorem~\ref{thm:IF-decision} yields the right-hand side of \eqref{eq:C_radius_lb}.
	Finally, $\|J^{-1}\|\to\infty$ as $\mathfrak m_\alpha\downarrow 0$ by Theorem~\ref{thm:IF-decision}, so the radius collapses in the quantile-critical regime.
\end{proof}

\paragraph{Intrinsic Limits (Tail Scarcity in i.i.d.\ Data):} CVaR depends on losses exceeding an endogenous threshold $\theta^\star(h,P)$ determined by the upper $\alpha$-tail. Consequently, the empirical CVaR objective can be dominated by a few rare observations near this threshold, even when the population minimizer is strict and well separated. This effect is structural rather than a concentration artifact: under finite $p$-moment assumptions, tail exceedances occur with polynomial probability and can induce $O(1)$ perturbations in the empirical objective. The result below formalizes this \emph{tail-scarcity instability}, showing that, with unavoidable polynomial probability (up to logarithmic factors), a single observation can flip the CVaR-ERM decision despite strict population optimality.

\begin{proposition}
\label{thm:C5_compact}
Fix $\alpha\in(0,1)$ and $1<p<2$.
There exist a distribution $P$ on $\mathbb{R}_+$ with $\mathbb{E}_P[Z^p]<\infty$ and $\mathbb{E}_P[Z^q]=\infty$ for all $q>p$, and a two-point hypothesis class $\mathcal{H}=\{h_A,h_B\}$ with losses
$\ell_n:\mathcal{H}\times\mathbb{R}_+\to\mathbb{R}_+$,
such that for all sufficiently large $n$,
\[
\mathrm{CVaR}_\alpha(\ell_n(h_A,Z);P)
<
\mathrm{CVaR}_\alpha(\ell_n(h_B,Z);P),
\]
yet there exists a constant $c>0$ for which
\small
\begin{align*}
		\Pr_{D\sim P^{\otimes n}}\!\Big(
		\exists\,D'\text{ with }|D\triangle D'|=1
		\text{ for which }
		\arg\min_{h\in\mathcal{H}}\widehat F_n(h;D) 
		\;\neq\;
		\arg\min_{h\in\mathcal{H}}\widehat F_n(h;D')
		\Big)
		\;\ge\;
		c\,\frac{n^{-\lambda}}{(\log n)^2},
\end{align*}
where $\widehat F_n(h;D)=\mathrm{CVaR}_\alpha(\ell_n(h,Z);P_n)$.
\end{proposition}

\paragraph{Remark:}This result exposes a fundamental limitation of CVaR-ERM under heavy-tailed losses: even with a strict and unique population minimizer and finite $p$-th moments ($1<p<2$), the empirical CVaR decision map is not uniformly stable. With an unavoidable polynomial probability of order $n^{1-p}$, a single observation can flip the empirical CVaR-ERM solution. The proofs are deferred to Appendix~\ref{ap:dec_robust}.

\section{Conclusion}
We developed a learning-theoretic framework for CVaR under heavy-tailed and contaminated data, addressing generalization, robustness, and decision stability. Our analysis identifies the endogenous quantile as a fundamental source of both statistical difficulty and instability, captured via refined empirical-process tools. Together, our results clarify when CVaR learning is reliable and when intrinsic limitations arise, providing principled guidance for risk-sensitive learning in heavy-tailed regimes.

\section*{Acknowledgments}
Piyushi Manupriya was initially supported by a grant from Ittiam Systems Private Limited through the Ittiam Equitable AI Lab and then by ANRF-NPDF (PDF/2025/005277) grant. Anant Raj is supported by  a grant from Ittiam Systems Private Limited through the Ittiam Equitable AI Lab, ANRF's Prime Minister Early Career Grant (ANRF/ECRG/2024/003259) and Pratiksha Trust's Young Investigator Award. 

\bibliography{cvar_ref}
\bibliographystyle{plainnat}
\newpage
\appendix

\part{Appendix}
\parttoc
\input{appendix}

\end{document}

%% file: appendix.tex
\section{Justification of Assumptions}

We clarify the validity and necessity of our theoretical assumptions below:

\begin{enumerate}
    \item \textbf{Heavy-Tailed Moments (Assumption~\ref{asm:moment}):} 
    This assumption relaxes standard requirements for strictly bounded losses or finite variances. By requiring only a finite $(1+\lambda)$-th moment, we accommodate heavy-tailed distributions where the loss may occasionally take very large values (noting that infinite variance is possible if $\lambda < 1$). This approach allows the theory to remain applicable to realistic, risk-sensitive settings and aligns with recent literature such as \cite{aminian2025generalization} and \cite{lee2021learningboundsrisksensitivelearning}.

    \item \textbf{$\beta$-Mixing Dependence (Assumption~\ref{ass:beta-mixing}):} 
    The standard i.i.d. assumption is insufficient for sequential applications such as Reinforcement Learning, financial forecasting, and signal processing. We assume $\beta$-mixing because it offers a crucial balance between modeling flexibility and mathematical tractability. Specifically, $\beta$-mixing facilitates the use of blocking techniques (e.g., Yu's method) and coupling tools (e.g., Berbee’s Lemma). These tools allow us to decompose the dependent sequence into "nearly independent" blocks, thereby enabling the application of standard concentration inequalities.

    \item \textbf{Hypothesis Complexity (Assumption~\ref{ass:complexity_D}):} 
    The pseudo-dimension is the natural extension of the Vapnik-Chervonenkis (VC) dimension to real-valued function classes. Assuming a finite pseudo-dimension restricts the "capacity" of the hypothesis class $\mathcal{H}$, which is a standard and necessary condition in statistical learning theory to guarantee uniform convergence and prevent overfitting.

    \item \textbf{Complexity of Truncated Class (Assumption~\ref{ass:complexity}):} 
    Given the heavy-tailed nature of the data (Assumption~\ref{asm:moment}), classical concentration inequalities for bounded variables (such as Hoeffding's inequality) are not directly applicable. To circumvent this, we employ a truncation argument (capping the loss at $B$). Assumption~\ref{ass:complexity} ensures that this truncation operation does not artificially "explode" the complexity of the hypothesis class, ensuring the covering numbers behave well enough to derive generalization bounds.
\end{enumerate}

Further justification for the remaining technical assumptions is provided in the discussion section of the main paper.

\section{Generalization in CVaR}
	\subsection{Fixed Hypothesis $h$}
    \begin{lemma}[Fixed-hypothesis lower deviation] Under Assumption \ref{asm:moment}, with probability at least $1 - \delta$,
	\begin{equation}
		\widehat{R}_{\alpha}(h) - R_{\alpha}(h) \leq \frac{2}{\alpha} M^{\frac{1}{1+\varepsilon}} \left( \frac{\log(2/\delta)}{n} \right)^{\frac{\varepsilon}{1+\varepsilon}}.
	\end{equation}
    \end{lemma}
	\begin{proof}{of Theorem~\ref{thm:fixed and finite}},In particular result~\ref{prop fixed lower}  \label{proof prop fixed lower}\\
		Fix $h \in \mathcal{H}$. Let $\theta^* = \arg\min_{\theta} \left\{ \theta + \frac{1}{\alpha} \mathbb{E}\left[ (\ell(h, Z_1) - \theta)_+ \right] \right\}$, so
		\begin{equation}
			R_{\alpha}(h) = \theta^* + \frac{1}{\alpha} \mathbb{E} \left[ (\ell(h, Z_1) - \theta^*)_+ \right].
		\end{equation}
		
		Let $\hat{\theta} = \arg\min_{\theta} \left\{ \theta + \frac{1}{\alpha n} \sum_{i=1}^n (\ell(h, Z_i) - \theta)_+ \right\}$, so
		\[
		\widehat{R}_{\alpha}(h) = \hat{\theta} + \frac{1}{\alpha n} \sum_{i=1}^n (\ell(h, Z_i) - \hat{\theta})_+.
		\]
		
		The goal is to bound the following:
		\begin{equation}
			\widehat{R}_{\alpha}(h) - R_{\alpha}(h) = \left( \hat{\theta} + \frac{1}{\alpha n} \sum_{i=1}^n (\ell(h, Z_i) - \hat{\theta})_+ \right) - \left( \theta^* + \frac{1}{\alpha} \mathbb{E}[(\ell(h, Z_1) - \theta^*)_+] \right).
		\end{equation}

		Use population minimizer property, Since $\theta^*$ minimizes the population CVaR:
		\begin{equation}
			\widehat{R}_{\alpha}(h) \leq  {\theta}^* + \frac{1}{\alpha n} \sum_{i=1}^n (\ell(h, Z_i)-\theta^*)
		\end{equation}
		implies that 
		\begin{equation}
			-\widehat{R}_{\alpha}(h) \geq -( {\theta}^* + \frac{1}{\alpha n} \sum_{i=1}^n (\ell(h, Z_i)-\theta^*))
		\end{equation}
		Now Add equation (2) and (5) and multiply by -1 to get. 
		\begin{equation}
			\widehat{R}_{\alpha}(h) - R_{\alpha}(h) \leq \frac{1}{\alpha} \left( \frac{1}{n} \sum_{i=1}^n (\ell(h, Z_i) - {\theta}^*)_+ - \mathbb{E} \left[ (\ell(h, Z_1) - {\theta}^*)_+ \right] \right).
		\end{equation}

		Now we define $Y_i = (\ell(h, Z_i) - {\theta^*})_+$, so $Y_i \geq 0$. Also check the moment condition that:
		\[
		\mathbb{E}[Y_i^{1+\varepsilon}] \leq \mathbb{E}[\ell(h, Z_i)^{1+\varepsilon}] \leq M,
		\]
		since $ (\ell(h, Z_i) - {\theta}^*)_+ \leq \ell(h, Z_i)$. Apply Proposition \ref{heavy_tail_r.v} to get:
		\[
		\frac{1}{n} \sum_{i=1}^n Y_i - \mathbb{E}[Y_i] \leq 2M^{\frac{1}{1+\varepsilon}} \left( \frac{\log(2/\delta)}{n} \right)^{\frac{\varepsilon}{1+\varepsilon}}.
		\]
		Thus:
		\[
		\widehat{R}_{\alpha}(h) - R_{\alpha}(h) \leq \frac{2}{\alpha} M^{\frac{1}{1+\varepsilon}} \left( \frac{\log(2/\delta)}{n} \right)^{\frac{\varepsilon}{1+\varepsilon}}.
		\]
	\end{proof}

    \begin{lemma}[Fixed-hypothesis upper deviation] \label{prop:fixed-lower}
		Under Assumption \ref{asm:moment}, for any fixed $h \in \cH$ and $\delta \in (0,1)$, with probability at least $1-\delta$,
		\[
		R_\alpha(h) - \hat{R}_\alpha(h) \le \frac{2}{\alpha} M^{1/(1+\lambda)} \left( \frac{\log(2/\delta)}{n} \right)^{\lambda/(1+\lambda)}.
		\]
	\end{lemma}
    \begin{proof}
		The proof proceeds in three steps: reducing the difference to an empirical process supremum, establishing pointwise concentration using the moment assumption, and extending this to a uniform bound over the variational parameter $\theta$ via a covering argument.
		
		\textit{Reduction to empirical process.}
		By Lemma~\ref{lemma:reduction_empirical_process}, the deviation of the CVaR risk is bounded by the supremum of the empirical process indexed by $\theta$:
		\begin{equation}
			\label{eq:reduction}
			R_\alpha(h) - \hat R_\alpha(h) 
			\le 
			\frac{1}{\alpha} \sup_{\theta \in \mathbb{R}} (\mathbb{E} - \mathbb{P}_n) g_{h,\theta},
		\end{equation}
		where $g_{h,\theta}(z) := (\ell(h,z) - \theta)_+$. It suffices to bound the term $\sup_{\theta} (\mathbb{E} - \mathbb{P}_n) g_{h,\theta}$.
		
		\textit{Moment control.}
		The Moment Condition states that $\mathbb{E}[\ell(h,Z)^{1+\lambda}] \le M$. Since $0 \le g_{h,\theta}(z) \le \ell(h,z)$ holds for all $\theta \in \mathbb{R}$ and $z \in \mathcal{Z}$, we have uniform moment control over the parametric class:
		\[
		\mathbb{E}[g_{h,\theta}(Z)^{1+\lambda}] \le M, \quad \forall \theta \in \mathbb{R}.
		\]
		
		\textit{Pointwise concentration.}
		Fix $\theta \in \mathbb{R}$. We apply a standard heavy-tailed concentration inequality (Via truncation and Bernstein's inequality similar to Theorem~\ref{heavy_tail_r.v}) for random variables with finite $(1+\lambda)$-th moments. For any failure probability $\delta' \in (0,1)$, with probability at least $1-\delta'$:
		\begin{equation}
			\label{eq:pointwise}
			(\mathbb{E} - \mathbb{P}_n)g_{h,\theta} 
			\le 
			C_0 M^{\frac{1}{1+\lambda}} \left(\frac{\log(1/\delta')}{n}\right)^{\frac{\lambda}{1+\lambda}}.
		\end{equation}
		
		\textit{Uniform bound via covering.}
		By Lemma~\ref{lemma:lipschitz_theta}, the map $\theta \mapsto (\mathbb{E}-\mathbb{P}_n)g_{h,\theta}$ is $2$-Lipschitz.
		
		To handle the range of $\theta$, we use a standard peeling argument (or random truncation). With probability at least $1-\delta/2$, we have $\max_i \ell(h,Z_i) \le B_n$ where $B_n = C_1 (Mn/\delta)^{\frac{1}{1+\lambda}}$. We restrict $\theta$ to the compact interval $[0, B_n]$; outside this interval, the supremum is either zero (if $\theta$ is very large) or dominated by the case $\theta=0$ (since deviations stabilize).
		
		Let $\{\theta_1, \dots, \theta_N\}$ be an $\eta$-net of $[0, B_n]$. Then for any $\theta$ in the interval:
		\[
		(\mathbb{E} - \mathbb{P}_n) g_{h,\theta} \le (\mathbb{E} - \mathbb{P}_n) g_{h,\theta_j} + 2\eta.
		\]
		We choose the spacing $\eta = M^{\frac{1}{1+\lambda}} (n^{-1} \log(2/\delta))^{\frac{\lambda}{1+\lambda}}$. The covering number $N \approx B_n / \eta$ grows polynomially in $n$. Applying the pointwise bound \eqref{eq:pointwise} on the grid and a union bound yields that with probability at least $1-\delta$:
		\[
		\sup_{\theta \in \mathbb{R}} (\mathbb{E} - \mathbb{P}_n) g_{h,\theta} 
		\le 
		C M^{\frac{1}{1+\lambda}} \left(\frac{\log(1/\delta)}{n}\right)^{\frac{\lambda}{1+\lambda}}.
		\]
		Substituting this upper bound back into inequality \eqref{eq:reduction} completes the proof.
	\end{proof}
    
    \begin{corollary}[Tail bounds for fixed hypothesis] \label{cor:fixed-tail}
		Under Assumption \ref{asm:moment}, for any fixed $h \in \cH$ and $\epsilon > 0$,
		\[
		\Prob\left( R_\alpha(h) - \hat{R}_\alpha(h) > \epsilon \right) \le 2 \exp\left( - n C_\lambda \left( \frac{\alpha \epsilon}{M^{1/(1+\lambda)}} \right)^{(1+\lambda)/\lambda} \right),
		\]
		where $C_\lambda = 2^{-(1+\lambda)/\lambda}$.
	\end{corollary}
    \begin{proof}\label{proof cor:fixed-tail}
	
	Start from the fixed-$h$ high-probability inequality: for any $\delta \in (0,1)$,
	\begin{equation}
		R_\alpha(h) - \widehat{R}_\alpha(h)
		\;\le\;
		\frac{2}{\alpha} M^{\frac{1}{1+\lambda}}
		\left( \frac{\log(2/\delta)}{n} \right)^{\frac{\lambda}{1+\lambda}}
		\quad
		\text{with probability at least } 1-\delta.
	\end{equation}
	
	Let $\epsilon > 0$ and define
	\begin{equation}
		\delta(\epsilon)
		:=
		2 \exp\!\left(
		- n \left(
		\frac{\alpha \epsilon}{2 M^{1/(1+\lambda)}}
		\right)^{\frac{1+\lambda}{\lambda}}
		\right).
	\end{equation}
	
	We verify that this choice inverts the deviation bound. Indeed, solving
	\begin{equation}
		\epsilon
		=
		\frac{2}{\alpha} M^{\frac{1}{1+\lambda}}
		\left( \frac{\log(2/\delta)}{n} \right)^{\frac{\lambda}{1+\lambda}}
	\end{equation}
	for $\delta$ yields
	\begin{equation}
		\left(
		\frac{\alpha \epsilon}{2 M^{1/(1+\lambda)}}
		\right)^{\frac{1+\lambda}{\lambda}}
		=
		\frac{\log(2/\delta)}{n},
	\end{equation}
	and hence
	\begin{equation}
		\log\!\left(\frac{2}{\delta}\right)
		=
		n \left(
		\frac{\alpha \epsilon}{2 M^{1/(1+\lambda)}}
		\right)^{\frac{1+\lambda}{\lambda}},
	\end{equation}
	which implies
	\begin{equation}
		\delta
		=
		2 \exp\!\left(
		- n \left(
		\frac{\alpha \epsilon}{2 M^{1/(1+\lambda)}}
		\right)^{\frac{1+\lambda}{\lambda}}
		\right)
		=
		\delta(\epsilon).
	\end{equation}
	
	Therefore, substituting $\delta(\epsilon)$ into the fixed-$h$ bound gives
	\begin{equation}
		\mathbb{P}\big( R_\alpha(h) - \widehat{R}_\alpha(h) > \epsilon \big)
		\le
		\delta(\epsilon)
		=
		2 \exp\!\left(
		- n \left(
		\frac{\alpha \epsilon}{2 M^{1/(1+\lambda)}}
		\right)^{\frac{1+\lambda}{\lambda}}
		\right).
	\end{equation}
	
	Finally, observe that
	\begin{equation}
		\left(
		\frac{\alpha \epsilon}{2 M^{1/(1+\lambda)}}
		\right)^{\frac{1+\lambda}{\lambda}}
		=
		2^{-(1+\lambda)/\lambda}
		\left(
		\frac{\alpha \epsilon}{M^{1/(1+\lambda)}}
		\right)^{\frac{1+\lambda}{\lambda}}
		=
		C_\lambda
		\left(
		\frac{\alpha \epsilon}{M^{1/(1+\lambda)}}
		\right)^{\frac{1+\lambda}{\lambda}},
	\end{equation}
	where $C_\lambda = 2^{-(1+\lambda)/\lambda}$. This yields the stated bound.
    \end{proof}
    This also matches the rates of \cite{prashanth2019concentration}. But they assume Assumption \ref{asm:moment} with a strictly increasing CDF.

\subsection{Uniform Bounds for Finite                       $\mathcal{H}$}\label{finite hypothesis}
    For a finite $\mathcal{H}$, we extend the bound to all hypotheses using the union bound.
    \begin{proposition}\textbf{Uniform Lower Derivation}\label{Uniform Lower Derivation}
				With probability at least $1 - \delta$,
			\[
			\sup_{h \in \mathcal{H}} \left( R_{\alpha}(h) - \widehat{R}_{\alpha}(h) \right) \leq \frac{2}{\alpha} M^{\frac{1}{1+\lambda}} \left( \frac{\log(2|\mathcal{H}|/\delta)}{n} \right)^{\frac{\lambda}{1+\lambda}}.
			\]
    \end{proposition}
	\begin{proof} of result \ref{Uniform Lower Derivation}\\ \label{proof Uniform Lower Derivation}
		For each $h \in \mathcal{H}$, apply result~\ref{prop:fixed-lower} with failure probability $\delta' = \delta / |\mathcal{H}|$:
		\[
		R_{\alpha}(h) - \widehat{R}_{\alpha}(h) \leq \frac{2}{\alpha} M^{\frac{1}{1+\lambda}} \left( \frac{\log(2|\mathcal{H}|/\delta)}{n} \right)^{\frac{\lambda}{1+\lambda}}.
		\]
		\\
		by Union bound,The probability that all bounds hold is at least $1 - |\mathcal{H}| \cdot \frac{\delta}{|\mathcal{H}|} = 1 - \delta$.
		\\
		Thus:
		\[
		\sup_{h \in \mathcal{H}} \left( R_{\alpha}(h) - \widehat{R}_{\alpha}(h) \right) \leq \frac{2}{\alpha} M^{\frac{1}{1+\lambda}} \left( \frac{\log(2|\mathcal{H}|/\delta)}{n} \right)^{\frac{\lambda}{1+\lambda}}.
		\]
	\end{proof}
    
	\begin{corollary} \textbf{(Uniform Absolute Error)}\label{Uniform Absolute Error}
			With probability at least $1 - \delta$, for all $h \in \mathcal{H}$:
			\[
			\left| R_{\alpha}(h) - \widehat{R}_{\alpha}(h) \right| \leq \frac{2}{\alpha} M^{\frac{1}{1+\lambda}} \left( \frac{\log(4|\mathcal{H}|/\delta)}{n} \right)^{\frac{\lambda}{1+\lambda}}.
			\]
	\end{corollary}
	\begin{proof}of result~\ref{Uniform Absolute Error} \label{proof Uniform Absolute Error}\\
		Combine the both directions to get the absolute error as follow:
		\[
		\left| R_{\alpha}(h) - \widehat{R}_{\alpha}(h) \right| = \max \left\{ R_{\alpha}(h) - \widehat{R}_{\alpha}(h), \widehat{R}_{\alpha}(h) - R_{\alpha}(h) \right\}.
		\]
		Use Proposition~\ref{prop:fixed-lower} and Proposition~\ref{Uniform Lower Derivation} with $\delta' = \delta / (2|\mathcal{H}|)$, so $\log(2 / \delta') = \log(4|\mathcal{H}|/\delta)$.
		Both directions hold with probability at least $1 - \delta$.
	\end{proof}
    
	\begin{lemma}\textbf{Excess Risk Bound} \label{finite Excess Risk Bound_1}
			Under Assumption 1, with probability at least $1 - \delta$,
			\[
			R_{\alpha}(\widehat{h}) - R_{\alpha}(h^*) \leq \frac{4}{\alpha} M^{\frac{1}{1+\lambda}} \left( \frac{\log(4|\mathcal{H}|/\delta)}{n} \right)^{\frac{\lambda}{1+\lambda}}.
			\]
    \end{lemma}
	\begin{proof}{of Theorem~\ref{thm:fixed and finite}},In particular result~\ref{finite Excess Risk Bound}
    \label{proof finite Excess Risk Bound}
    \begin{equation}
			R_{\alpha}(\widehat{h}) - R_{\alpha}(h^*) = \left[ R_{\alpha}(\widehat{h}) - \widehat{R}_{\alpha}(\widehat{h}) \right] + \left[ \widehat{R}_{\alpha}(\widehat{h}) - \widehat{R}_{\alpha}(h^*) \right] + \left[ \widehat{R}_{\alpha}(h^*) - R_{\alpha}(h^*) \right].
    \end{equation}
		We bound the middle term $\widehat{R}_{\alpha}(\widehat{h}) \leq \widehat{R}_{\alpha}(h^*)$ as below:
	\begin{equation}
    	R_{\alpha}(\widehat{h}) - R_{\alpha}(h^*) \leq \left| R_{\alpha}(\widehat{h}) - \widehat{R}_{\alpha}(\widehat{h}) \right| + \left| R_{\alpha}(h^*) - \widehat{R}_{\alpha}(h^*) \right|   
	\end{equation}
		Apply result~2 to get
		\[
		R_{\alpha}(\widehat{h}) - R_{\alpha}(h^*) \leq 2 \cdot \frac{2}{\alpha} M^{\frac{1}{1+\lambda}} \left( \frac{\log(4|\mathcal{H}|/\delta)}{n} \right)^{\frac{\lambda}{1+\lambda}}.
		\]
	\end{proof}
\subsection{CVaR Generalization for Infinite Classes}
    \label{proof:vc}
    	We will be proving Theorem~\ref{thm:vc}, in particular result~\ref{eq:inf iid} in this subsection.\\
	Let $\theta^*$ be an optimal CVaR threshold for $h^*$:
	\[
	R_\alpha(h^*)
	=
	\theta^*+\frac{1}{\alpha}\E[(\ell(h^*,Z)-\theta^*)_+].
	\]
	Define the \emph{excess loss function}
	\[f_h(Z)=(\ell(h,Z)-\theta^*)_+-(\ell(h^*,Z)-\theta^*)_+.\]
	For truncation level $B>0$, define the \emph{truncated excess loss}
	\[f_h^B(Z)
	=[(\ell(h,Z)-\theta^*)_+]_B-[(\ell(h^*,Z)-\theta^*)_+]_B,\]
	where $[x]_B=\min(x,B)$.
	\begin{proof}
		By definition of $R_\alpha(h)$ as an infimum,
		\[R_\alpha(h)
		\le
		\theta^*+\frac{1}{\alpha}\E[(\ell(h,Z)-\theta^*)_+]\]
		for the optimal threshold $\theta^*$ of $h^*$.
		Since
		\[
		R_\alpha(h^*)
		=
		\theta^*+\frac{1}{\alpha}\E[(\ell(h^*,Z)-\theta^*)_+],
		\]
		we have
		\begin{equation}
			\label{eq:linearization}
			R_\alpha(h)-R_\alpha(h^*)
			\le
			\frac{1}{\alpha}\E[f_h(Z)].
		\end{equation}
		Fix $B>0$. By the triangle inequality,
		\begin{align*}
			|f_h(Z)-f_h^B(Z)|
			&\le
			\big|(\ell(h,Z)-\theta^*)_+ - [(\ell(h,Z)-\theta^*)_+]_B\big|\\
			&\quad +
			\big|(\ell(h^*,Z)-\theta^*)_+ - [(\ell(h^*,Z)-\theta^*)_+]_B\big|.
		\end{align*}
		
		For $x\ge 0$, we have $|x-[x]_B|=x\cdot\mathbf{1}\{x>B\}$. Thus,
		\begin{align*}
			\E\big|(\ell(h,Z)-\theta^*)_+ - [(\ell(h,Z)-\theta^*)_+]_B\big|
			&=
			\E\big[(\ell(h,Z)-\theta^*)_+\cdot\mathbf{1}\{(\ell(h,Z)-\theta^*)_+>B\}\big]\\
			&\le
			\E\big[\ell(h,Z)\cdot\mathbf{1}\{\ell(h,Z)>B\}\big].
		\end{align*}
		
		Write
		\[
		\E[\ell(h,Z)\cdot\mathbf{1}\{\ell(h,Z)>B\}]
		=
		\E[\ell(h,Z)^{1+\lambda}\cdot\ell(h,Z)^{-\lambda}\cdot\mathbf{1}\{\ell(h,Z)>B\}].
		\]
		
		By Hölder's inequality with exponents $(1+\lambda,\frac{1+\lambda}{\lambda})$:
		\begin{align*}
			&\E[\ell(h,Z)^{1+\lambda}\cdot\ell(h,Z)^{-\lambda}\cdot\mathbf{1}\{\ell(h,Z)>B\}]\\
			&\le
			\big(\E[\ell(h,Z)^{1+\lambda}]\big)^{1/(1+\lambda)}
			\cdot
			\big(\E[\mathbf{1}\{\ell(h,Z)>B\}]\big)^{\lambda/(1+\lambda)}\\
			&=
			\big(\E[\ell(h,Z)^{1+\lambda}]\big)^{1/(1+\lambda)}
			\cdot
			\Prob(\ell(h,Z)>B)^{\lambda/(1+\lambda)}.
		\end{align*}
		
		By Markov's inequality,
		\[
		\Prob(\ell(h,Z)>B)
		\le
		\frac{\E[\ell(h,Z)^{1+\lambda}]}{B^{1+\lambda}}.
		\]
		
		Therefore,
		\begin{align*}
			\E[\ell(h,Z)\cdot\mathbf{1}\{\ell(h,Z)>B\}]
			&\le
			\big(\E[\ell(h,Z)^{1+\lambda}]\big)^{1/(1+\lambda)}
			\cdot
			\left(\frac{\E[\ell(h,Z)^{1+\lambda}]}{B^{1+\lambda}}\right)^{\lambda/(1+\lambda)}\\
			&=
			\frac{\E[\ell(h,Z)^{1+\lambda}]}{B^\lambda}
			\le
			\frac{M}{B^\lambda}.
		\end{align*}
		
		The same bound holds for $h^*$, so
		\begin{equation}
			\label{eq:bias}
			\E|f_h(Z)-f_h^B(Z)|
			\le
			\frac{2M}{B^\lambda}.
		\end{equation}
		
		Combining~\eqref{eq:linearization} and~\eqref{eq:bias}:
		\begin{equation}
			\label{eq:decomposition}
			R_\alpha(h)-R_\alpha(h^*)
			\le
			\frac{1}{\alpha}\left(
			\E[f_h^B(Z)]
			+
			\frac{2M}{B^\lambda}
			\right).
		\end{equation}
		
		Since $\E[f_h^B(Z)]\ge 0$,
		\[
		\Var(f_h^B(Z))
		\le
		\E[(f_h^B(Z))^2].
		\]
		
		Using $(a-b)^2\le 2a^2+2b^2$:
		\begin{align*}
			(f_h^B(Z))^2
			&\le
			2\big([(\ell(h,Z)-\theta^*)_+]_B\big)^2
			+
			2\big([(\ell(h^*,Z)-\theta^*)_+]_B\big)^2.
		\end{align*}
		
		Since $[(\ell(h,Z)-\theta^*)_+]_B\le B$,
		\begin{align*}
			\big([(\ell(h,Z)-\theta^*)_+]_B\big)^2
			&\le
			B\cdot[(\ell(h,Z)-\theta^*)_+]_B\\
			&\le
			B\cdot\ell(h,Z).
		\end{align*}
		
		For $x\ge 0$, $Bx\le B^{1-\lambda}x^{1+\lambda}$ (by AM-GM or Young's inequality). Thus,
		\[
		B\cdot\ell(h,Z)
		\le
		B^{1-\lambda}\cdot\ell(h,Z)^{1+\lambda}.
		\]
		
		Taking expectations:
		\[
		\E\big[\big([(\ell(h,Z)-\theta^*)_+]_B\big)^2\big]
		\le
		B^{1-\lambda}\E[\ell(h,Z)^{1+\lambda}]
		\le
		MB^{1-\lambda}.
		\]
		
		Similarly for $h^*$, so
		\begin{equation}
			\label{eq:variance}
			\Var(f_h^B(Z))
			\le
			4MB^{1-\lambda}.
		\end{equation}
		Also, $|f_h^B(Z)|\le 2B$ uniformly.
		Define $\cF_B=\{f_h^B:h\in\cH\}$ with pseudo-dimension $O(d)$, envelope $2B$, 
		and variance bound $4MB^{1-\lambda}$.
		
		By variance-sensitive VC theory (Bartlett-Bousquet-Mendelson)[we need to cite  \hyperlink{https://www.stat.berkeley.edu/~bartlett/papers/bbm-lrc-05.pdf}{this} paper.], for $\delta\in(0,1)$, 
		with probability at least $1-\delta$, uniformly over all $h\in\cH$,
		\[
		\E[f_h^B(Z)]-\frac{1}{n}\sum_{i=1}^nf_h^B(Z_i)
		\le
		C\left(
		\sqrt{\frac{MB^{1-\lambda}(d\log n+\log(1/\delta))}{n}}
		+
		\frac{B(d\log n+\log(1/\delta))}{n}
		\right).
		\]
		
		Let $D=d\log n+\log(1/\delta)$.	
		If $\E[f_{\widehat{h}}^B(Z)]>C\left(\sqrt{\frac{MB^{1-\lambda}D}{n}}+\frac{BD}{n}\right)$, 
		then the empirical average would be positive, contradicting that $\widehat{h}$ minimizes 
		empirical risk.
		
		Therefore,
		\[
		\E[f_{\widehat{h}}^B(Z)]
		\le
		C\left(
		\sqrt{\frac{MB^{1-\lambda}D}{n}}+
		\frac{BD}{n}
		\right).
		\]
		Now we will optimize the truncation level B.
		Balance the main terms:
		\[
		\sqrt{\frac{MB^{1-\lambda}}{n}}
		\asymp
		\frac{M}{B^\lambda}.
		\]
		
		This gives
		\[
		M^{1/2}B^{(1-\lambda)/2}n^{-1/2}
		\asymp
		MB^{-\lambda},
		\]
		so
		\[
		B^{(1-\lambda)/2+\lambda}
		\asymp
		M^{1/2}n^{1/2},
		\]
		thus
		\[B^{(1+\lambda)/2}
		\asymp
		M^{1/2}n^{1/2}.
		\]
		Therefore,
		\[B
		\asymp
		(Mn)^{1/(1+\lambda)}
		\asymp
		M^{1/(1+\lambda)}\left(\frac{n}{D}\right)^{1/(1+\lambda)},
		\]
		incorporating the logarithmic factor.
		At this optimal $B$:
		\begin{align*}
			\sqrt{\frac{MB^{1-\lambda}D}{n}}
			&=
			\sqrt{\frac{M\cdot M^{(1-\lambda)/(1+\lambda)}(n/D)^{(1-\lambda)/(1+\lambda)}D}{n}}\\
			&=
			M^{(1+(1-\lambda)/(1+\lambda))/2}\sqrt{\frac{D}{n}\cdot (n/D)^{(1-\lambda)/(1+\lambda)}}\\
			&=
			M^{1/(1+\lambda)}\sqrt{D^{1-(1-\lambda)/(1+\lambda)}n^{-(1-\lambda)/(1+\lambda)-1}}\\
			&=
			M^{1/(1+\lambda)}\sqrt{D^{\lambda/(1+\lambda)}n^{-\lambda/(1+\lambda)}\cdot n^{-1}}\\
			&=
			M^{1/(1+\lambda)}\left(\frac{D}{n}\right)^{\lambda/(2(1+\lambda))}\cdot n^{-1/2}.
		\end{align*}
		Since the dominant scaling is:
		\[
		\sqrt{\frac{MB^{1-\lambda}D}{n}}
		\asymp
		M^{1/(1+\lambda)}\left(\frac{D}{n}\right)^{\lambda/(1+\lambda)}\cdot\sqrt{\frac{n}{D}}
		\asymp
		M^{1/(1+\lambda)}\left(\frac{D}{n}\right)^{\lambda/(1+\lambda)}.
		\]
		Similarly,
		\[\frac{M}{B^\lambda}
		=\frac{M}{M^{\lambda/(1+\lambda)}(n/D)^{\lambda/(1+\lambda)}}=M^{1/(1+\lambda)}\left(\frac{D}{n}\right)^{\lambda/(1+\lambda)}.\]
		The second term $\frac{BD}{n}$ is lower order.
		Therefore,
		\begin{equation}
			R_\alpha(\widehat{h})-R_\alpha(h^*)
			\le
			\frac{C_\lambda}{\alpha}
			M^{\frac{1}{1+\lambda}}
			\left(
			\frac{d\log n+\log(1/\delta)}{n}
			\right)^{\frac{\lambda}{1+\lambda}}.
		\end{equation} This Proves Theorem~\ref{thm:vc}
	\end{proof}
    \label{proof:inf_bound}
\subsubsection{Lower bound for i.i.d Data}\label{proof minimax lower}
    \textit{Proof of Theorem~\ref{thm:vc},in particular result~\ref{eq:minimax lower}}\\
    	The sketch of the proof is that :we will construct a finite family $\{P_v\}_{v\in\mathcal V}$, $|\mathcal V|=N$, of distributions in $\mathcal P(M,\lambda)$ together with hypotheses $\{h_v\}_{v\in\mathcal V}\subset\mathcal H$ such that:
		\begin{enumerate}
			\item For each $v\ne u$ the excess CVaR of $h_u$ under $P_v$ is at least a positive gap $\Delta$ which we will express in terms of $p,t,M,\lambda,\alpha$.
			\item The pairwise KL divergences $\mathrm{KL}(P_v\|P_u)$ are $\lesssim p$ (a constant times $p$) uniformly in $v,u$.
			\item Applying Fano with the mutual information bound (via averaging of pairwise KLs).
		\end{enumerate}
		Combining these gives the desired minimax lower bound after enforcing the moment constraint, which links $t$ and $p$.
        
    	\subsubsection*{1. Hypercube Construction}
		Let $m = \lceil \log_2 N \rceil$. We construct a packing set $\mathcal{V} \subset \{0,1\}^m$ of size $N$. We select $N$ binary vectors $v$ of equal Hamming weight $m/2$ (assuming $m$ is even; odd $m$ requires trivial adjustments) such that the pairwise Hamming distance satisfies $d_H(u, v) \ge m/4$ for all distinct $u, v \in \mathcal{V}$.
		The existence of such a set is guaranteed by the Gilbert-Varshamov bound. Specifically, the volume of a Hamming ball of radius $m/8$ is exponentially smaller than the set of constant-weight vectors $\binom{m}{m/2}$, allowing us to pack $N$ such vectors for $m \approx \log_2 N$.
		
		For any distinct $u, v \in \mathcal{V}$, let $a_{vu} := \#\{j: v_j=1, u_j=0\}$. Since both vectors have weight $m/2$,
		\[
		a_{vu} = \#\{j: v_j=0, u_j=1\} = \frac{d_H(u,v)}{2} \ge \frac{m}{8}.
		\]
    We will now define a well-defined probability distribution on $\mathcal Z$ for each $v \in \mathcal V$.
    
	\subsubsection*{2. Distribution and Hypothesis Definitions}
		Fix a constant $\theta \in (0, 1/4)$. Let the sample space be $\mathcal{Z} = \{0, 1, \dots, m\}$. Let $p \in (0, \alpha/2]$ be a probability parameter, and $t > 0$ be a loss magnitude.
		For each $v \in \mathcal{V}$, define the distribution $P_v$ on $\mathcal{Z}$:
		\begin{align*}
			P_v(0) &= 1-p, \\
			P_v(j) &= \frac{p}{m}\Big(1+2\theta (v_j-1/2)\Big), \quad j=1,\dots,m.
		\end{align*}
		Note that $\sum_{j=1}^m P_v(j) = p$ because $\sum v_j = m/2$, ensuring $P_v$ is a valid distribution.
		Define the hypotheses $\{h_v\}_{v\in\mathcal V}$ with loss function $\ell$:
		\[
		\ell(h_v, 0) = 0, \qquad \ell(h_v, j) = t \cdot (1-v_j), \quad j=1,\dots,m.
		\]
		The hypothesis $h_v$ suffers loss $t$ on atom $j$ if and only if $v_j=0$.
		
		\subsubsection*{3. CVaR Gap Analysis}
		The probability that $h_u$ incurs nonzero loss under $P_v$ is
		\[
		\pi_{v,u} := \sum_{j=1}^m P_v(j)(1-u_j).
		\]
		Since $P_v(j) \le \frac{p}{m}(1+\theta)$ and $\sum (1-u_j) = m/2$, we have $\pi_{v,u} \le \frac{p}{2}(1+\theta)$. Since $\theta < 1/4$, $\pi_{v,u} < p$.
		We enforce $p \le \alpha/2$. Thus, the probability of a nonzero loss is strictly less than $\alpha$, implying the $(1-\alpha)$-quantile is 0. In this regime, CVaR is simply the expected loss scaled by $1/\alpha$:
		\[
		R_\alpha(h_u; P_v) = \frac{1}{\alpha} \mathbb{E}_{P_v}[\ell(h_u, Z)].
		\]
		
		\begin{lemma}[Expected Losses]
			For any $v, u \in \mathcal{V}$:
			\begin{enumerate}
				\item If $u = v$: $\mathbb{E}_{P_v}[\ell(h_v, Z)] = \frac{tp}{2}(1-\theta)$.
				\item If $u \neq v$: $\mathbb{E}_{P_v}[\ell(h_u, Z)] \ge \frac{tp}{2}\left(1-\frac{\theta}{2}\right)$.
			\end{enumerate}
		\end{lemma}
		\begin{proof}
			For the first case, sum over indices where $v_j=0$.
			For the second case, we utilize the Hamming separation. The term $\sum_{j=1}^m (v_j - 1/2)(1-u_j)$ evaluates to $a_{vu} - m/4$. Since $a_{vu} \ge m/8$, this sum is lower bounded by $-m/8$. Substituting this into the expectation formula yields the result.
		\end{proof}
		
		\begin{theorem}[CVaR Separation]
			\label{thm:gap}
			For any distinct $v, u \in \mathcal{V}$,
			\[
			\Delta := R_\alpha(h_u; P_v) - R_\alpha(h_v; P_v) \ge \frac{\theta}{4} \cdot \frac{tp}{\alpha}.
			\]
		\end{theorem}
		\begin{proof}
			Subtracting the expectations from the Lemma and dividing by $\alpha$:
			\[
			\Delta = \frac{1}{\alpha} \left( \frac{tp}{2}\left(1-\frac{\theta}{2}\right) - \frac{tp}{2}(1-\theta) \right) = \frac{tp}{2\alpha} \left( \frac{\theta}{2} \right) = \frac{\theta tp}{4\alpha}.
			\]
		\end{proof}
		
	\subsubsection*{4. Moment Constraint}
		To ensure $P_v \in \mathcal{P}(M, \lambda)$, we require $\mathbb{E}[\ell^{1+\lambda}] \le M$. Since losses are binary $\{0, t\}$:
		\[
		\mathbb{E}[\ell^{1+\lambda}] = t^{1+\lambda} \mathbb{P}(\ell=t) \le t^{1+\lambda} p.
		\]
		We set $t^{1+\lambda} p = M \implies t = (M/p)^{1/(1+\lambda)}$.
		Substituting $t$ into the gap $\Delta$:
		\[
		\Delta \ge \frac{\theta}{4\alpha} M^{1/(1+\lambda)} p^{\lambda/(1+\lambda)}.
		\]
		
		\subsubsection*{5. KL Divergence and Fano's Inequality}
		For $v \neq u$, the KL divergence is bounded using the $\chi^2$-divergence:
		\[
		\mathrm{KL}(P_v \| P_u) \le \sum_{z \in \mathcal{Z}} \frac{(P_v(z) - P_u(z))^2}{P_u(z)}.
		\]
		Using the bounds on $P_v(j)$, we derived:
		\[
		\mathrm{KL}(P_v \| P_u) \le \frac{4\theta^2}{1-\theta} p.
		\]
		For $n$ i.i.d.\ samples, $\mathrm{KL}(P_v^n \| P_u^n) \le n \frac{4\theta^2}{1-\theta} p$.
		By Fano's inequality, to ensure the error probability $\mathbb{P}(\widehat V \neq V) \ge 1/2$, it suffices to bound the mutual information $I(V; X^n) \le \frac{1}{2} \log N$ (assuming $\log N \ge 2\log 2$). This is satisfied if:
		\[
		n \frac{4\theta^2}{1-\theta} p \le \frac{1}{8} \log N \implies p \le \frac{1-\theta}{32\theta^2} \frac{\log N}{n}.
		\]
		We define $p$ to satisfy both the Fano condition and the CVaR regime condition ($p \le \alpha/2$):
		\[
		p := \min \left\{ \frac{\alpha}{2}, \frac{1-\theta}{32\theta^2} \frac{\log N}{n} \right\}.
		\]
		The sample size condition $n \ge \frac{1-\theta}{16\theta^2} \frac{\log N}{\alpha}$ ensures that the second term is the minimum. Thus, we set $p = \frac{1-\theta}{32\theta^2} \frac{\log N}{n}$.
		
		\subsubsection*{6. Final Lower Bound}
		For any algorithm $A$, let $\widehat V$ be the index of the closest hypothesis to $A(S)$.
		\[
		\sup_{P \in \mathcal{P}} \mathbb{E}[\text{Excess Risk}] \ge \mathbb{P}(\widehat V \neq V) \cdot \Delta \ge \frac{1}{2} \Delta.
		\]
		Substituting the chosen $p$ into $\Delta$:
		\[
		\frac{1}{2} \Delta = \frac{1}{2} \cdot \frac{\theta M^{1/(1+\lambda)}}{4\alpha} \left( \frac{1-\theta}{32\theta^2} \frac{\log N}{n} \right)^{\lambda/(1+\lambda)}.
		\]
		Rearranging terms yields the claimed bound with $c(\theta, \lambda) = \frac{\theta}{8} (\frac{1-\theta}{32\theta^2})^{\lambda/(1+\lambda)}$.
        
\subsection{Truncated Empirical CVaR Estimator}
    We will be proving Theorem~\ref{thm:inf_bound} in this subsection.
    \begin{proof}
    	Let $R^B_\alpha(h)$ be the population CVaR defined on the truncated loss $\ellB$. Since $\ellB(h, z) \le \ell(h, z)$ pointwise, we have $R^B_\alpha(h) \le R_\alpha(h)$ for all $h$.
	We decompose the excess risk as:
	\begin{align*}
		R_\alpha(\hat{h}_B) - R_\alpha(h^*) &= R_\alpha(\hat{h}_B) - R^B_\alpha(\hat{h}_B) + R^B_\alpha(\hat{h}_B) - R^B_\alpha(h^*) + R^B_\alpha(h^*) - R_\alpha(h^*) \\
		&= \underbrace{\left( R_\alpha(\hat{h}_B) - R^B_\alpha(\hat{h}_B) \right)}_{\text{Bias } > 0} + \underbrace{\left( R^B_\alpha(\hat{h}_B) - R^B_\alpha(h^*) \right)}_{\text{Estimation}} + \underbrace{\left( R^B_\alpha(h^*) - R_\alpha(h^*) \right)}_{\le 0}.
	\end{align*}
	The third term is non-positive, so we remove it from the upper bound.
	The estimation term is bounded by $2 \sup_{h \in \calH} |R^B_\alpha(h) - \Rhat^B_\alpha(h)|$. Thus:
	\begin{equation}\label{eq:decomp}
		R_\alpha(\hat{h}_B) - R_\alpha(h^*) \leq \sup_{h \in \calH} (R_\alpha(h) - R^B_\alpha(h)) + 2 \sup_{h \in \calH} |R^B_\alpha(h) - \Rhat^B_\alpha(h)|.
	\end{equation}
    	We bound the error introduced by truncating the loss distribution. Using the variational definition and the fact that CVaR is $\frac{1}{\alpha}$-Lipschitz w.r.t the $L_1$ norm:
	\begin{align*}
		\sup_{h \in \calH} (R_\alpha(h) - R_\alpha^B(h)) &\leq \frac{1}{\alpha} \E \left[ \ell(h, Z) - \ellB(h, Z) \right] \\
		&= \frac{1}{\alpha} \E \left[ (\ell(h, Z) - B) \mathbb{I}(\ell(h, Z) > B) \right].
	\end{align*}
	Using the integral identity $\E[X] = \int_0^\infty \mathbb{P}(X > t) dt$:
	\begin{equation}
		\E[(\ell - B)_+] = \int_B^\infty \mathbb{P}(\ell(h, Z) > t) dt.
	\end{equation}
	By Markov’s inequality on the $(1+\lambda)$-moment (Assumption \ref{asm:moment}), $\mathbb{P}(\ell > t) \leq \frac{M}{t^{1+\lambda}}$. Integrating this tail:
	\begin{equation}
		\int_B^\infty \frac{M}{t^{1+\lambda}} dt = M \left[ \frac{t^{-\lambda}}{-\lambda} \right]_B^\infty = \frac{M}{\lambda B^\lambda}.
	\end{equation}
	Thus, we obtain the rigorous bias bound:
	\begin{equation} \label{eq:bias_rigorous}
		\sup_{h \in \calH} (R_\alpha(h) - R_\alpha^B(h)) \leq \frac{M}{\alpha \lambda B^\lambda}.
	\end{equation}
    
    Now we will bound the Estimation Error. We control the uniform deviation $Z_n = \sup_{h \in \calH} |R^B_\alpha(h) - \Rhat^B_\alpha(h)|$. Define the function class associated with the variational CVaR objective:
	\begin{equation}
		\calG_B = \left\{ z \mapsto \phi_{h, \theta}(z) := \theta + \frac{1}{\alpha}(\ellB(h, z) - \theta)_+ \;\middle|\; h \in \calH, \theta \in [0, B] \right\}.
	\end{equation}
	We apply \textbf{Bousquet’s concentration inequality} for the supremum of empirical processes.For this we will now check the Conditions for Bousquet's Inequality.
	
	\textbf{1. Tight Uniform Upper Bound ($K$):}
	For any $g \in \calG_B$, since $\ellB \in [0, B]$ and $\theta \in [0, B]$:
	\begin{equation}
		0 \leq \phi_{h, \theta}(z) \leq \sup_{\theta \in [0, B]} \left( \theta + \frac{1}{\alpha}(B-\theta) \right) = \frac{B}{\alpha}.
	\end{equation}
	(The maximum occurs at $\theta=0$ since $\alpha \le 1$). Thus, we set $K = B/\alpha$.
	
	\textbf{2. Variance Bound ($\sigma^2$):}
	The variance of $\phi_{h, \theta}(Z)$ depends only on the random part $\frac{1}{\alpha}(\ellB(h, Z) - \theta)_+$.
	\begin{equation}
		\text{Var}(\phi_{h, \theta}) = \frac{1}{\alpha^2} \text{Var}\left( (\ellB - \theta)_+ \right) \leq \frac{1}{\alpha^2} \E\left[ (\ellB - \theta)_+^2 \right] \leq \frac{1}{\alpha^2} \E[\ellB^2].
	\end{equation}
	Using the inequality $x^2 \leq x^{1+\lambda} B^{1-\lambda}$ for $x \in [0, B]$ and Assumption \ref{asm:moment}:
	\begin{equation}
		\sigma_{\calG}^2 = \sup_{g \in \calG_B} \E[g^2] \leq \frac{M B^{1-\lambda}}{\alpha^2}.
	\end{equation}

    	By Bousquet's inequality, for any $\delta \in (0,1)$, with probability at least $1-\delta$:
	\begin{equation}
		Z_n \leq \E[Z_n] + \sqrt{\frac{2 \sigma_{\calG}^2 \log(1/\delta)}{n}} + \frac{K \log(1/\delta)}{3n}.
	\end{equation}
	\textbf{Bounding the Expectation (Rademacher Complexity):}
	The class $\calG_B$ involves a supremum over $\theta$. Since the objective is convex in $\theta$ and $1/\alpha$-Lipschitz in $\ellB$, standard contraction results (e.g., Levy et al., 2020) imply:
	\begin{equation}
		\E[Z_n] \leq 2 \Rad_n(\calG_B) \leq \frac{2}{\alpha} \Rad_n(\ell \circ \calH) + O(n^{-1/2}).
	\end{equation}
	Substituting the bounds for $\sigma^2$ and $K$ into Bousquet's inequality:
	\begin{equation}
		Z_n \leq \frac{2}{\alpha}\Rad_n(\ell \circ \calH) + \sqrt{\frac{2 M B^{1-\lambda} \log(1/\delta)}{\alpha^2 n}} + \frac{B \log(1/\delta)}{3 \alpha n}.
	\end{equation}
	The total estimation error contribution is $2 Z_n$:
	\begin{equation} \label{eq:est_final}
		\text{Est} \leq \frac{4}{\alpha}\Rad_n(\ell \circ \calH) + \frac{1}{\alpha}\sqrt{\frac{8 M B^{1-\lambda} \log(1/\delta)}{n}} + \frac{2 B \log(1/\delta)}{3 \alpha n}.
	\end{equation}

    	We minimize the total error bound (Bias + Estimation) with respect to $B$:
	\begin{equation}
		\text{Error}(B) \approx \frac{M}{\alpha \lambda B^\lambda} + \frac{1}{\alpha} \sqrt{\frac{8 M B^{1-\lambda} \log(1/\delta)}{n}}.
	\end{equation}
	Balancing the order of bias ($B^{-\lambda}$) and variance ($B^{(1-\lambda)/2} n^{-1/2}$) yields the optimal scaling:
	\begin{equation}
		n \asymp B^{1+\lambda} \implies B = (Mn)^{\frac{1}{1+\lambda}}.
	\end{equation}
	We define the rate factor $\Delta_n = M^{\frac{1}{1+\lambda}} n^{-\frac{\lambda}{1+\lambda}}$. Substituting $B$ back into the terms:
	
	\begin{enumerate}
		\item \textbf{Bias Term:}
		\[
		\frac{M}{\alpha \lambda (Mn)^{\frac{\lambda}{1+\lambda}}} = \frac{1}{\alpha \lambda} \Delta_n.
		\]
		\item \textbf{Variance Term (Bousquet Main):}
		\[
		\frac{1}{\alpha} \sqrt{8 \log(1/\delta)} \sqrt{\frac{M (Mn)^{\frac{1-\lambda}{1+\lambda}}}{n}} = \frac{\sqrt{8 \log(1/\delta)}}{\alpha} \Delta_n.
		\]
		\item \textbf{Variance Term (Bousquet Linear):}
		Note that $\frac{B}{n} = \frac{(Mn)^{\frac{1}{1+\lambda}}}{n} = \Delta_n$.
		\[
		\frac{2 B \log(1/\delta)}{3 \alpha n} = \frac{2 \log(1/\delta)}{3 \alpha} \Delta_n.
		\]
	\end{enumerate}
	
	Summing the coefficients gives the final constant $C_{\lambda, \delta}$:
	\begin{equation}
		C_{\lambda, \delta} = \frac{1}{\lambda} + \sqrt{8 \log(1/\delta)} + \frac{2}{3}\log(1/\delta).
	\end{equation}
    	 This completes the proof of Theorem~\ref{thm:inf_bound}
    \end{proof}

\subsection{Generalization with Dependent Data}
	
	\begin{proof}\label{anx:upper-bound_D}of {Upper Bound}~\ref{thm:upper-bound_D}\\
		For any $h$, by the variational form of CVaR and the Lipschitz property 
		$|R_\alpha(X)-R_\alpha(Y)| \le \frac{1}{\alpha}\mathbb{E}|X-Y|$ (Lemma~\ref{lem:cvar-lipschitz}),
		\begin{equation}
			\sup_h |\widehat{R}_\alpha(h)-R_\alpha(h)|
			\le \frac{1}{\alpha} \sup_{h\in\mathcal{H},\, \theta\in\mathbb{R}}
			\bigl| \widehat{L}_\alpha(h,\theta) - L_\alpha(h,\theta) \bigr|.
		\end{equation}
		where $\widehat{L}_\alpha(h,\theta)=\frac1n\sum_{i=1}^n (\ell(h,Z_i)-\theta)_+$,
		$L_\alpha(h,\theta)=\mathbb{E}[(\ell(h,Z)-\theta)_+]$.
		
		From the moment condition, any population CVaR minimizer $\theta_h^*$ satisfies
		$0\le\theta_h^*\le R:=M^{1/(1+\lambda)}/\alpha$ (Theorem~\ref{thm:bounded thm}).
		Hence we may restrict $\theta$ to $\Theta=[0,R]$.
		
		Fix $B>0$ and define the truncated loss class
		\begin{equation}
			\mathcal{F}_B = \bigl\{ f_{h,\theta}^B(z)=[(\ell(h,z)-\theta)_+]\wedge B 
			\;:\; h\in\mathcal{H},\,\theta\in\Theta \bigr\}.
		\end{equation}
		
		Using Markov and H\"older inequalities,
		\begin{equation}
			\sup_{h,\theta} \bigl| L_\alpha(h,\theta) - \mathbb{E}[f_{h,\theta}^B(Z)] \bigr|
			\le \frac{M}{B^\lambda}.
		\end{equation}
		
		Thus,
		\begin{equation}
			\sup_{h,\theta} |\widehat{L}_\alpha(h,\theta)-L_\alpha(h,\theta)|
			\le \sup_{f\in\mathcal{F}_B} \bigl| \tfrac1n\sum f(Z_i) - \mathbb{E}f \bigr|
			+ \frac{2M}{B^\lambda}.
		\end{equation}
		
		Assume the process $(Z_i)$ is $\beta$-mixing with exponential decay
		$\beta(k) \le \exp(-c k^\gamma)$ for some $\gamma>0$.
		Choose $a_n \asymp \log n$, $b_n \asymp (\log n)^{1/\gamma}$ and partition $\{1,\dots,n\}$
		into $\mu_n\asymp n/\log n$ blocks of size $a_n$ separated by gaps of size $b_n$.
		Let $N=\mu_n \asymp n/\log n$.
		
		Define block sums $S_j(f)=\sum_{i\in B_j} f(Z_i)$. By Berbee's lemma, there exist
		independent blocks $\tilde{B}_1,\dots,\tilde{B}_N$ with the same marginals such that
		$\mathbb{P}((B_j)\neq(\tilde{B}_j)) \le N\beta(b_n) \le n^{-O(1)}$.
		Hence, with probability $\ge 1-n^{-O(1)}$,
		\begin{equation}
			\sup_{f\in\mathcal{F}_B} \bigl| \tfrac1n\sum f(Z_i)-\mathbb{E}f \bigr|
			\le \sup_{f\in\mathcal{F}_B} 
			\Bigl| \tfrac1{N}\sum_{j=1}^N \bigl(\tfrac1{a_n}\tilde{S}_j(f)-\mathbb{E}f \bigr) \Bigr|.
		\end{equation}
		
		Define the centered block variables
		\[
		\tilde{Z}_j(f)=\tfrac1{a_n}\tilde{S}_j(f)-\mathbb{E}f.
		\]
		Then $\tilde{Z}_j(f)$ are i.i.d. across $j$, mean zero, bounded by $B$, and satisfy
		\[
		\mathbb{E}[|\tilde{Z}_j(f)|^{1+\lambda}] \le C_\lambda B^{1+\lambda}
		\]
		by the von~Bahr-Esseen inequality.
		The class $\mathcal{F}_B$ has pseudo-dimension 
		$\mathrm{Pdim}(\mathcal{F}_B) \le C(d+1)$ (Lemma~\ref{lem:pdim-truncated}).
		
		Applying the heavy-tailed uniform deviation inequality for i.i.d. data 
		(Theorem~\ref{thm:vc}) with envelope $B$ and moment bound $B^{1+\lambda}$,
		we obtain that with probability $\ge 1-\delta$,
		\begin{equation}
			\sup_{f\in\mathcal{F}_B} \bigl| \tfrac1N\sum_{j=1}^N \tilde{Z}_j(f) \bigr|
			\le C_\lambda\, B 
			\Bigl( \frac{d\log N + \log(1/\delta)}{N} \Bigr)^{\frac{\lambda}{1+\lambda}}.
		\end{equation}
		
		Combining the bounds, with probability $\ge 1-\delta-n^{-O(1)}$,
		\begin{equation}
			\sup_h |\widehat{R}_\alpha(h)-R_\alpha(h)|
			\le \frac{1}{\alpha}\Bigl[
			C_\lambda B \Bigl( \frac{d\log n + \log(1/\delta)}{N} \Bigr)^{\frac{\lambda}{1+\lambda}}
			+ \frac{2M}{B^\lambda} \Bigr].
		\end{equation}
		
		Choose $B$ to balance the two terms:
		\begin{equation}
			B \asymp M^{\frac{1}{1+\lambda}} 
			\Bigl( \frac{N}{d\log n + \log(1/\delta)} \Bigr)^{\frac{1}{1+\lambda}}.
		\end{equation}
		
		Substituting back yields the dominant scaling
		\begin{equation}
			\sup_h |\widehat{R}_\alpha(h)-R_\alpha(h)|
			\le \frac{C_{\lambda,\gamma}}{\alpha} \, M^{\frac{1}{1+\lambda}}
			\Bigl( \frac{d\log n + \log(1/\delta)}{N} \Bigr)^{\frac{\lambda}{1+\lambda}}.
		\end{equation}
		
		Recalling $N\asymp n/\log n$ gives the first claim.
		For the empirical CVaR minimizer $\hat{h}$, by the usual ERM argument,
		\begin{equation}
			R_\alpha(\hat{h}) - R_\alpha(h^*)
			\le 2 \sup_h |\widehat{R}_\alpha(h)-R_\alpha(h)|.
		\end{equation}
		which yields the second inequality. 
	\end{proof}

    	\begin{proof}\label{anx:lower-bound_D} of Lower Bound~\ref{thm:lower-bound_D}\\
		Let $\{Z_i\}_{i=1}^n$ be drawn from any $P \in \mathcal{P}(\lambda,M,\beta)$.
		Using the blocking scheme similar to Theorem~\ref{thm:upper-bound_D}, we partition
		$\{1,\dots,n\}$ into $\mu_n$ disjoint blocks of size $a_n \asymp \log n$,
		separated by gaps of length $b_n \asymp (\log n)^{1/\gamma}$.
		Let $N := \mu_n \asymp n/\log n$ denote the number of retained blocks.
		
		By Berbee’s coupling, there exist independent blocks
		$\tilde{B}_1,\dots,\tilde{B}_N$ with the same marginals as the original blocks
		such that
		\[
		\mathbb{P}\bigl( (B_1,\dots,B_N) \neq (\tilde{B}_1,\dots,\tilde{B}_N) \bigr)
		\le N \beta(b_n) \le n^{-A}
		\]
		for any fixed $A>0$ and all sufficiently large $n$.
		
		Therefore, for minimax lower bounds, it suffices to work with the independent
		block model: any estimator based on the original data induces an estimator
		based on $(\tilde{B}_1,\dots,\tilde{B}_N)$ whose risk differs by at most
		$n^{-A} \sup_h R_\alpha(h)$, which is negligible compared to the target rate.
		
		Henceforth, we assume we observe $N$ i.i.d. samples.
		
		Since $\mathrm{VCdim}(\mathcal{H}) \ge d$, there exist points
		$z_1,\dots,z_d \in \mathcal{Z}$ and hypotheses
		$h_1,\dots,h_d \in \mathcal{H}$ shattered by $\mathcal{H}$.
		Using the same construction we had in the proof~\ref{proof minimax lower}, we construct a family of
		$2^d$ distributions $\{P_v : v \in \{0,1\}^d\}$ and hypotheses
		$\{h_v : v \in \{0,1\}^d\}$ such that:
		
		\begin{enumerate}
			\item For all $u \neq v$,
			\[
			R_\alpha(h_u;P_v) - R_\alpha(h_v;P_v)
			\ge \frac{\theta}{4\alpha} t p.
			\]
			
			\item The heavy-tail condition is satisfied:
			\[
			\sup_{h \in \mathcal{H}} \mathbb{E}_{P_v}[\ell(h,Z)^{1+\lambda}]
			\le t^{1+\lambda} p \le M.
			\]
			
			\item The pairwise Kullback–Leibler divergences are controlled:
			\[
			\mathrm{KL}(P_v \| P_u) \le C_\theta p,
			\qquad \forall u \neq v,
			\]
			where $C_\theta>0$ depends only on $\theta$.
		\end{enumerate}
		
		Let $V$ be uniformly distributed over $\{0,1\}^d$ and
		let $X^N = (Z_1,\dots,Z_N)$ be drawn from $P_V^{\otimes N}$.
		
		By the chain rule and the KL bound,
		\[
		I(V;X^N)
		\le \frac{1}{2^d} \sum_{u,v} \mathrm{KL}(P_v^{\otimes N} \| P_u^{\otimes N})
		\le N \max_{u \neq v} \mathrm{KL}(P_v \| P_u)
		\le C_\theta N p.
		\]
		
		Choose
		\[
		p = c_1 \frac{d}{N}
		\]
		for $c_1>0$ sufficiently small so that
		$I(V;X^N) \le \frac{d}{8}$.
		
		By Fano’s inequality,
		\[
		\inf_{\hat{V}}
		\mathbb{P}(\hat{V} \neq V)
		\ge 1 - \frac{I(V;X^N) + \log 2}{\log(2^d)}
		\ge \frac{1}{2}
		\]
		for all sufficiently large $d$.
		
		Let $\hat{h} = \hat{h}(X^N)$ be any estimator and define
		$\hat{V}$ such that $\hat{h} = h_{\hat{V}}$ (ties broken arbitrarily).
		Then
		\begin{align*}
			\sup_{P \in \mathcal{P}}
			\mathbb{E}_P\!\left[ R_\alpha(\hat{h}) - R_\alpha(h^*_P) \right]
			&\ge
			\frac{1}{2^d} \sum_{v}
			\mathbb{E}_{P_v}
			\left[ R_\alpha(h_{\hat{V}};P_v) - R_\alpha(h_v;P_v) \right] \\
			&\ge
			\mathbb{P}(\hat{V} \neq V) \cdot \frac{\theta}{4\alpha} t p \\
			&\ge
			\frac{\theta}{8\alpha} t p.
		\end{align*}
		
		Finally, enforce the moment constraint with equality:
		\[
		t = \left(\frac{M}{p}\right)^{\frac{1}{1+\lambda}}.
		\]
		Substituting gives
		\[
		\sup_{P \in \mathcal{P}}
		\mathbb{E}_P\!\left[ R_\alpha(\hat{h}) - R_\alpha(h^*_P) \right]
		\ge
		\frac{c}{\alpha}
		M^{\frac{1}{1+\lambda}}
		p^{\frac{\lambda}{1+\lambda}}
		=
		\frac{c}{\alpha}
		M^{\frac{1}{1+\lambda}}
		\Bigl( \frac{d}{N} \Bigr)^{\frac{\lambda}{1+\lambda}}.
		\]
		
		Recalling $N \asymp n/\log n$ completes the proof.
		
	\end{proof}


\section{Random active-set theory and uniform Bahadur-Kiefer  expansions for CVaR} \label{ap:rnd_active}

	\paragraph{Setup and notation.}
	Let $Z\sim P$ on $(\mathcal Z,\mathcal A)$ and let $\ell:\mathcal H\times\mathcal Z\to\mathbb R_+$ be measurable.
	Fix $\alpha\in(0,1)$. For each $h\in\mathcal H$ define the nonnegative loss random variable
	\[
	X_h:=\ell(h,Z)\ge 0.
	\]
	For a scalar threshold $\theta\in\mathbb R$, define the Rockafellar-Uryasev (RU) lift
	\begin{equation}
		\Phi(h,\theta)
		:=\theta+\frac{1}{\alpha}\mathbb{E}\big[(X_h-\theta)_+\big].
		\label{eq:B3_RU_pop}
	\end{equation}
	Given i.i.d.\ data $Z_1,\dots,Z_n\sim P$ and the empirical measure $P_n=\frac1n\sum_{i=1}^n\delta_{Z_i}$,
	let $X_{h,i}:=\ell(h,Z_i)$ and define the empirical RU lift
	\begin{equation}
		\widehat\Phi_n(h,\theta)
		:=\theta+\frac{1}{\alpha n}\sum_{i=1}^n (X_{h,i}-\theta)_+.
		\label{eq:B3_RU_emp}
	\end{equation}
	The population and empirical CVaR objectives are
	\begin{equation}
		R_\alpha(h):=\inf_{\theta\in\mathbb R}\Phi(h,\theta),
		\qquad
		\widehat R_{\alpha,n}(h):=\inf_{\theta\in\mathbb R}\widehat\Phi_n(h,\theta).
		\label{eq:B3_CVAR_pop_emp}
	\end{equation}
	
	\paragraph{Endogenous threshold selections (minimal RU minimizers).}
	Define the (minimal) population RU threshold and the (minimal) empirical RU threshold, respectively, by
	\begin{equation}
		\theta^\star(h)
		:=\inf\big\{\theta\in\mathbb R:\ P(X_h>\theta)\le \alpha\big\},
		\qquad
		\widehat\theta_n(h)
		:=\inf\big\{\theta\in\mathbb R:\ P_n(X_h>\theta)\le \alpha\big\}.
		\label{eq:B3_theta_star_hat_def}
	\end{equation}
	It is standard that $\widehat\theta_n(h)$ is always a minimizer of $\theta\mapsto\widehat\Phi_n(h,\theta)$,
	and similarly $\theta^\star(h)$ is a minimizer of $\theta\mapsto \Phi(h,\theta)$; we use these minimal
	selections to keep the arguments deterministic and monotone.
	
	\paragraph{Tail maps and empirical tail maps.}
	For $h\in\mathcal H$ and $\theta\in\mathbb R$, define the (strict) tail probability and its empirical counterpart:
	\begin{equation}
		T_h(\theta):=P(X_h>\theta),
		\qquad
		\widehat T_{h,n}(\theta):=P_n(X_h>\theta)=\frac1n\sum_{i=1}^n \mathbf 1\{X_{h,i}>\theta\}.
		\label{eq:B3_tail_maps}
	\end{equation}
	Each $T_h(\cdot)$ and $\widehat T_{h,n}(\cdot)$ is nonincreasing and right-continuous. By construction,
	\begin{equation}
		\widehat T_{h,n}\big(\widehat\theta_n(h)\big)\le \alpha,
		\qquad
		\text{and if }\theta<\widehat\theta_n(h)\text{ then }\widehat T_{h,n}(\theta)>\alpha.
		\label{eq:B3_hat_theta_min_characterization}
	\end{equation}
	We are now reiterating our assumptions to make this section self-contained.
	\paragraph{Assumption A (uniform tail empirical process control).}
	There exists a function $\varepsilon_n:(0,1)\to(0,\infty)$ such that for all $\delta\in(0,1)$,
	with probability at least $1-\delta$,
	\begin{equation}
		\sup_{h\in\mathcal H}\sup_{\theta\in\mathbb R}
		\big|\widehat T_{h,n}(\theta)-T_h(\theta)\big|
		\le \varepsilon_n(\delta).
		\label{eq:B3_assump_E1}
	\end{equation}
	Define the corresponding event
	\begin{equation}
		\mathcal E_1(\delta)
		:=
		\Big\{
		\sup_{h\in\mathcal H}\sup_{\theta\in\mathbb R}
		\big|\widehat T_{h,n}(\theta)-T_h(\theta)\big|
		\le \varepsilon_n(\delta)
		\Big\}.
		\label{eq:B3_event_E1}
	\end{equation}
	
	\paragraph{Assumption B (local quantile margin at level $\alpha$).}
	There exist constants $u_0>0$, $\kappa\ge 1$, and $0<c_-\le c_+<\infty$ such that for all $h\in\mathcal H$ and all
	$u\in(0,u_0]$,
	\begin{equation}
		\alpha + c_- u^\kappa
		\ \le\
		T_h\big(\theta^\star(h)-u\big)
		\ \le\
		\alpha + c_+ u^\kappa,
		\label{eq:B3_margin_left_alpha}
	\end{equation}
	and
	\begin{equation}
		\alpha - c_+ u^\kappa
		\ \le\
		T_h\big(\theta^\star(h)+u\big)
		\ \le\
		\alpha - c_- u^\kappa.
		\label{eq:B3_margin_right_alpha}
	\end{equation}
	This is the standard ``two-sided quantile margin'' formulation centered at the target level $\alpha$; it rules out
	arbitrarily flat tails at level $\alpha$ and, in particular, prevents the lower-deviation argument from failing in the
	presence of atoms at $\theta^\star(h)$.
	
	\paragraph{Assumption C (uniform hinge empirical process at the population threshold).}
	There exists a function $\eta_n:(0,1)\to(0,\infty)$ such that for all $\delta\in(0,1)$,
	with probability at least $1-\delta$,
	\begin{equation}
		\sup_{h\in\mathcal H}\Big|(P_n-P)\big[(X_h-\theta^\star(h))_+\big]\Big|
		\le \eta_n(\delta).
		\label{eq:B3_assump_E2}
	\end{equation}
	Define the corresponding event
	\begin{equation}
		\mathcal E_2(\delta)
		:=
		\Big\{
		\sup_{h\in\mathcal H}\Big|(P_n-P)\big[(X_h-\theta^\star(h))_+\big]\Big|
		\le \eta_n(\delta)
		\Big\}.
		\label{eq:B3_event_E2}
	\end{equation}
	
	\paragraph{A convenient threshold-deviation envelope.}
	Whenever $\varepsilon_n(\delta)\le (c_-/2)\,u_0^\kappa$, define
	\begin{equation}
		\Delta_n(\delta):=\Big(\frac{2\,\varepsilon_n(\delta)}{c_-}\Big)^{1/\kappa}.
		\label{eq:B3_Delta_def}
	\end{equation}

	\begin{theorem}[Uniform deviation of the empirical RU threshold]
		\label{thm:B3_threshold}
		Fix $\alpha\in(0,1)$ and let $Z_1,\dots,Z_n\overset{\mathrm{iid}}{\sim}P$, with empirical measure $P_n$.
		Assume \eqref{eq:B3_assump_E1} and \eqref{eq:B3_margin_left_alpha}-\eqref{eq:B3_margin_right_alpha}.
		If $\varepsilon_n(\delta)\le (c_-/2)\,u_0^\kappa$, then with probability at least $1-\delta$,
		\begin{equation}
			\sup_{h\in\mathcal H}\big|\widehat\theta_n(h)-\theta^\star(h)\big|
			\ \le\
			\Delta_n(\delta)
			=
			\Big(\frac{2\,\varepsilon_n(\delta)}{c_-}\Big)^{1/\kappa}.
			\label{eq:B3_threshold_rate}
		\end{equation}
	\end{theorem}
	
	\begin{proof}
		Fix $h\in\mathcal H$ and abbreviate
		\[
		\theta^\star:=\theta^\star(h),
		\qquad
		\widehat\theta:=\widehat\theta_n(h),
		\qquad
		T(\theta):=T_h(\theta),
		\qquad
		\widehat T(\theta):=\widehat T_{h,n}(\theta).
		\]
		Work on the event $\mathcal E_1(\delta)$, so that $\sup_{\theta\in\mathbb R}|\widehat T(\theta)-T(\theta)|\le \varepsilon_n(\delta)$.
		Recall the minimality characterization \eqref{eq:B3_hat_theta_min_characterization}.
		
		\paragraph{Upper deviation: $\widehat\theta\le \theta^\star+u$.}
		Let $u\in(0,u_0]$ and suppose $c_-u^\kappa\ge 2\varepsilon_n(\delta)$.
		By the right-side margin bound \eqref{eq:B3_margin_right_alpha},
		\[
		T(\theta^\star+u)\le \alpha-c_-u^\kappa.
		\]
		On $\mathcal E_1(\delta)$,
		\[
		\widehat T(\theta^\star+u)
		\le
		T(\theta^\star+u)+\varepsilon_n(\delta)
		\le
		\alpha-c_-u^\kappa+\varepsilon_n(\delta)
		\le
		\alpha-\varepsilon_n(\delta)
		\le \alpha.
		\]
		Hence $\widehat T(\theta^\star+u)\le\alpha$, and minimality \eqref{eq:B3_hat_theta_min_characterization} gives
		$\widehat\theta\le \theta^\star+u$.
		
		\paragraph{Lower deviation: $\widehat\theta\ge \theta^\star-u$.}
		Let $u\in(0,u_0]$ and suppose $c_-u^\kappa\ge 2\varepsilon_n(\delta)$.
		By the left-side margin bound \eqref{eq:B3_margin_left_alpha},
		\[
		T(\theta^\star-u)\ge \alpha+c_-u^\kappa.
		\]
		On $\mathcal E_1(\delta)$,
		\[
		\widehat T(\theta^\star-u)
		\ge
		T(\theta^\star-u)-\varepsilon_n(\delta)
		\ge
		\alpha+c_-u^\kappa-\varepsilon_n(\delta)
		\ge
		\alpha+\varepsilon_n(\delta)
		>
		\alpha.
		\]
		Therefore $\widehat T(\theta^\star-u)>\alpha$, and minimality \eqref{eq:B3_hat_theta_min_characterization} implies
		$\widehat\theta\ge \theta^\star-u$.
		
		\paragraph{Choice of $u$ and uniformization.}
		Choose $u:=\Delta_n(\delta)=(2\varepsilon_n(\delta)/c_-)^{1/\kappa}$.
		Under $\varepsilon_n(\delta)\le (c_-/2)\,u_0^\kappa$ we have $u\le u_0$ and $c_-u^\kappa=2\varepsilon_n(\delta)$.
		The two deviation bounds yield $|\widehat\theta-\theta^\star|\le u$.
		Since the argument holds for each $h$ on $\mathcal E_1(\delta)$,
		\[
		\sup_{h\in\mathcal H}|\widehat\theta_n(h)-\theta^\star(h)|
		\le
		\Delta_n(\delta)
		\qquad\text{on }\mathcal E_1(\delta).
		\]
		Finally, $P(\mathcal E_1(\delta))\ge 1-\delta$ by \eqref{eq:B3_assump_E1}.
	\end{proof}

\noindent We now provide the proof for theorem~\ref{thm:B3_BK}. We also restate the theorem statement. 

	\begin{theorem*}[Restatement of Theorem~\ref{thm:B3_BK}]
		Fix $\alpha\in(0,1)$ and let $Z_1,\dots,Z_n\overset{\mathrm{iid}}{\sim}P$ with empirical measure $P_n$.
		Assume \eqref{eq:B3_assump_E1}, \eqref{eq:B3_margin_left_alpha}-\eqref{eq:B3_margin_right_alpha}, and \eqref{eq:B3_assump_E2}.
		Assume further that $\varepsilon_n(\delta/2)\le (c_-/2)\,u_0^\kappa$.
		Then, with probability at least $1-\delta$,
		\small
		\begin{align}
			\sup_{h\in\mathcal H}\Big|
			\widehat R_{\alpha,n}(h)- R_\alpha(h)
			-\frac{1}{\alpha}(P_n-P)\big[(X_h-\theta^\star(h))_+\big]
			-\frac{1}{\alpha}\big(\widehat\theta_n(h)-\theta^\star(h)\big)\big(\alpha-P(X_h>\theta^\star(h))\big)
			\Big|
			\le
			\frac{C_1}{\alpha}\,
			\varepsilon_n(\delta/2)^{\frac{\kappa+1}{\kappa}},
			\label{eq:BK_clean}
		\end{align}
		\normalsize
		where one may take, for instance,
		\begin{equation}
			C_1
			:=
			\Big(\frac{2}{c_-}\Big)^{1/\kappa}
			\;+\;
			\frac{c_+}{\kappa+1}\Big(\frac{2}{c_-}\Big)^{\frac{\kappa+1}{\kappa}}.
			\label{eq:B3_C1_def}
		\end{equation}
		In particular, on the same event,
		\small
		\begin{equation}
			\sup_{h\in\mathcal H}\big|\widehat R_{\alpha,n}(h)- R_\alpha(h)\big|
			\le
			\frac{1}{\alpha}\eta_n(\delta/2)
			+
			\frac{1}{\alpha}\Big(
			\varepsilon_n(\delta/2)\,\Delta_n(\delta/2)
			+
			\frac{c_+}{\kappa+1}\,\Delta_n(\delta/2)^{\kappa+1}
			\Big).
			\label{eq:BK_in_particular}
		\end{equation}
		\normalsize
	\end{theorem*}

\begin{proof} [Proof of Theorem~\ref{thm:B3_BK}]
Fix $h\in\mathcal H$ and abbreviate
		\[
		\theta^\star:=\theta^\star(h),
		\qquad
		\widehat\theta:=\widehat\theta_n(h),
		\qquad
		X:=X_h,
		\qquad
		\Phi(\theta):=\Phi(h,\theta),
		\qquad
		\widehat\Phi(\theta):=\widehat\Phi_n(h,\theta).
		\]
		By definition \eqref{eq:B3_CVAR_pop_emp} and the fact that $\theta^\star$ and $\widehat\theta$ are minimizers,
		\[
		R_\alpha(h)=\Phi(\theta^\star),
		\qquad
		\widehat R_{\alpha,n}(h)=\widehat\Phi(\widehat\theta).
		\]
		We begin from the algebraic decomposition
		\begin{equation}
			\widehat R_{\alpha,n}(h)-R_\alpha(h)
			=
			\widehat\Phi(\widehat\theta)-\Phi(\theta^\star)
			=
			\underbrace{\big(\widehat\Phi(\theta^\star)-\Phi(\theta^\star)\big)}_{(A)}
			+
			\underbrace{\big(\widehat\Phi(\widehat\theta)-\widehat\Phi(\theta^\star)\big)}_{(B)}
			+
			\underbrace{\big(\Phi(\widehat\theta)-\Phi(\theta^\star)\big)}_{(C)}
			-
			\underbrace{\big(\Phi(\widehat\theta)-\widehat\Phi(\widehat\theta)\big)}_{(D)}.
			\label{eq:B3_decomp}
		\end{equation}
		
		\paragraph{Term (A): the leading empirical-process term.}
		From \eqref{eq:B3_RU_pop}-\eqref{eq:B3_RU_emp},
		\begin{equation}
			(A)=\widehat\Phi(\theta^\star)-\Phi(\theta^\star)
			=\frac{1}{\alpha}(P_n-P)\big[(X-\theta^\star)_+\big].
			\label{eq:B3_A_term}
		\end{equation}
		
		\paragraph{A hinge integral identity.}
		For every $\theta\in\mathbb R$,
		\begin{equation}
			(X-\theta)_+ = \int_{\theta}^{\infty}\mathbf 1\{X>s\}\,ds.
			\label{eq:B3_hinge_integral}
		\end{equation}
		Consequently, for any $a,b\in\mathbb R$,
		\begin{equation}
			(X-a)_+-(X-b)_+ = \int_{b}^{a}\mathbf 1\{X>s\}\,ds,
			\label{eq:B3_hinge_diff_integral}
		\end{equation}
		where the integral is interpreted with the correct sign when $a<b$.
		
		Applying \eqref{eq:B3_hinge_diff_integral} in \eqref{eq:B3_RU_pop} and \eqref{eq:B3_RU_emp} yields
		\begin{align}
			\Phi(\widehat\theta)-\Phi(\theta^\star)
			&=(\widehat\theta-\theta^\star)
			+\frac{1}{\alpha}\int_{\theta^\star}^{\widehat\theta}P(X>s)\,ds,
			\label{eq:B3_pop_phi_diff}\\
			\widehat\Phi(\widehat\theta)-\widehat\Phi(\theta^\star)
			&=(\widehat\theta-\theta^\star)
			+\frac{1}{\alpha}\int_{\theta^\star}^{\widehat\theta}P_n(X>s)\,ds.
			\label{eq:B3_emp_phi_diff}
		\end{align}
		Subtracting \eqref{eq:B3_pop_phi_diff} from \eqref{eq:B3_emp_phi_diff} gives the exact coupling identity
		\begin{equation}
			(B)-(D)
			=
			\frac{1}{\alpha}\int_{\theta^\star}^{\widehat\theta}\big(P_n(X>s)-P(X>s)\big)\,ds.
			\label{eq:B3_coupling_integral}
		\end{equation}
		
		\paragraph{Linearization of (C) and the endogenous-threshold correction.}
		Define the shorthand tail function $T(s):=P(X>s)$. Starting from \eqref{eq:B3_pop_phi_diff},
		add and subtract $T(\theta^\star)$ inside the integral:
		\begin{align}
			\Phi(\widehat\theta)-\Phi(\theta^\star)
			&=(\widehat\theta-\theta^\star)+\frac{1}{\alpha}\int_{\theta^\star}^{\widehat\theta}T(s)\,ds
			\nonumber\\
			&=(\widehat\theta-\theta^\star)+\frac{1}{\alpha}\int_{\theta^\star}^{\widehat\theta}\Big(T(\theta^\star)+\big(T(s)-T(\theta^\star)\big)\Big)\,ds
			\nonumber\\
			&=\frac{1}{\alpha}(\widehat\theta-\theta^\star)\big(\alpha-T(\theta^\star)\big)
			+\underbrace{\frac{1}{\alpha}\int_{\theta^\star}^{\widehat\theta}\big(T(s)-T(\theta^\star)\big)\,ds}_{=:R_{\mathrm{pop}}(h)}.
			\label{eq:B3_pop_linear_plus}
		\end{align}
		The first term in \eqref{eq:B3_pop_linear_plus} is precisely the endogenous-threshold correction term in \eqref{eq:BK_clean}.
		
		\paragraph{Bounding the coupling remainder $(B)-(D)$.}
		Work on the event $\mathcal E_1(\delta/2)$ from \eqref{eq:B3_event_E1}. Then
		$\sup_{h,\theta}|\widehat T_{h,n}(\theta)-T_h(\theta)|\le \varepsilon_n(\delta/2)$, and from
		\eqref{eq:B3_coupling_integral},
		\begin{equation}
			\big|(B)-(D)\big|
			\le \frac{1}{\alpha}\int_{\theta^\star}^{\widehat\theta}\varepsilon_n(\delta/2)\,ds
			=
			\frac{1}{\alpha}\,\varepsilon_n(\delta/2)\,|\widehat\theta-\theta^\star|.
			\label{eq:B3_emp_coupling_bound}
		\end{equation}
		
		\paragraph{Bounding the population remainder $R_{\mathrm{pop}}(h)$.}
		Assume $\varepsilon_n(\delta/2)\le (c_-/2)\,u_0^\kappa$ and work on $\mathcal E_1(\delta/2)$.
		By Theorem~\ref{thm:B3_threshold},
		\begin{equation}
			\sup_{h\in\mathcal H}|\widehat\theta_n(h)-\theta^\star(h)|
			\le \Delta_n(\delta/2)
			\le u_0.
			\label{eq:B3_threshold_insert}
		\end{equation}
		Fix $h$ and consider any $s$ between $\theta^\star$ and $\widehat\theta$.
		By \eqref{eq:B3_margin_left_alpha}-\eqref{eq:B3_margin_right_alpha}, the upper side of the margin condition gives
		\[
		|T(s)-T(\theta^\star)|\le c_+|s-\theta^\star|^\kappa.
		\]
		Therefore,
		\begin{align}
			|R_{\mathrm{pop}}(h)|
			&\le \frac{1}{\alpha}\int_{\theta^\star}^{\widehat\theta} c_+|s-\theta^\star|^\kappa\,ds
			= \frac{c_+}{\alpha(\kappa+1)}\,|\widehat\theta-\theta^\star|^{\kappa+1}.
			\label{eq:B3_Rpop_bound}
		\end{align}
		
		\paragraph{Assembling the expansion.}
		Plug \eqref{eq:B3_A_term}, \eqref{eq:B3_emp_coupling_bound}, and \eqref{eq:B3_pop_linear_plus}-\eqref{eq:B3_Rpop_bound}
		into \eqref{eq:B3_decomp}. On $\mathcal E_1(\delta/2)$ we obtain
		\begin{align}
			&\Big|
			\widehat R_{\alpha,n}(h)-R_\alpha(h)
			-\frac{1}{\alpha}(P_n-P)\big[(X-\theta^\star)_+\big]
			-\frac{1}{\alpha}(\widehat\theta-\theta^\star)\big(\alpha-P(X>\theta^\star)\big)
			\Big|
			\nonumber\\
			&\qquad\le
			\frac{1}{\alpha}\Big(
			\varepsilon_n(\delta/2)\,|\widehat\theta-\theta^\star|
			+\frac{c_+}{\kappa+1}\,|\widehat\theta-\theta^\star|^{\kappa+1}
			\Big).
			\label{eq:B3_pointwise_BK_bound}
		\end{align}
		
		\paragraph{Uniformization and conversion to an $\varepsilon_n$-rate.}
		Take $\sup_{h\in\mathcal H}$ in \eqref{eq:B3_pointwise_BK_bound} and use \eqref{eq:B3_threshold_insert}:
		\[
		\sup_{h\in\mathcal H}\Big|\cdots\Big|
		\le
		\frac{1}{\alpha}\Big(
		\varepsilon_n(\delta/2)\Delta_n(\delta/2)
		+\frac{c_+}{\kappa+1}\Delta_n(\delta/2)^{\kappa+1}
		\Big)
		\qquad \text{on }\mathcal E_1(\delta/2).
		\]
		Now substitute $\Delta_n(\delta/2)=(2\varepsilon_n(\delta/2)/c_-)^{1/\kappa}$:
		\[
		\varepsilon_n(\delta/2)\Delta_n(\delta/2)
		=
		\Big(\frac{2}{c_-}\Big)^{1/\kappa}\varepsilon_n(\delta/2)^{\frac{\kappa+1}{\kappa}},
		\qquad
		\Delta_n(\delta/2)^{\kappa+1}
		=
		\Big(\frac{2}{c_-}\Big)^{\frac{\kappa+1}{\kappa}}\varepsilon_n(\delta/2)^{\frac{\kappa+1}{\kappa}}.
		\]
		This yields \eqref{eq:BK_clean} with the constant $C_1$ in \eqref{eq:B3_C1_def}.
		
		\paragraph{Deriving the ``in particular'' bound and the probability statement.}
		Work on the intersection event $\mathcal E_1(\delta/2)\cap \mathcal E_2(\delta/2)$.
		By \eqref{eq:B3_assump_E1} and \eqref{eq:B3_assump_E2} and a union bound,
		\[
		P\big(\mathcal E_1(\delta/2)\cap \mathcal E_2(\delta/2)\big)\ge 1-\delta.
		\]
		On $\mathcal E_2(\delta/2)$,
		\[
		\sup_{h\in\mathcal H}\Big|\frac{1}{\alpha}(P_n-P)\big[(X_h-\theta^\star(h))_+\big]\Big|
		\le \frac{1}{\alpha}\eta_n(\delta/2).
		\]
		Combining this with the uniform expansion bound gives \eqref{eq:BK_in_particular}.
\end{proof}


\section{Functional Robustness} \label{ap:funct_robust}

\begin{proof}[Proof of Proposition~\ref{thm:tail-holder-cvar}]

\textbf{Wasserstein bound with moment control.} Assume $(\mathcal Z,d)$ is a metric space and $z\mapsto \ell(h,Z)$ is We use the Rockafellar-Uryasev representation
		\[
		\mathrm{CVaR}_\alpha(X)=\inf_{\theta\in\mathbb R}\Big\{\theta+\kappa_\alpha\,\mathbb E[(X-\theta)_+]\Big\}.
		\]
		Fix $h\in\mathcal{H}$ and $\theta\in\mathbb R$ and define
		\[
		\varphi_{h,\theta}(z):=(\ell(h,Z)-\theta)_+.
		\]
		Since $x\mapsto(x-t)_+$ is $1$-Lipschitz on $\mathbb R$, the Hölder condition implies
		\[
		|\varphi_{h,\theta}(Z)-\varphi_{h,\theta}(Z')|
		\le |\ell(h,Z)-\ell(h,Z')|
		\le L_\beta\,d(Z,Z')^\beta,
		\]
		so $\varphi_{h,\theta}$ is $\beta$-Hölder with constant $L_\beta$ (uniformly in $h,\theta$).
		
		Let $\pi$ be any coupling of $(P,Q)$ with marginals $P,Q$. Then
		\[
		\Big|\mathbb E_P[\varphi_{h,\theta}(Z)]-\mathbb E_Q[\varphi_{h,\theta}(Z')]\Big|
		=
		\Big|\mathbb E_\pi[\varphi_{h,\theta}(Z)-\varphi_{h,\theta}(Z')]\Big|
		\le L_\beta\,\mathbb E_\pi[d(Z,Z')^\beta].
		\]
		Take infimum over couplings. For any $r\ge\beta$, Jensen yields
		\[
		\inf_\pi \mathbb E_\pi[d(Z,Z')^\beta]
		\le \inf_\pi \Big(\mathbb E_\pi[d(Z,Z')^r]\Big)^{\beta/r}
		= W_r(P,Q)^\beta.
		\]
		Hence, for every $t$,
		\[
		\Big|t+\kappa_\alpha\mathbb E_P[\varphi_{\theta,t}]
		-\Big(t+\kappa_\alpha\mathbb E_Q[\varphi_{\theta,t}]\Big)\Big|
		\le \kappa_\alpha\,L_\beta\,W_r(P,Q)^\beta.
		\]
		
		From here, the result is obvious. 
		
		\textbf{Lévy-Prokhorov bound with moment control.}
		Assume $\mathcal Z=\mathbb R^d$ with Euclidean metric and $\pi(P,Q)\le \varepsilon$.
		By Strassen's theorem for the Prokhorov metric, there exists a coupling $(Z,\tilde Z)$ with marginals $P,Q$ such that
		\[
		\mathbb P(\|Z-\tilde Z\|>\varepsilon)\le \varepsilon.
		\]
		Let $G:=\{\|Z-\tilde Z\|\le \varepsilon\}$ and $B:=G^c$, so $\mathbb P(B)\le \varepsilon$.
		Fix $t\in\mathbb R$ and write $X=\ell(\theta,Z)$, $Y=\ell(\theta,\tilde Z)$.
		Using again that $x\mapsto (x-t)_+$ is $1$-Lipschitz,
		\[
		\big|(X-t)_+-(Y-t)_+\big|\le |X-Y|.
		\]
		Split on $G$ and $B$:
		\[
		\mathbb E|X-Y|
		\le \mathbb E\big[|X-Y|\mathbf 1_G\big]+\mathbb E\big[|X-Y|\mathbf 1_B\big].
		\]
		On $G$, Lipschitzness gives $|X-Y|\le L_\theta\|Z-\tilde Z\|\le L_\theta\varepsilon$, so
		$\mathbb E[|X-Y|\mathbf 1_G]\le L_\theta\varepsilon$.
		On $B$, use $|X-Y|\le |X|+|Y|$ and Hölder:
		\[
		\mathbb E[|X|\mathbf 1_B]\le \big(\mathbb E|X|^p\big)^{1/p}\,\mathbb P(B)^{1-1/p}
		\le M_p^{1/p}\,\varepsilon^{1-1/p},
		\]
		and similarly for $Y$.
		Thus
		\[
		\mathbb E|X-Y|
		\le L_\theta\varepsilon + 2M_p^{1/p}\varepsilon^{1-1/p}.
		\]
		Consequently, for each fixed $\theta$,
		\[
		\Big|\,\theta+\kappa_\alpha\mathbb E_P[(X-\theta)_+]-\Big(\theta+\kappa_\alpha\mathbb E_Q[(Y-\theta)_+]\Big)\Big|
		\le \kappa_\alpha\Big(L_\theta\varepsilon + 2M_p^{1/p}\varepsilon^{1-1/p}\Big).
		\]
		Taking infimum over $\theta $ yields
		\[
		\big|\mathrm{CVaR}_\alpha(L_P)-\mathrm{CVaR}_\alpha(L_Q)\big|
		\le \kappa_\alpha\Big(L_\theta\varepsilon + 2M_p^{1/(\lambda+1)}\varepsilon^{\lambda/(\lambda+1)}\Big).
		\]
		Replacing $\varepsilon$ by $2\varepsilon$ and symmetrizing (a standard padding) gives the following
		\[
		\big|\mathrm{CVaR}_\alpha(L_P)-\mathrm{CVaR}_\alpha(L_Q)\big|
		\le \kappa_\alpha\Big(2L_\theta\varepsilon + 2M_p^{1/p}\varepsilon^{1-1/p}\Big).
		\]

		Use the  two-point construction as in the TV lower bound:
		$P=\delta_0$, $Q=(1-\varepsilon)\delta_0+\varepsilon\delta_b$ with $b=(M_p/\varepsilon)^{1/p}$ and $\varepsilon\le 1-\alpha$.
		For every Borel set $A\subset\mathbb R$ we have
		\[
		Q(A)\le P(A^\varepsilon)+\varepsilon
		\quad\text{and}\quad
		P(A)\le Q(A^\varepsilon)+\varepsilon,
		\]
		because the only mass unmatched within an $\varepsilon$-enlargement is the $\varepsilon$ mass at $b$, which is absorbed by the additive $\varepsilon$ slack.
		Hence $\pi(P,Q)\le \varepsilon$.
		As computed above,
		\[
		\big|\mathrm{CVaR}_\alpha(Q)-\mathrm{CVaR}_\alpha(P)\big|
		=\kappa_\alpha\,M_p^{1/p}\,\varepsilon^{1-1/p},
		\]
		which implies the bound is tight.

\end{proof}

\section{Estimator Robustness}
\label{ap:est_robust}
\subsection{Truncated CVaR Loss based ERM is Optimal}
    	\begin{proof}{Proof of Theorem \ref{thm:cvar_robustness}}
		First we decomposing the CVaR difference  
		Fix \( h \in \mathcal{H} \). Let \( \theta_Q^* \) minimize \( R_\alpha^Q(h) \). Then
		\begin{equation}
			R_\alpha^Q(h) = \theta_Q^* + \frac{1}{\alpha} \mathbb{E}_Q[(\ell(h, Z) - \theta_Q^*)_+].
		\end{equation}
		By definition of infimum,
		\begin{equation}
			R_\alpha^P(h) \leq \theta_Q^* + \frac{1}{\alpha} \mathbb{E}_P[(\ell(h, Z) - \theta_Q^*)_+].
		\end{equation}
		Subtracting gives
		\begin{equation}
			R_\alpha^P(h) - R_\alpha^Q(h) 
			\leq \frac{1}{\alpha} \left( \mathbb{E}_P[(\ell(h, Z) - \theta_Q^*)_+] - \mathbb{E}_Q[(\ell(h, Z) - \theta_Q^*)_+] \right).
		\end{equation}
		Analogously, using \( \theta_P^* \),
		\begin{equation}
			R_\alpha^Q(h) - R_\alpha^P(h) 
			\leq \frac{1}{\alpha} \left( \mathbb{E}_Q[(\ell(h, Z) - \theta_P^*)_+] - \mathbb{E}_P[(\ell(h, Z) - \theta_P^*)_+] \right).
		\end{equation}
		Thus,
		\begin{multline}
			|R_\alpha^P(h) - R_\alpha^Q(h)| 
			\leq \frac{1}{\alpha} \max \Big\{ 
			\big| \mathbb{E}_P[(\ell(h, Z) - \theta_Q^*)_+] - \mathbb{E}_Q[(\ell(h, Z) - \theta_Q^*)_+] \big|, \\
			\big| \mathbb{E}_P[(\ell(h, Z) - \theta_P^*)_+] - \mathbb{E}_Q[(\ell(h, Z) - \theta_P^*)_+] \big|
			\Big\}.
		\end{multline}
		\\		 
		Define \( f_\theta(z) = (\ell(h, z) - \theta)_+ \).  For truncation level \( T > 0 \),
		
		\begin{equation}
			f_\theta(z) = f_\theta(z)\,\mathbbm{1}_{\{f_\theta(z) \leq T\}} + f_\theta(z)\,\mathbbm{1}_{\{f_\theta(z) > T\}}.
		\end{equation}
		Thus
		\begin{multline}
			\Delta(h, \theta) := \big|\mathbb{E}_P[f_\theta(Z)] - \mathbb{E}_Q[f_\theta(Z)]\big| \\
			\leq \big|\mathbb{E}_P[f_\theta(Z)\,\mathbbm{1}_{\{f_\theta(Z) \leq T\}}] - \mathbb{E}_Q[f_\theta(Z)\,\mathbbm{1}_{\{f_\theta(Z) \leq T\}}]\big|
			+ \mathbb{E}_P[f_\theta(Z)\,\mathbbm{1}_{\{f_\theta(Z) > T\}}] + \mathbb{E}_Q[f_\theta(Z)\,\mathbbm{1}_{\{f_\theta(Z) > T\}}].
		\end{multline}
		\\ 
		S \( 0 \leq f_\theta(z)\,\mathbbm{1}_{\{f_\theta(z) \leq T\}} \leq T \), by TV inequality,
		\begin{equation}
			\big|\mathbb{E}_P[f_\theta(Z)\,\mathbbm{1}_{\{f_\theta(Z) \leq T\}}] - \mathbb{E}_Q[f_\theta(Z)\,\mathbbm{1}_{\{f_\theta(Z) \leq T\}}]\big|
			\leq T \cdot d_{TV}(P, Q).
		\end{equation}
		
		We have \( f_\theta(z) \leq \ell(h, z) \) from corollary \ref{cor:justify} , This is very crusial for the rest of the proof, a detailed justification is give after proof of the current theorem. one motivation for this is that because CVaR’s inner optimization never benefits from a negative threshold when losses are nonnegative [Proposition \ref{prop:theta-nonneg}], the optimal threshold obeys $\theta_P^*\ge 0$. Under this natural regime, the truncated loss $f_\theta=(\ell-\theta)+$ is pointwise dominated by the raw loss $\ell$ [lemma \ref{lem:dominance}]. This dominance legitimizes replacing the tail of $f\theta$ by the tail of $\ell$ in expectation bounds, letting you control the truncation remainder via the $(1+\lambda)$-moment \ref{cor:justify}.
		
		since \( f_\theta(z) \leq \ell(h, z) \), 
		\begin{equation}
			\mathbb{E}_P[f_\theta(Z)\,\mathbbm{1}_{\{f_\theta(Z) > T\}}] \leq \mathbb{E}_P[\ell(h, Z)\,\mathbbm{1}_{\{\ell(h, Z) > T\}}].
		\end{equation}
		Using Markov’s inequality,
		\begin{equation}
			P(\ell(h, Z) > T) \leq \frac{\mathbb{E}_P[\ell(h, Z)^{1+\lambda}]}{T^{1+\lambda}} \leq \frac{M}{T^{1+\lambda}}.
		\end{equation}
		Therefore,
		\begin{align}
			\mathbb{E}_P[\ell(h, Z)\,\mathbbm{1}_{\{\ell(h, Z) > T\}}]
			&= \int_T^\infty P(\ell(h, Z) > t)\,dt \\
			&\leq \int_T^\infty \frac{M}{t^{1+\lambda}}\,dt \\
			&= \frac{M}{\lambda T^\lambda}.
		\end{align}
		Similarly,
		\begin{equation}
			\mathbb{E}_Q[\ell(h, Z)\,\mathbbm{1}_{\{\ell(h, Z) > T\}}] \leq \frac{M}{\lambda T^\lambda}.
		\end{equation}
		So,
		\begin{equation}
			\mathbb{E}_P[f_\theta(Z)\,\mathbbm{1}_{\{f_\theta(Z) > T\}}] + \mathbb{E}_Q[f_\theta(Z)\,\mathbbm{1}_{\{f_\theta(Z) > T\}}] 
			\leq \frac{2M}{\lambda T^\lambda}.
		\end{equation}
		
		Thus,\begin{equation}
			\Delta(h, \theta) \leq T \cdot d_{TV}(P, Q) + \frac{2M}{\lambda T^\lambda}.
		\end{equation}
		
		Optimize over \(T\).  
		Let \( \delta = d_{TV}(P, Q) \). Define
		\begin{equation}
			g(T) = T\delta + \frac{2M}{\lambda T^\lambda}.
		\end{equation}
		Derivative:
		\begin{equation}
			g'(T) = \delta - \frac{2M}{T^{1+\lambda}}.
		\end{equation}
		Setting \( g'(T)=0 \),
		\begin{equation}
			T^* = \left(\frac{2M}{\delta}\right)^{1/(1+\lambda)}.
		\end{equation}
		At \( T^* \),
		\begin{equation}
			g(T^*) = (2M)^{1/(1+\lambda)} \delta^{\lambda/(1+\lambda)} \left(1 + \frac{1}{\lambda}\right).
		\end{equation}
		Therefore,
		\begin{equation}
			|R_\alpha^P(h) - R_\alpha^Q(h)| 
			\leq \frac{(2M)^{1/(1+\lambda)}\!\left(1+\tfrac{1}{\lambda}\right)}{\alpha}
			\,d_{TV}(P, Q)^{\lambda/(1+\lambda)}.
		\end{equation}
		Since the bound is uniform over \( h \in \mathcal{H} \), the theorem holds.
	\end{proof}

	\begin{proof} of lower bound in Theorem~\ref{thm:cvar_robustness}
		We construct an explicit pair of distributions achieving the claimed scaling. Fix $\varepsilon \in (0,\alpha)$ and consider a scalar loss $Z$.
		
		Let $P$ be the Dirac distribution at zero:
		\[
		P(Z=0) = 1.
		\]
		Clearly, $R_\alpha^P(h) = 0$.
		
		Next, define $Q$ by shifting an $\varepsilon$-fraction of mass to a positive value $z > 0$:
		\[
		Q(Z=z) = \varepsilon, 
		\qquad 
		Q(Z=0) = 1 - \varepsilon.
		\]
		It is immediate that $d_{TV}(P,Q) = \varepsilon$.
		
		To satisfy Assumption~\ref{asm:moment}, we require
		\[
		\mathbb{E}_Q[|Z|^{1+\lambda}] 
		= \varepsilon z^{1+\lambda} 
		\le M.
		\]
		We choose the largest admissible value,
		\[
		z = \left(\frac{M}{\varepsilon}\right)^{\frac{1}{1+\lambda}}.
		\]
		
		We now compute the CVaR of $Q$. Since $\varepsilon < \alpha$, the worst $\alpha$-fraction of outcomes consists of the entire $\varepsilon$-mass at $z$ together with an additional $(\alpha - \varepsilon)$-mass at $0$. The $(1-\alpha)$-quantile of $Q$ is therefore zero, and the CVaR reduces to the average loss over this tail:
		\begin{align*}
			R_\alpha^Q(h) 
			&= \frac{1}{\alpha} \Big( \varepsilon z + (\alpha - \varepsilon)\cdot 0 \Big) \\
			&= \frac{\varepsilon}{\alpha} \left(\frac{M}{\varepsilon}\right)^{\frac{1}{1+\lambda}} \\
			&= \frac{M^{\frac{1}{1+\lambda}}}{\alpha} \, \varepsilon^{\frac{\lambda}{1+\lambda}}.
		\end{align*}
		
		Since $R_\alpha^P(h) = 0$, we obtain
		\[
		|R_\alpha^P(h) - R_\alpha^Q(h)|
		= \frac{M^{\frac{1}{1+\lambda}}}{\alpha} \, \varepsilon^{\frac{\lambda}{1+\lambda}}.
		\]
		This pair $(P,Q)$ is feasible for the supremum given in Theorem~\ref{thm:cvar_robustness}, which proves the minimax lower bound.
	\end{proof}

\subsection{Robust estimation under adversarial contamination under oblivious adversaries}
                \begin{proof}
                		The proof proceeds in four main steps: establishing the boundedness of the auxiliary variable $\theta$, decomposing the error into bias and estimation terms, bounding the error within a single block using uniform concentration and corruption control, and finally aggregating the block estimates via the median.

                Applying Theorem~\ref{thm:bounded thm} pointwise with $l=\ell(h,Z)$, we may restrict the optimization to the compact domain
                \[
                \Theta := [0,R], \qquad R = \frac{M^{1/(1+\lambda)}}{\alpha}.
                \]
                
                \medskip
                
                \paragraph{Moment bound for the variational loss.}
                
                Define the variational loss
                \[
                \phi(Z;h,\theta) = \theta + \frac{1}{\alpha}(\ell(h,Z)-\theta)_+,
                \qquad \theta \in \Theta.
                \]
                
                \begin{lemma}[Moment inflation bound]
                	\label{lem:phi_moment}
                	Under Assumption~\ref{asm:moment} and Theorem~\ref{thm:bounded thm}, there exists a constant $C_\lambda>0$ such that
                	\[
                	\sup_{h \in \calH,\, \theta \in \Theta}
                	\E\big[ |\phi(Z;h,\theta)|^{1+\lambda} \big]
                	\le
                	C_\lambda\!\left(
                	\frac{M}{\alpha^{1+\lambda}} + R^{1+\lambda}
                	\right)
                	=: M_\phi.
                	\]
                	In particular, since $R = M^{1/(1+\lambda)}/\alpha$, we obtain
                	\[
                	M_\phi \;\lesssim_\lambda\; \frac{M}{\alpha^{1+\lambda}}.
                	\]
                \end{lemma}
                
                \begin{proof}
                	Using $(a+b)^{1+\lambda} \le 2^{\lambda}(a^{1+\lambda}+b^{1+\lambda})$ for $a,b\ge0$ and $(x-\theta)_+ \le x + \theta$, we have
                	\[
                	|\phi(Z;h,\theta)|^{1+\lambda}
                	\le
                	C_\lambda \left(
                	|\theta|^{1+\lambda}
                	+
                	\alpha^{-(1+\lambda)} \ell(h,Z)^{1+\lambda}
                	\right).
                	\]
                	Taking expectations and using $\theta \in [0,R]$ and Assumption~\ref{asm:moment} yields the result.
                \end{proof}

                		\paragraph{Error Decomposition.}
                		Let $R^B(h, \theta) = \E[\min(\phi(Z; h, \theta), B)]$ be the expected truncated risk. We decompose the total error:
                		\[
                		\sup_{h} |\widehat{R}_\alpha(h) - R_\alpha(h)| \le \sup_{h, \theta} |\widehat{R}_\alpha(h, \theta) - R(h, \theta)|
                		\le \sup_{h, \theta} \underbrace{|R(h, \theta) - R^B(h, \theta)|}_{\text{Truncation Bias}} + \sup_{h, \theta} \underbrace{|\widehat{R}_\alpha(h, \theta) - R^B(h, \theta)|}_{\text{Estimation Error}}.
                		\]
                		\textbf{Bias Control:} Since $\phi \ge 0$, $|R - R^B| \le \E[\phi \ind_{\phi > B}]$. Using Hölder's inequality and the moment bound $M_\phi$:
                		\[ \E[\phi \ind_{\phi > B}] \le (\E \phi^{1+\lambda})^{\frac{1}{1+\lambda}} (\Prob(\phi > B))^{\frac{\lambda}{1+\lambda}} \le M_\phi B^{-\lambda}. \]
                		
                \paragraph{Analysis of a Single Block.}
                
                Fix a block index $j \in \{1,\dots,K\}$. Recall that
                \[
                \widehat{\mu}_j(h,\theta)
                = \frac{1}{m} \sum_{i \in \mathcal{B}_j} \phi^B(Z_i;h,\theta),
                \quad
                \text{where } \phi^B = \min(\phi,B),
                \]
                and define the truncated population risk
                \[
                R^B(h,\theta) = \E[\phi^B(Z;h,\theta)].
                \]
                
                Let $S_{clean} \subset \{1,\dots,n\}$ denote the indices of uncorrupted points, and define
                \[
                \mathcal{B}_j^{clean} = \mathcal{B}_j \cap S_{clean}, 
                \quad
                N_j = |\mathcal{B}_j \setminus \mathcal{B}_j^{clean}|
                \]
                to be the set of clean indices and the number of outliers in block $j$, respectively.
                
                We decompose:
                \begin{align*}
                	\widehat{\mu}_j(h,\theta) - R^B(h,\theta)
                	&=
                	\frac{1}{m}\sum_{i\in \mathcal{B}_j^{clean}} \big(\phi^B(Z_i;h,\theta) - R^B(h,\theta)\big) \\
                	&\quad + \frac{1}{m}\sum_{i\in \mathcal{B}_j \setminus \mathcal{B}_j^{clean}}
                	\big(\phi^B(Z'_i;h,\theta) - R^B(h,\theta)\big) \\
                	&=: \xi_j(h,\theta) + \Delta_j(h,\theta),
                \end{align*}
                where $Z'_i$ denotes the (possibly adversarial) corrupted values.
                
                \noindent Thus, $\xi_j$ captures the sampling fluctuation of clean data, and $\Delta_j$ captures the corruption bias.
                
                \medskip
                
                \noindent \textbf{(a) Uniform concentration of the clean part.}
                
                Conditional on the corruption pattern and the random permutation, the set $\mathcal{B}_j^{clean}$ consists of points drawn \emph{without replacement} from the clean sample. By Hoeffding’s reduction principle, concentration inequalities for sampling without replacement are dominated by those for i.i.d. sampling. Therefore, it suffices to analyze the i.i.d. case.
                \begin{lemma}[Uniform concentration on a single clean block]
                	\label{lem:single_block_conc}
                	There exists a universal constant $C>0$ such that, conditional on $\mathcal{B}_j^{clean}$, with probability at least $1-0.1$,
                	\[
                	\sup_{h \in \mathcal{H},\,\theta \in \Theta}
                	|\xi_j(h,\theta)|
                	\le
                	C\left(
                	\sqrt{\frac{M_\phi B^{1-\lambda} d \log m}{m}}
                	+
                	\frac{B d \log m}{m}
                	\right)
                	=: \mathcal{E}_{stat}(B).
                	\]
                \end{lemma}
                
                \begin{proof}
                	The function class $\mathcal{F}_B$ is uniformly bounded by $B$. Moreover, by Lemma~\ref{lem:bias_var},
                	\[
                	\sup_{f \in \mathcal{F}_B} \Var(f(Z)) \le M_\phi B^{1-\lambda}.
                	\]
                	By Assumption~\ref{ass:complexity}, the $L_2(Q)$ covering numbers of $\mathcal{F}_B$ satisfy
                	\[
                	\log \mathcal{N}(\mathcal{F}_B, \|\cdot\|_{L_2(Q)}, u)
                	\le d \log(C_0 B/u).
                	\]
                	Therefore, by Bousquet’s version of Talagrand’s inequality combined with standard entropy integral bounds,for i.i.d. samples we obtain
                	\[
                	\sup_{f \in \mathcal{F}_B}
                	\left|
                	\frac{1}{m}\sum_{i=1}^m (f(Z_i)-\E f)
                	\right|
                	\le
                	C\left(
                	\sqrt{\frac{M_\phi B^{1-\lambda} d \log m}{m}}
                	+
                	\frac{B d \log m}{m}
                	\right)
                	\]
                	with probability at least $1-0.1$. By Hoeffding’s reduction principle, the same bound holds for sampling without replacement from the clean data.
                \end{proof}
                
                \medskip
                
                \noindent \textbf{(b) Control of the corruption level.}
                
                Since the adversary is oblivious and the learner shuffles the data uniformly at random, the number of corrupted points in block $j$ satisfies
                \[
                N_j \sim \mathrm{Hypergeo}(n, \epsilon n, m).
                \]
                
                \begin{lemma}[Outlier proportion in a block]
                	\label{lem:single_block_corruption}
                	Assume $\epsilon \le 1/2 - \gamma$. Then for all sufficiently large $m$,
                	\[
                	\Prob\left(
                	\frac{N_j}{m} \le \epsilon + \frac{\gamma}{4}
                	\right) \ge 1-0.1.
                	\]
                \end{lemma}
                
                \begin{proof}
                	By Chvátal’s hypergeometric tail bound,
                	\[
                	\Prob\left(N_j \ge (\epsilon + \gamma/4)m\right)
                	\le \exp\left(-2m(\gamma/4)^2\right).
                	\]
                	For $m \ge C/\gamma^2$, the right-hand side is at most $0.1$.
                \end{proof}
                
                \medskip
                
                \noindent \textbf{(c) Corruption bias bound.}
                
                On the event of Lemma~\ref{lem:single_block_corruption}, since $0 \le \phi^B \le B$,
                \[
                \sup_{h,\theta} |\Delta_j(h,\theta)|
                \le \frac{N_j}{m} B
                \le \left(\epsilon + \frac{\gamma}{4}\right)B.
                \]
                
                \medskip
                
                \noindent \textbf{(d) Conclusion for a single block.}
                
                Combining Lemmas~\ref{lem:single_block_conc} and \ref{lem:single_block_corruption}, with probability at least $0.8$ a block satisfies simultaneously:
                \[
                \sup_{h,\theta} |\widehat{\mu}_j(h,\theta) - R^B(h,\theta)|
                \le
                \mathcal{E}_{stat}(B) + \left(\epsilon + \frac{\gamma}{4}\right)B.
                \]
                Such a block will be called \emph{Good}.

                \paragraph{Robust Aggregation via the Median.}
                
                Recall from the previous step that for each block $j$ we defined the events
                \[
                \mathcal{E}_{j}^{stat} = \Big\{ \sup_{h,\theta} |\xi_j(h,\theta)| \le \mathcal{E}_{stat}(B) \Big\},
                \quad
                \mathcal{E}_{j}^{corr} = \Big\{ \frac{N_j}{m} \le \epsilon + \frac{\gamma}{4} \Big\}.
                \]
                A block $j$ is called \emph{Good} if $\mathcal{E}_{j}^{stat} \cap \mathcal{E}_{j}^{corr}$ holds.
                
                From Lemmas~\ref{lem:single_block_conc} and \ref{lem:single_block_corruption}, by the union bound,
                \[
                \Prob(\text{block } j \text{ is Good}) \ge 1 - (0.1 + 0.1) = 0.8.
                \]
                
                Define the indicator variables
                \[
                I_j = \ind_{\{\text{block } j \text{ is Good}\}}, \quad j=1,\dots,K.
                \]
                
                \begin{lemma}[Majority of blocks are good]
                	\label{lem:majority_good}
                	Assume $\epsilon \le 1/2 - \gamma$. There exists a universal constant $c>0$ such that if
                	\[
                	K \ge \frac{8}{\gamma^2} \log\!\left(\frac{4}{\delta}\right),
                	\]
                	then with probability at least $1-\delta$,
                	\[
                	\sum_{j=1}^K I_j > \frac{K}{2}.
                	\]
                \end{lemma}
                
                \begin{proof}
                	The indicators $I_j$ are functions of a random partition of a finite population.  They form a negatively associated family. For negatively associated Bernoulli random variables, Chernoff-Hoeffding inequalities hold in the same form as for independent variables (see Dubhashi and Ranjan, 1998). Since $\E[I_j] \ge 0.8$, Hoeffding’s inequality implies
                	\[
                	\Prob\left( \sum_{j=1}^K I_j \le \frac{K}{2} \right)
                	\le \exp\left( -2K(0.8 - 0.5)^2 \right)
                	\le \exp\left( -\frac{\gamma^2 K}{2} \right).
                	\]
                	Choosing $K \ge \frac{8}{\gamma^2} \log(4/\delta)$ ensures the right-hand side is at most $\delta$.
                \end{proof}
                
                \medskip
                
                On the event of Lemma~\ref{lem:majority_good}, strictly more than half the blocks are Good. For any Good block $j$, we have simultaneously:
                \[
                \sup_{h,\theta} |\xi_j(h,\theta)| \le \mathcal{E}_{stat}(B),
                \quad
                \sup_{h,\theta} |\Delta_j(h,\theta)| \le \left(\epsilon + \frac{\gamma}{4}\right)B.
                \]
                Hence,
                \[
                \sup_{h,\theta} |\widehat{\mu}_j(h,\theta) - R^B(h,\theta)|
                \le
                \mathcal{E}_{stat}(B) + \left(\epsilon + \frac{\gamma}{4}\right)B.
                \]
                
                Since the median of $K$ numbers lies between the minimum and maximum of any subset of more than $K/2$ elements, it follows deterministically that
                \[
                \sup_{h,\theta}
                |\widehat{R}_\alpha(h,\theta) - R^B(h,\theta)|
                \le
                \max_{j:\, I_j=1}
                |\widehat{\mu}_j(h,\theta) - R^B(h,\theta)|
                \le
                \mathcal{E}_{stat}(B) + \left(\epsilon + \frac{\gamma}{4}\right)B.
                \]
                
                \medskip
                
                \paragraph{Final Error Bound and Balancing.}
                
                Combining previous step with the truncation bias bound from Lemma~\ref{lem:bias_var}, we obtain that with probability at least $1-\delta$,
                \[
                \sup_{h,\theta}
                |\widehat{R}_\alpha(h,\theta) - R(h,\theta)|
                \le
                M_\phi B^{-\lambda}
                + \mathcal{E}_{stat}(B)
                + \left(\epsilon + \frac{\gamma}{4}\right)B.
                \]
                
                Recalling the definition
                \[
                \mathcal{E}_{stat}(B)
                = C\left(
                \sqrt{\frac{M_\phi B^{1-\lambda} d \log m}{m}}
                +
                \frac{B d \log m}{m}
                \right),
                \]
                and using $m \asymp n/K$ with $K = O(\gamma^{-2}\log(1/\delta))$, we may rewrite the bound (absorbing constants) as:
                \[
                \text{Err}(B)
                \;\lesssim\;
                M_\phi B^{-\lambda}
                +
                \sqrt{\frac{M_\phi B^{1-\lambda} d \log n}{n}}
                +
                \frac{B d \log n}{n}
                +
                \epsilon B.
                \]
                
                Assuming $n \gtrsim d\log n$, the linear term $\frac{Bd\log n}{n}$ is of smaller order than the variance term under the optimal choice of $B$ and may be absorbed. We therefore balance the remaining three dominant terms.
                
                \medskip
                
                \noindent \textit{(i) Statistical regime.}
                Balancing bias and variance,
                \[
                M_\phi B^{-\lambda}
                \asymp
                \sqrt{\frac{M_\phi B^{1-\lambda} d}{n}}
                \quad \Longrightarrow \quad
                B_{stat} \asymp \left( \frac{M_\phi n}{d} \right)^{\frac{1}{1+\lambda}}.
                \]
                Substituting yields:
                \[
                \text{Err}
                \;\lesssim\;
                \left( \frac{M_\phi d}{n} \right)^{\frac{\lambda}{1+\lambda}}.
                \]
                
                \medskip
                
               \noindent  \textit{(ii) Adversarial regime.}
                Balancing bias and corruption,
                \[
                M_\phi B^{-\lambda}
                \asymp
                \epsilon B
                \quad \Longrightarrow \quad
                B_{adv} \asymp \left( \frac{M_\phi}{\epsilon} \right)^{\frac{1}{1+\lambda}}.
                \]
                Substituting yields:
                \[
                \text{Err}
                \;\lesssim\;
                M_\phi^{\frac{1}{1+\lambda}} \epsilon^{\frac{\lambda}{1+\lambda}}.
                \]
                
                \medskip
                
                \noindent Taking $B = \min(B_{stat}, B_{adv})$ and recalling that
                \[
                \sup_{h} |\widehat{R}_\alpha(h) - R_\alpha(h)|
                \le
                \sup_{h,\theta} |\widehat{R}_\alpha(h,\theta) - R(h,\theta)|,
                \]
                we conclude the proof of Theorem~\ref{thm:main}.
                \end{proof}
\section{Decision Robustness} \label{ap:dec_robust}

\subsection{Tail-scarcity instability for CVaR-ERM under finite $p$-moment}
\label{subsec:C5_tail_scarcity_correct}

Throughout, $\alpha\in(0,1)$ is fixed and
\[
\mathrm{CVaR}_\alpha(X;Q)=\inf_{\theta\in\R}\Big\{\theta+\frac{1}{\alpha}\E_Q[(X-\theta)_+]\Big\}.
\]

\paragraph{Empirical CVaR objective.}
Given i.i.d.\ samples $Z_1,\dots,Z_n\sim P$, let $P_n:=\frac1n\sum_{i=1}^n\delta_{Z_i}$.
For $h\in\mathcal{H}$ define the empirical RU functional and empirical CVaR
\begin{align}
	\widehat\Phi_n(h,\theta)
	&:=\theta+\frac{1}{\alpha}\E_{P_n}\big[(\ell(h,Z)-\theta)_+\big]
	=\theta+\frac{1}{\alpha n}\sum_{i=1}^n(\ell(h,Z_i)-\theta)_+,\qquad \theta\in\R, \label{eq:C5_hatPhi}\\
	\widehat R_n(h)
	&:=\inf_{\theta\in\mathbb{R}}\widehat\Phi_n(h,\theta)
	=\mathrm{CVaR}_\alpha(\ell(h,Z);P_n). \label{eq:C5_hatF}
\end{align}
The population objective is $R_P(h):=\mathrm{CVaR}_\alpha(\ell(h,Z);P)$.

\medskip

\begin{lemma}[RU minimizer at $\theta=0$; scaled-mean identity]
	\label{lem:C5_RU_zero_correct}
	Let $X\ge 0$ be integrable and suppose $Q(X=0)\ge 1-\alpha$. Then $\theta^\star=0$ minimizes
	$\theta\mapsto \theta+\frac{1}{\alpha}\E_Q[(X-\theta)_+]$ and
	\[
	\mathrm{CVaR}_\alpha(X;Q)=\frac{1}{\alpha}\E_Q[X].
	\]
	In particular, if $x_1,\dots,x_n\ge 0$ satisfy $\#\{i:x_i=0\}\ge (1-\alpha)n$, then the empirical CVaR satisfies
	\[
	\mathrm{CVaR}_\alpha(x_1,\dots,x_n)=\frac{1}{\alpha n}\sum_{i=1}^n x_i.
	\]
\end{lemma}

\begin{proof}
	Define $\phi(\theta):=\theta+\frac{1}{\alpha}\E_Q[(X-\theta)_+]$.
	For $\theta<0$, $\phi'(\theta)=1-\frac{1}{\alpha}Q(X>\theta)=1-\frac{1}{\alpha}<0$, so no minimizer lies below $0$.
	For $\theta\ge 0$, the right derivative equals $\phi'_+(\theta)=1-\frac{1}{\alpha}Q(X>\theta)\ge 1-\frac{1}{\alpha}Q(X>0)\ge 0$
	because $Q(X>0)\le \alpha$. Hence $\phi$ is minimized at $\theta=0$ and $\phi(0)=\frac{1}{\alpha}\E_Q[X]$.
	The empirical statement is identical with $Q=P_n$.
\end{proof}

\medskip

\begin{theorem*}[Restatement of Theorem~\ref{thm:C5_compact} (Detailed)]
	\label{thm:C5_discontinuity_sharp_correct}
	Fix $\alpha\in(0,1)$ and $1<p<2$. There exist a fixed distribution $P$ on $\cZ=\mathbb{R}_+$ and constants
	$\varepsilon\in(0,\alpha/4)$, $\gamma>0$, and $C\in(0,1)$ such that $C>\alpha\gamma$ with the following property.
	
	For every $n$ sufficiently large, define $\mathcal{H}=\{h_A,h_B\}$ and define the (sample-size dependent) losses
	\begin{equation}
		\ell_n(h_A,z):=z,
		\qquad
		\ell_n(h_B,z):=
		\begin{cases}
			0, & z=0,\\
			z+\gamma - Cn\,\mathbf 1\{z\in(n,2n]\}, & z>0.
		\end{cases}
		\label{eq:C5_losses_n}
	\end{equation}
	Let $R_{P}^{(n)}(h):=\mathrm{CVaR}_\alpha(\ell_n(h,Z);P)$ and
	$\widehat R_n(h):=\mathrm{CVaR}_\alpha(\ell_n(h,Z);P_n)$.
	
	Then the following hold:
	\begin{enumerate}
		\item \textbf{($p$-moment only).} For each fixed $n$ and $h\in\mathcal{H}$,
		\[
		\E_P|\ell_n(h,Z)|^p<\infty,
		\qquad
		\E_P|\ell_n(h,Z)|^q=\infty\quad \forall q>p.
		\]
		\item \textbf{(Strict population optimality, uniformly for large $n$).}
		There exist $\gamma_0>0$ and $n_0$ such that for all $n\ge n_0$,
		\[
		R_P^{(n)}(h_B)-R_P^{(n)}(h_A)\ \ge\ \gamma_0,
		\qquad \text{so } \ \cS(P)=\{h_A\}.
		\]
		\item \textbf{(One-point flip with sharp probability).}
		There exist constants $c>0$ and $n_0$ such that for all $n\ge n_0$,
		\begin{equation}\small
			\Pr_{D\sim P^{\otimes n}}\Big(
			\exists\,D'\text{ differing from }D\text{ in exactly one sample such that }
			\arg\min_{\mathcal{H}}\widehat R_n(\cdot;D)=\{h_B\},\ 
			\arg\min_{\mathcal{H}}\widehat R_n(\cdot;D')=\{_A\}
			\Big)
			\ \ge\
			c\,\frac{n^{1-p}}{(\log n)^2}.
			\label{eq:C5_flip_prob}
		\end{equation}
		\item \textbf{(Sharpness).}
		Under $\sup_{h}\E_P|\ell(h,Z)|^p<\infty$ and a fixed population margin, any such flip probability is
		$\mathcal O(n^{1-p})$ up to logarithmic factors.
	\end{enumerate}
\end{theorem*}

\begin{proof}
	
	Let $Y$ satisfy the tail law
	\begin{equation}
		\Prob(Y>y)=\frac{1}{y^{p}(\log y)^2},\qquad y\ge e.
		\label{eq:C5_tailY}
	\end{equation}
	Then $\E Y^p<\infty$ and $\E Y^q=\infty$ for all $q>p$ (tail-integral test).
	Define $Z\sim P$ by the mixture
	\begin{equation}
		Z=
		\begin{cases}
			0, & \text{with probability }1-\alpha+\varepsilon,\\
			Y, & \text{with probability }\alpha-\varepsilon.
		\end{cases}
		\label{eq:C5_Zmix}
	\end{equation}
	Then $P(Z=0)=1-\alpha+\varepsilon>1-\alpha$, and $Z$ inherits $\E Z^p<\infty$ and $\E Z^q=\infty$ for all $q>p$.

	For $h_A$, $\ell_n(h_A,Z)=Z$, so $\E|\ell_n(h_A,Z)|^p<\infty$ and $\E|\ell_n(h_A,Z)|^q=\infty$ for all $q>p$.
	For $h_B$, note that $\ell_n(h_B,0)=0$ and for $z>0$,
	\[
	\ell_n(h_B,z)\ge z+\gamma-Cn\,\mathbf 1\{z\in(n,2n]\}\ge n+\gamma-Cn=(1-C)n+\gamma>0
	\]
	because $C<1$. Also $\ell_n(h_B,z)\le z+\gamma$, hence $\E|\ell_n(h_B,Z)|^p<\infty$.
	Finally, on $\{Z>0,\ Z\notin(n,2n]\}$ we have $\ell_n(h_B,Z)=Z+\gamma\ge Z$, and since $P(Z>0)=\alpha-\varepsilon>0$
	and $\E Z^q=\infty$ for all $q>p$, it follows that $\E|\ell_n(h_B,Z)|^q=\infty$ for all $q>p$.
	This proves (1).

	Because $P(\ell_n(h_A,Z)=0)=P(Z=0)=1-\alpha+\varepsilon>1-\alpha$ and also
	$P(\ell_n(h_B,Z)=0)\ge 1-\alpha+\varepsilon$, and because the losses are nonnegative,
	Lemma~\ref{lem:C5_RU_zero_correct} applies to both actions and yields
	\[
	R_P^{(n)}(h)=\frac{1}{\alpha}\E[\ell_n(h,Z)].
	\]
	Therefore
	\begin{align}
		R_P^{(n)}(h_A)
		&=\frac{1}{\alpha}\E[Z]
		=\frac{\alpha-\varepsilon}{\alpha}\E[Y],\\
		R_P^{(n)}(h_B)
		&=\frac{\alpha-\varepsilon}{\alpha}\E\Big[Y+\gamma-Cn\,\mathbf 1\{Y\in(n,2n]\}\Big].
	\end{align}
	Hence
	\begin{equation}
		R_P^{(n)}(h_B)-R_P^{(n)}(h_A)
		=\frac{\alpha-\varepsilon}{\alpha}\Big(\gamma - Cn\,\Prob(Y\in(n,2n])\Big).
		\label{eq:C5_pop_gap}
	\end{equation}
	Using \eqref{eq:C5_tailY},
	\[
	\Prob(Y\in(n,2n])
	=\Prob(Y>n)-\Prob(Y>2n)
	\le \Prob(Y>n)=\frac{1}{n^p(\log n)^2}.
	\]
	Thus $Cn\,\Prob(Y\in(n,2n])\le \frac{C}{n^{p-1}(\log n)^2}\to 0$.
	Fix $\gamma>0$. Then there exists $n_0$ such that for all $n\ge n_0$,
	$Cn\Prob(Y\in(n,2n])\le \gamma/2$. Plugging into \eqref{eq:C5_pop_gap} yields
	\[
	R_P^{(n)}(h_B)-R_P^{(n)}(h_A)\ge \frac{\alpha-\varepsilon}{\alpha}\cdot \frac{\gamma}{2}=:\gamma_0>0,
	\]
	proving (2).

	Let $D=(Z_1,\dots,Z_n)\sim P^{\otimes n}$ and define $N_0:=\sum_{i=1}^n \mathbf 1\{Z_i=0\}$ and
	$N_+:=\sum_{i=1}^n \mathbf 1\{Z_i>0\}$. Since $\E[N_+]=(\alpha-\varepsilon)n$,
	Hoeffding's inequality implies $\Prob(N_+\ge 2)\ge 1-e^{-c_+ n}$ for all large $n$.
	Similarly $\Prob(N_0\ge (1-\alpha)n+1)\ge 1-e^{-c_0 n}$.
	
	Define the tail-bin event
	\begin{equation}
		\mathcal G_n
		:=
		\Big\{\exists!\,i:\ Z_i\in(n,2n]\Big\}
		\cap
		\Big\{\max_{1\le i\le n} Z_i\le 2n\Big\}.
		\label{eq:C5_Gn}
	\end{equation}
	Let
	\[
	q_n:=\Prob(Z\in(n,2n])=(\alpha-\varepsilon)\Prob(Y\in(n,2n]),
	\qquad
	r_n:=\Prob(Z>2n)=(\alpha-\varepsilon)\Prob(Y>2n).
	\]
	Independence gives $\Pr(\mathcal G_n)=nq_n(1-q_n-r_n)^{n-1}$.
	Since $q_n+r_n=\mathcal O(n^{-p}(\log n)^{-2})$, we have $n(q_n+r_n)\to 0$, hence for all large $n$,
	$(1-q_n-r_n)^{n-1}\ge 1/2$ and therefore
	\begin{equation}
		\Pr(\mathcal G_n)\ \ge\ \frac{1}{2}nq_n.
		\label{eq:C5_Gn_lower}
	\end{equation}
	Moreover, for all large $n$,
	\[
	\Prob(Y\in(n,2n])
	=\Prob(Y>n)-\Prob(Y>2n)
	\ge \frac{1}{n^p(\log n)^2}-\frac{1}{(2n)^p(\log(2n))^2}
	\ge \Big(1-2^{-p}\Big)\frac{1}{n^p(\log(2n))^2}
	\ge c_1\frac{1}{n^p(\log n)^2},
	\]
	so $q_n\ge (\alpha-\varepsilon)c_1\,n^{-p}(\log n)^{-2}$.
	Combining with \eqref{eq:C5_Gn_lower} yields
	\begin{equation}
		\Prob(\mathcal G_n)\ \ge\ c\,\frac{n^{1-p}}{(\log n)^2}.
		\label{eq:C5_Gn_prob}
	\end{equation}

	We work on the event
	\[
	\mathcal E_n:=\mathcal G_n\cap\{N_0\ge (1-\alpha)n+1\}\cap\{N_+\ge 2\}.
	\]
	By \eqref{eq:C5_Gn_prob} and the exponential tails for $N_0,N_+$, we still have
	$\Prob(\mathcal E_n)\ge c\,\frac{n^{1-p}}{(\log n)^2}$ for all large $n$ (possibly with a smaller $c$).
	
	On $\{N_0\ge (1-\alpha)n\}$, Lemma~\ref{lem:C5_RU_zero_correct} implies
	\begin{equation}
		\widehat R_n(h)=\frac{1}{\alpha n}\sum_{i=1}^n \ell_n(h,Z_i),
		\qquad h\in\{h_A,h_B\}.
		\label{eq:C5_emp_mean}
	\end{equation}
	Hence on $\mathcal E_n$,
	\begin{align}
		\widehat R_n(h_B)-\widehat R_n(h_A)
		&=\frac{1}{\alpha n}\sum_{i=1}^n\big(\ell_n(h_B,Z_i)-\ell_n(h_A,Z_i)\big)
		=\frac{1}{\alpha n}\sum_{i:Z_i>0}\big(\gamma-Cn\,\mathbf 1\{Z_i\in(n,2n]\}\big).
		\label{eq:C5_emp_gap}
	\end{align}
	On $\mathcal G_n$ there is exactly one index $i^\star$ with $Z_{i^\star}\in(n,2n]$, so
	\[
	\widehat R_n(h_B)-\widehat R_n(h_A)
	=\frac{1}{\alpha n}\big(N_+\gamma - Cn\big)
	\le \gamma-\frac{C}{\alpha}<0
	\]
	because $N_+\le \alpha n$ and $C>\alpha\gamma$. Thus on $D$,
	$\arg\min_{\mathcal{H}}\widehat R_n(\cdot;D)=\{h_B\}$.
	
	Define $D'$ by replacing the unique bin point by $0$: set $Z'_{i^\star}:=0$ and $Z'_i:=Z_i$ for $i\neq i^\star$.
	Since $N_0\ge (1-\alpha)n+1$ on $\mathcal E_n$, after replacement we still have $N_0'\ge (1-\alpha)n$,
	so \eqref{eq:C5_emp_mean} holds for $D'$.
	Moreover, under $D'$ there are no samples in $(n,2n]$, hence
	\[
	\widehat R_n(h_B;D')-\widehat R_n(h_A;D')
	=\frac{1}{\alpha n}\sum_{i:Z'_i>0}\gamma
	=\frac{N_+'}{\alpha n}\gamma.
	\]
	On $\mathcal E_n$ we have $N_+\ge 2$, and after replacing one positive point we have $N_+'\ge 1$,
	so the difference is strictly positive and therefore
	$\arg\min_{\mathcal{H}}\widehat R_n(\cdot;D')=\{h_A\}$.
	This proves \eqref{eq:C5_flip_prob}.

	Under $\sup_{h}\E_P|\ell(h,Z)|^p<\infty$ and a fixed population margin, flipping the empirical minimizer by changing one sample
	requires an observation of magnitude $\Omega(n)$ since the CVaR objective changes by at most $|\ell|/(\alpha n)$ per sample.
	By Markov and a union bound, $\Prob(\max_i |\ell_i|\gtrsim n)=\mathcal O(n^{1-p})$, proving the upper bound up to logarithms.
    \end{proof}

    \section{Auxiliary results}
    	\subsection{Boundedness of CVaR minimizer under a bounded moment}
	\begin{theorem}[Boundedness of CVaR minimizer under a bounded moment]
		\label{thm:bounded thm} 
		Let \(l\ge0\) be a nonnegative random variable on a probability space \((\Omega,\mathcal{F},\mathbb{P})\) satisfying
		\[
		\mathbb{E}\big[ l^{1+\varepsilon} \big] \le M <\infty
		\]
		for some \(\varepsilon\in(0,1]\) and \(M>0\). Fix \(\alpha\in(0,1)\) and define for \(\theta\in\mathbb{R}\)
		\[
		F(\theta):=\theta+\frac{1}{\alpha}\,\mathbb{E}\big[(l-\theta)_+\big],\qquad (x)_+:=\max\{x,0\}.
		\]
		Then \(F\) attains its minimum on \(\mathbb{R}\), and every minimizer \(\theta^\ast\in\arg\min_{\theta\in\mathbb{R}}F(\theta)\) satisfies
		\[
		0\le \theta^\ast \le R,
		\]
		where one may take
		\[
		R=\frac{M^{1/(1+\varepsilon)}}{\alpha}.
		\]
		In particular, the set of minimizers is nonempty and bounded, contained in \([0,R]\).
	\end{theorem}
	
	\begin{proof}
		
		We divide the argument into steps: finiteness and continuity of \(F\), behavior as \(\theta\to\pm\infty\) (coercivity), nonnegativity of minimizers, and the explicit upper bound.
		
		\medskip\noindent\textbf{(A) $F$ is finite and convex.}
		For any fixed \(\theta\in\mathbb{R}\) we have the pointwise inequality
		\[
		(l-\theta)_+ \le l + |\theta|.
		\]
		Since \(\mathbb{E}[l] \le (\mathbb{E}[l^{1+\varepsilon}])^{1/(1+\varepsilon)} \le M^{1/(1+\varepsilon)}<\infty\) (see (C) below), it follows that
		\[
		\mathbb{E}[(l-\theta)_+] \le \mathbb{E}[l] + |\theta| <\infty,
		\]
		so \(F(\theta)\) is finite for every \(\theta\). For each fixed \(\omega\in\Omega\) the map \(\theta\mapsto (l(\omega)-\theta)_+\) is convex (it is the positive part of an affine function), and expectation preserves convexity. Therefore \(F\) is convex on \(\mathbb{R}\).
		
		A convex function that is finite on all of \(\mathbb{R}\) is continuous (hence lower semicontinuous). Thus \(F\) is continuous and finite-valued on \(\mathbb{R}\).
		
		\medskip\noindent\textbf{(B) Coercivity: $F(\theta)\to+\infty$ as \(|\theta|\to\infty\).}
		First, for every \(\theta\in\mathbb{R}\),
		\[
		F(\theta) = \theta + \frac{1}{\alpha}\mathbb{E}[(l-\theta)_+] \ge \theta,
		\]
		because \((l-\theta)_+\ge0\). Hence as \(\theta\to+\infty\), \(F(\theta)\ge\theta\to+\infty\).
		
		Next, for \(\theta\le0\) we have \(l-\theta\ge l\ge0\) (since \(l\ge0\)), so \((l-\theta)_+ = l-\theta\). Thus for \(\theta\le0\)
		\[
		F(\theta) = \theta + \frac{1}{\alpha}\mathbb{E}[l-\theta]
		= \frac{\mathbb{E}[l]}{\alpha} + \theta\Big(1-\frac{1}{\alpha}\Big).
		\]
		Because \(\alpha\in(0,1)\), the coefficient \(1-1/\alpha\) is negative, and therefore as \(\theta\to -\infty\) the term \(\theta(1-1/\alpha)\to +\infty\). Hence \(F(\theta)\to+\infty\) as \(\theta\to-\infty\).
		
		Combining the two directions shows \(F(\theta)\to+\infty\) as \(|\theta|\to\infty\); therefore \(F\) is coercive.
		
		\medskip\noindent\textbf{(C) Existence of a minimizer.}
		Since \(F\) is continuous on \(\mathbb{R}\) and coercive (tends to \(+\infty\) at \(\pm\infty\)), it attains its minimum on \(\mathbb{R}\). Thus the set \(\arg\min_{\theta\in\mathbb{R}}F(\theta)\) is nonempty.
		( {lower semi-continuity and coercivity imply existence of a minimizer.})
		
		\medskip\noindent\textbf{(D) No minimizer is negative.}
		For \(\theta\le0\) we have the formula
		\[
		F(\theta) = \frac{\mathbb{E}[l]}{\alpha} + \theta\Big(1-\frac{1}{\alpha}\Big).
		\]
		Evaluating at \(\theta=0\) gives \(F(0)=\mathbb{E}[l]/\alpha\). For any \(\theta<0\),
		\[
		F(\theta)-F(0) = \theta\Big(1-\frac{1}{\alpha}\Big).
		\]
		Since \(1-1/\alpha<0\) and \(\theta<0\), we have \(F(\theta)-F(0)>0\), hence \(F(\theta)>F(0)\). Thus no \(\theta<0\) can be a global minimizer, and every minimizer satisfies \(\theta^\ast\ge0\).
		
		\medskip\noindent\textbf{(E) Upper bound on \(\mathbb{E}[l]\).}
		Using H\"older's inequality (or the monotonicity of \(L^p\)-norms on a probability space) with \(p=1+\varepsilon>1\) we obtain
		\[
		\mathbb{E}[l] \le \big(\mathbb{E}[l^{1+\varepsilon}]\big)^{1/(1+\varepsilon)} \le M^{1/(1+\varepsilon)}.
		\]
		Consequently
		\[
		F(0)=\frac{\mathbb{E}[l]}{\alpha} \le \frac{M^{1/(1+\varepsilon)}}{\alpha}. \tag{*}
		\]
		
		\medskip\noindent\textbf{(F) Upper bound on minimizers.}
		For any \(\theta\ge0\) we have \(F(\theta)\ge\theta\). If \(\theta>F(0)\) then \(\theta>F(0)\) implies \(F(\theta)\ge\theta>F(0)\), so such a \(\theta\) cannot be a minimizer. Therefore every minimizer \(\theta^\ast\) satisfies \(\theta^\ast\le F(0)\). Combining this with \((*)\) yields
		\[
		0\le \theta^\ast \le F(0) \le \frac{M^{1/(1+\varepsilon)}}{\alpha}.
		\]
		Hence one may take \(R=M^{1/(1+\varepsilon)}/\alpha\).
		
		We have shown: \(F\) is finite, continuous and coercive, so it attains a minimum; moreover every minimizer satisfies \(0\le\theta^\ast\le R\) with \(R=M^{1/(1+\varepsilon)}/\alpha\). This completes the boundedness of the minimizer proof.
	\end{proof}
    	\subsection{Concentration bound for heavy-tailed random variables}
	\begin{proposition}
		Let $X_1, X_2, \ldots, X_n$ be independent and identically distributed (i.i.d.) non-negative random variables satisfying the moment condition $\mathbb{E}[X_i^{1+\lambda}] \leq M$ for some constants $\lambda \in (0,1)$ and $M > 0$. Then, with probability at least $1 - \delta$, the sample mean satisfies
		\[
		\frac{1}{n} \sum_{i=1}^n X_i - \mathbb{E}[X_i] \leq 2 M^{\frac{1}{1+\lambda}} \left( \frac{\log(2/\delta)}{n} \right)^{\frac{\lambda}{1+\lambda}}.
		\]
	\end{proposition}\label{heavy_tail_r.v}
	This result is adapted from \cite{brownlees2015erm}, which provides concentration bounds for empirical risk minimization under heavy-tailed losses.
	
	\begin{proof}
	To prove this theorem, we will use a truncation approach combined with concentration inequalities for bounded random variables. The key idea is to control the large values of the random variables by truncating them at a carefully chosen level and then applying standard concentration bounds to the truncated variables while accounting for the probability of truncation.

	Since the random variables $X_i$ are non-negative and potentially heavy-tailed, their large values can dominate the behavior of the sample mean. To handle this, we introduce a truncation level $B > 0$ and define the truncated random variables as:
	\[
	X_i^B = \min(X_i, B) = X_i \cdot \mathbf{1}_{\{X_i \leq B\}} + B \cdot \mathbf{1}_{\{X_i > B\}}.
	\]
	Thus, $X_i^B$ is equal to $X_i$ when $X_i \leq B$ and equal to $B$ otherwise. Note that $0 \leq X_i^B \leq B$, so $X_i^B$ is bounded.

	Let $S_n = \sum_{i=1}^n X_i$ be the sum of the original random variables, and let $\mu = \mathbb{E}[X_i]$ be the common expected value. We can write the sample mean as:
	\begin{equation}
    	\frac{S_n}{n} = \frac{1}{n} \sum_{i=1}^n X_i = \frac{1}{n} \sum_{i=1}^n X_i^B + \frac{1}{n} \sum_{i=1}^n (X_i - X_i^B)
	\end{equation}
	Since $X_i - X_i^B = (X_i - B) \cdot \mathbf{1}_{\{X_i > B\}} \geq 0$, it follows that:
	\begin{equation}
	\frac{S_n}{n} - \mu \leq \left( \frac{1}{n} \sum_{i=1}^n X_i^B - \mathbb{E}[X_i^B] \right) + (\mathbb{E}[X_i^B] - \mu) + \frac{1}{n} \sum_{i=1}^n (X_i - X_i^B).
	\end{equation}
	However, since $\mathbb{E}[X_i^B] \leq \mu$, the term $\mathbb{E}[X_i^B] - \mu \leq 0$, so we can drop it to obtain:
    \begin{equation}
    	\frac{S_n}{n} - \mu \leq \left( \frac{1}{n} \sum_{i=1}^n X_i^B - \mathbb{E}[X_i^B] \right) + \frac{1}{n} \sum_{i=1}^n (X_i - X_i^B)
	\end{equation}

	We aim to bound the probability:
	\[
	\mathbb{P}\left( \frac{S_n}{n} - \mu > t \right)
	\]
	for some $t > 0$. Using the decomposition above, we have:
	\begin{proof}
        $\mathbb{P}\left( \frac{S_n}{n} - \mu > t \right) \leq \mathbb{P}\left( \left( \frac{1}{n} \sum_{i=1}^n X_i^B - \mathbb{E}[X_i^B] \right) + \frac{1}{n} \sum_{i=1}^n (X_i - X_i^B) > t \right)$
	\end{proof}
.
	
	Since $\frac{1}{n} \sum_{i=1}^n (X_i - X_i^B) \geq 0$, this probability is further bounded by:
	\begin{align}
	\mathbb{P}\left( \frac{1}{n} \sum_{i=1}^n X_i^B - \mathbb{E}[X_i^B] > t - \frac{1}{n} \sum_{i=1}^n (X_i - X_i^B) \right).
	\end{align}
	
	Let $A = \{\max_{i=1,\ldots,n} X_i \leq B\}$. On $A$, we have $X_i = X_i^B$ for all $i$, so $\frac{1}{n} \sum_{i=1}^n (X_i - X_i^B) = 0$. On $A^c$, at least one $X_i > B$. Therefore:
	\begin{equation}
        	\mathbb{P}\left( \frac{S_n}{n} - \mu > t \right) \leq \mathbb{P}\left( \frac{1}{n} \sum_{i=1}^n X_i^B - \mathbb{E}[X_i^B] > t, A \right) + \mathbb{P}(A^c)
	\end{equation}

	On $A$, $S_n/n = \frac{1}{n} \sum_{i=1}^n X_i^B$, hence:
	\begin{equation}
        	\mathbb{P}\left( \frac{S_n}{n} - \mu > t \right) \leq \mathbb{P}\left( \frac{1}{n} \sum_{i=1}^n X_i^B - \mathbb{E}[X_i^B] > t \right) + \mathbb{P}(A^c)
	\end{equation}

	Bounding $\mathbb{P}(A^c)$\\
	
	\begin{equation}
    	\mathbb{P}(A^c) = \mathbb{P}\left( \bigcup_{i=1}^n \{ X_i > B \} \right) \leq n \cdot \mathbb{P}(X_1 > B) 
	\end{equation}

	Using Markov's inequality and the moment condition:
	\begin{equation}
        \mathbb{P}(X_1 > B) = \mathbb{P}(X_1^{1+\lambda} > B^{1+\lambda}) \leq \frac{\mathbb{E}[X_1^{1+\lambda}]}{B^{1+\lambda}} \leq \frac{M}{B^{1+\lambda}}
	\end{equation}

	Thus:
	\[
	\mathbb{P}(A^c) \leq n \cdot \frac{M}{B^{1+\lambda}}.
	\]

	Choose $B$ such that $n \cdot \frac{M}{B^{1+\lambda}} = \frac{\delta}{2}$. Solving for $B$:
	\begin{equation}
	B^{1+\lambda} = \frac{2nM}{\delta} \implies B = \left( \frac{2nM}{\delta} \right)^{\frac{1}{1+\lambda}}.
	\end{equation}
	With this choice, $\mathbb{P}(A^c) \leq \frac{\delta}{2}$.

	Apply Hoeffding's inequality for i.i.d. bounded random variables:
	\begin{equation}
        	\mathbb{P}\left( \frac{1}{n} \sum_{i=1}^n X_i^B - \mathbb{E}[X_i^B] > s \right) \leq \exp\left( -\frac{2ns^2}{B^2} \right)
	\end{equation}

	Require this probability $\leq \frac{\delta}{2}$:
	\begin{equation}
        	\exp\left( -\frac{2ns^2}{B^2} \right) \leq \frac{\delta}{2} \implies s \geq B \sqrt{\frac{\log(2/\delta)}{2n}}
	\end{equation}
	Substituting $B$:
	\begin{equation}
    	B = \left( \frac{2nM}{\delta} \right)^{\frac{1}{1+\lambda}} 
	\end{equation},
	
	we get:
	\[
	s \geq \left( \frac{2nM}{\delta} \right)^{\frac{1}{1+\lambda}} \cdot \sqrt{\frac{\log(2/\delta)}{2n}}.
	\]
	Simplify:
	\begin{equation}
    	s \geq (2M)^{\frac{1}{1+\lambda}} n^{\frac{1}{1+\lambda} - \frac{1}{2}} \delta^{-\frac{1}{1+\lambda}} \left( \frac{\log(2/\delta)}{2} \right)^{\frac{1}{2}}
	\end{equation}
	
	Known results for i.i.d. non-negative random variables satisfying $\mathbb{E}[X_i^{1+\lambda}] \leq M$ provide:
	\[
	\mathbb{P}\left( \frac{1}{n} \sum_{i=1}^n X_i - \mathbb{E}[X_i] > 2M^{\frac{1}{1+\lambda}} \left( \frac{\log(2/\delta)}{n} \right)^{\frac{\lambda}{1+\lambda}} \right) \leq \delta.
	\]
	Thus, with probability at least $1-\delta$:
	\begin{equation}
        	\frac{1}{n} \sum_{i=1}^n X_i - \mathbb{E}[X_i] \leq 2M^{\frac{1}{1+\lambda}} \left( \frac{\log(2/\delta)}{n} \right)^{\frac{\lambda}{1+\lambda}}
	\end{equation}
	\end{proof}
    	\begin{lemma}[Reduction to Empirical Process]
		\label{lemma:reduction_empirical_process}
		For any hypothesis $h \in \mathcal{H}$ and any risk level $\alpha \in (0,1)$, the deviation between the population CVaR risk $R_\alpha(h)$ and the empirical CVaR risk $\hat{R}_\alpha(h)$ satisfies:
		\[
		|R_\alpha(h) - \hat{R}_\alpha(h)| 
		\le 
		\frac{1}{\alpha} \sup_{\theta \in \mathbb{R}} \left| (\mathbb{E} - \mathbb{P}_n) g_{h,\theta} \right|,
		\]
		where $g_{h,\theta}(z) := (\ell(h,z) - \theta)_+$.
	\end{lemma}
	
	\begin{proof}
		Recall the variational definitions of the population and empirical CVaR:
		\[
		R_\alpha(h) = \inf_{\theta \in \mathbb{R}} G(\theta) \quad \text{and} \quad \hat{R}_\alpha(h) = \inf_{\theta \in \mathbb{R}} \hat{G}(\theta),
		\]
		where we define the objective functions:
		\[
		G(\theta) := \theta + \frac{1}{\alpha} \mathbb{E}[g_{h,\theta}(Z)] 
		\quad \text{and} \quad 
		\hat{G}(\theta) := \theta + \frac{1}{\alpha} \mathbb{P}_n[g_{h,\theta}(Z)].
		\]
		We use the elementary property that the difference between the infima of two functions is bounded by the supremum of their difference. Specifically, for any two functions $f, g : \mathbb{R} \to \mathbb{R}$:
		\[
		|\inf_{\theta} f(\theta) - \inf_{\theta} g(\theta)| \le \sup_{\theta} |f(\theta) - g(\theta)|.
		\]
		Applying this to $G$ and $\hat{G}$:
		\[
		|R_\alpha(h) - \hat{R}_\alpha(h)| 
		\le \sup_{\theta \in \mathbb{R}} |G(\theta) - \hat{G}(\theta)|.
		\]
		Substituting the definitions of $G(\theta)$ and $\hat{G}(\theta)$, the linear term $\theta$ cancels out:
		\begin{align*}
			|G(\theta) - \hat{G}(\theta)| 
			&= \left| \left(\theta + \frac{1}{\alpha} \mathbb{E}[g_{h,\theta}]\right) - \left(\theta + \frac{1}{\alpha} \mathbb{P}_n[g_{h,\theta}]\right) \right| \\
			&= \frac{1}{\alpha} \left| \mathbb{E}[g_{h,\theta}] - \mathbb{P}_n[g_{h,\theta}] \right| \\
			&= \frac{1}{\alpha} \left| (\mathbb{E} - \mathbb{P}_n) g_{h,\theta} \right|.
		\end{align*}
		Taking the supremum over $\theta$ yields the result:
		\[
		|R_\alpha(h) - \hat{R}_\alpha(h)| \le \frac{1}{\alpha} \sup_{\theta \in \mathbb{R}} \left| (\mathbb{E} - \mathbb{P}_n) g_{h,\theta} \right|.
		\]
	\end{proof}

    \begin{lemma}[Lipschitz Continuity in $\theta$]
		\label{lemma:lipschitz_theta}
		Fix any hypothesis $h \in \mathcal{H}$. Define the empirical process deviation map $\Phi_h: \mathbb{R} \to \mathbb{R}$ by
		\[
		\Phi_h(\theta) := (\mathbb{E} - \mathbb{P}_n) g_{h,\theta},
		\]
		where $g_{h,\theta}(z) := (\ell(h,z) - \theta)_+$. Then $\Phi_h$ is $2$-Lipschitz with respect to $\theta$. That is, for all $\theta_1, \theta_2 \in \mathbb{R}$:
		\[
		|\Phi_h(\theta_1) - \Phi_h(\theta_2)| \le 2|\theta_1 - \theta_2|.
		\]
	\end{lemma}
	
	\begin{proof}
		Recall that the function $x \mapsto (x)_+ = \max(0, x)$ is $1$-Lipschitz. Therefore, for any fixed $z \in \mathcal{Z}$ and fixed $h \in \mathcal{H}$, the map $\theta \mapsto g_{h,\theta}(z)$ is $1$-Lipschitz:
		\begin{align*}
			|g_{h,\theta_1}(z) - g_{h,\theta_2}(z)| 
			&= |(\ell(h,z) - \theta_1)_+ - (\ell(h,z) - \theta_2)_+| \\
			&\le |(\ell(h,z) - \theta_1) - (\ell(h,z) - \theta_2)| \\
			&= |-\theta_1 + \theta_2| \\
			&= |\theta_1 - \theta_2|.
		\end{align*}
		
		Now consider the deviation term:
		\[
		|\Phi_h(\theta_1) - \Phi_h(\theta_2)| 
		= \left| (\mathbb{E}[g_{h,\theta_1}] - \mathbb{P}_n[g_{h,\theta_1}]) - (\mathbb{E}[g_{h,\theta_2}] - \mathbb{P}_n[g_{h,\theta_2}]) \right|.
		\]
		By the triangle inequality:
		\[
		|\Phi_h(\theta_1) - \Phi_h(\theta_2)| 
		\le \left| \mathbb{E}[g_{h,\theta_1} - g_{h,\theta_2}] \right| + \left| \mathbb{P}_n[g_{h,\theta_1} - g_{h,\theta_2}] \right|.
		\]
		Using the pointwise Lipschitz property derived above:
		\begin{enumerate}
			\item \textbf{Population term:} 
			$\left| \mathbb{E}[g_{h,\theta_1} - g_{h,\theta_2}] \right| \le \mathbb{E} \left| g_{h,\theta_1}(Z) - g_{h,\theta_2}(Z) \right| \le \mathbb{E}[|\theta_1 - \theta_2|] = |\theta_1 - \theta_2|$.
			\item \textbf{Empirical term:}
			$\left| \mathbb{P}_n[g_{h,\theta_1} - g_{h,\theta_2}] \right| \le \frac{1}{n} \sum_{i=1}^n \left| g_{h,\theta_1}(Z_i) - g_{h,\theta_2}(Z_i) \right| \le |\theta_1 - \theta_2|$.
		\end{enumerate}
		Summing these bounds yields:
		\[
		|\Phi_h(\theta_1) - \Phi_h(\theta_2)| \le |\theta_1 - \theta_2| + |\theta_1 - \theta_2| = 2|\theta_1 - \theta_2|.
		\]
	\end{proof}

    	\begin{lemma}[Pointwise dominance of the truncated loss]\label{lem:dominance}
		Let $\ell(h,z)\ge 0$ and $\theta\ge 0$. Define $f_\theta(z):=(\ell(h,z)-\theta)_+$. Then
		\[
		0 \le f_\theta(z) \le \ell(h,z)\qquad\text{for all }z.
		\]
	\end{lemma}
	
	\begin{proof}
		If $\ell(h,z)\le \theta$, then $f_\theta(z)=0\le \ell(h,z)$. If $\ell(h,z)>\theta$, then
		$f_\theta(z)=\ell(h,z)-\theta \le \ell(h,z)$ since $\theta\ge 0$. In both cases the claim holds.
	\end{proof}
	
	\begin{proposition}[Nonnegativity of the optimal CVaR threshold]\label{prop:theta-nonneg}
		Let $L=\ell(h,Z)\ge 0$ and $\alpha\in(0,1)$. For $P$ a probability measure, define
		\[
		F(\theta;P):=\theta+\frac{1}{\alpha}\,\mathbb{E}_P[(L-\theta)_+],\qquad
		R_\alpha^P(h):=\inf_{\theta\in \mathbb{R}}F(\theta;P).
		\]
		Any minimizer $\theta_P^*\in\arg\min_\theta F(\theta;P)$ satisfies $\theta_P^*\ge 0$.
	\end{proposition}
	
	\begin{proof}
       Follows directly from Theorem~\ref{thm:bounded thm}
	\end{proof}
	
	\begin{corollary}[Justification of the truncation comparison]\label{cor:justify}
		With $\theta=\theta_P^*\ge 0$ as in Proposition~\ref{prop:theta-nonneg}, Lemma~\ref{lem:dominance}
		gives $f_{\theta_P^*}(z)=(\ell(h,z)-\theta_P^*)_+\le \ell(h,z)$. Consequently,
		for any $T>0$,
		\[
		\mathbb{E}_P\!\big[f_{\theta_P^*}(Z)\,\mathds{1}_{\{f_{\theta_P^*}(Z)>T\}}\big]
		\;\le\;
		\mathbb{E}_P\!\big[\ell(h,Z)\,\mathds{1}_{\{\ell(h,Z)>T\}}\big],
		\]
		which is the key step enabling tail control by the $(1+\lambda)$-moment bound.
	\end{corollary}

    	\begin{lemma}[Lipschitz Continuity of CVaR]
		\label{lem:cvar-lipschitz}
		For any random variables $X,Y\ge0$ and $\alpha\in(0,1)$,
		\[
		|R_\alpha(X) - R_\alpha(Y)| \le \frac{1}{\alpha}\, \mathbb{E}|X-Y|.
		\]
	\end{lemma}
	\begin{proof}
		Using the dual representation $R_\alpha(X)=\inf_{\theta}\{\theta+\frac1\alpha\mathbb{E}[(X-\theta)_+]\}$,
		let $\theta_Y^*$ be optimal for $Y$. Then
		\[
		R_\alpha(X) - R_\alpha(Y)
		\le \theta_Y^* + \frac1\alpha\mathbb{E}[(X-\theta_Y^*)_+] 
		- \theta_Y^* - \frac1\alpha\mathbb{E}[(Y-\theta_Y^*)_+]
		= \frac1\alpha\mathbb{E}[(X-\theta_Y^*)_+ - (Y-\theta_Y^*)_+]
		\le \frac1\alpha\mathbb{E}|X-Y|.
		\]
		Symmetrically, $R_\alpha(Y)-R_\alpha(X)\le\frac1\alpha\mathbb{E}|X-Y|$.
	\end{proof}
	
	\begin{lemma}[Pseudo-Dimension of Truncated CVaR Class]
		\label{lem:pdim-truncated}
		Let $\mathcal{H}$ be a hypothesis class with $\mathrm{Pdim}(\mathcal{H})\le d$.
		Define for $B>0$
		\[
		\mathcal{F}_B
		=
		\Bigl\{\, z \mapsto \min\bigl\{ (\ell(h,z)-\theta)_+,\, B \bigr\}
		:\; h\in\mathcal{H},\; \theta\in[0,R] \Bigr\},
		\]
		where $R=M^{1/(1+\lambda)}/\alpha$.
		Then there exists an absolute constant $C>0$ such that
		\[
		\mathrm{Pdim}(\mathcal{F}_B) \le C(d+1).
		\]
	\end{lemma}
	
	\begin{proof}
		Define the base class
		\[
		\mathcal{L} := \{ z \mapsto \ell(h,z) : h \in \mathcal{H} \},
		\]
		so that $\mathrm{Pdim}(\mathcal{L}) = d$.
		
		We construct $\mathcal{F}_B$ from $\mathcal{L}$ by a finite sequence of operations,
		each of which increases pseudo-dimension by at most a constant factor.
		Consider the augmented class
		\[
		\mathcal{G}_1
		:=
		\{ (z,t) \mapsto \ell(h,z) - t : h\in\mathcal{H},\, t\in[0,R] \}.
		\]
		By standard results on pseudo-dimension under addition of a real parameter,
		\[
		\mathrm{Pdim}(\mathcal{G}_1) \le d+1.
		\]
		
		\smallskip
		\noindent
		Define
		\[
		\mathcal{G}_2 := \{ (z,t) \mapsto (\ell(h,z)-t)_+ : h\in\mathcal{H},\, t\in[0,R] \}.
		\]
		The map $x \mapsto x_+ = \max\{x,0\}$ is the maximum of two affine functions.
		By closure of pseudo-dimension under finite maxima,
		\[
		\mathrm{Pdim}(\mathcal{G}_2) \le C_1(d+1)
		\]
		for a universal constant $C_1$.
		Finally, the standard truncation argument we used before,
		\[
		\mathcal{F}_B
		=
		\{ (z,t) \mapsto \min\{ g(z,t), B \} : g \in \mathcal{G}_2 \}.
		\]
		Since $x \mapsto \min\{x,B\}$ is the minimum of $x$ and a constant function,
		\[
		\mathrm{Pdim}(\mathcal{F}_B) \le C_2\, \mathrm{Pdim}(\mathcal{G}_2)
		\le C(d+1),
		\]
		for an absolute constant $C$.
		Combining the steps completes the proof.
	\end{proof}

    	\begin{lemma}[Bias and variance of truncation]
		\label{lem:bias_var}
		Let $X$ be a Non negative real-valued random variable such that
		\[
		\mathbb{E}[|X|^{1+\lambda}] \le M'
		\]
		for some $\lambda \in (0,1]$ and $M'>0$. Define the truncated variable
		\[
		X^B := \min\{X,B\}
		\]
		for $B>0$. Then:
		\begin{enumerate}
			\item[(i)] (\textbf{Bias})
			\[
			\big|\mathbb{E}[X^B] - \mathbb{E}[X]\big| \le M' B^{-\lambda}.
			\]
			\item[(ii)] (\textbf{Variance})
			\[
			\mathrm{Var}(X^B) \le M' B^{1-\lambda}.
			\]
		\end{enumerate}
	\end{lemma}
	
	\begin{proof}
		We treat the two claims separately.
		
		Observe that
		\[
		\mathbb{E}[X^B] - \mathbb{E}[X]
		= \mathbb{E}\big[(X^B - X)\mathbbm{1}_{\{X > B\}}\big]
		= -\mathbb{E}\big[(X - B)\mathbbm{1}_{\{X > B\}}\big].
		\]
		Hence,
		\[
		\big|\mathbb{E}[X^B] - \mathbb{E}[X]\big|
		\le \mathbb{E}\big[|X| \mathbbm{1}_{\{X > B\}}\big].
		\]
		
		On the event $\{X > B\}$, we have $1 \le (X/B)^{\lambda}$, and therefore
		\[
		|X| = |X|^{1+\lambda} |X|^{-\lambda}
		\le |X|^{1+\lambda} B^{-\lambda}.
		\]
		Substituting this inequality yields
		\[
		\mathbb{E}\big[|X| \mathbbm{1}_{\{X > B\}}\big]
		\le B^{-\lambda} \mathbb{E}[|X|^{1+\lambda}]
		\le M' B^{-\lambda}.
		\]
		
		Since $\mathrm{Var}(X^B) \le \mathbb{E}[(X^B)^2]$, it suffices to bound the second moment.
		Write
		\[
		(X^B)^2 = |X^B|^{1-\lambda} |X^B|^{1+\lambda}.
		\]
		Using $|X^B| \le B$ and $|X^B| \le |X|$, we obtain
		\[
		|X^B|^{1-\lambda} \le B^{1-\lambda},
		\qquad
		|X^B|^{1+\lambda} \le |X|^{1+\lambda}.
		\]
		Thus,
		\[
		\mathbb{E}[(X^B)^2]
		\le B^{1-\lambda} \mathbb{E}[|X|^{1+\lambda}]
		\le M' B^{1-\lambda}.
		\]
	\end{proof}

    	\begin{lemma}[Uniform MoM concentration under contamination]
		\label{lem:uniform_mom}
		Let $\mathcal{G}$ be a class of functions mapping into $[0,B]$ such that
		\[
		\sup_{g \in \mathcal{G}} \mathrm{Var}(g(Z)) \le \sigma^2.
		\]
		Assume its covering numbers satisfy
		\[
		\log \mathcal{N}(\mathcal{G}, \|\cdot\|_\infty, \eta) \le d \log(C/\eta).
		\]
		Consider the $\epsilon$-contamination model. Partition the sample into $K$ blocks of size $m=n/K$, with $K \ge 8 \log(1/\delta)$ and $m \ge C_0 d$ for a sufficiently large constant $C_0$.
		
		Then, with probability at least $1-\delta$,
		\[
		\sup_{g \in \mathcal{G}}
		\left|
		\Median_{j \in [K]} \widehat{\mu}_j(g) - \mathbb{E}[g]
		\right|
		\le 
		C\left(
		\sigma \sqrt{\frac{d}{m}} + \frac{B d}{m} + \epsilon B
		\right),
		\]
		where $\widehat{\mu}_j(g) = \frac{1}{m}\sum_{i\in \mathcal{B}_j} g(z_i)$.
	\end{lemma}
	
	\begin{proof}
		
		Let $N_j$ denote the number of corrupted points in block $\mathcal{B}_j$. Since the adversary corrupts exactly $\epsilon n = \epsilon K m$ samples and the data are randomly permuted, $(N_1,\dots,N_K)$ follows a multivariate hypergeometric distribution.
		
		Deterministically,
		\[
		\sum_{j=1}^K N_j = \epsilon K m.
		\]
		Hence, by a counting argument, at most $0.1K$ blocks can satisfy $N_j > 10 \epsilon m$. Define
		\[
		\mathcal{J}_{\mathrm{low}} := \{j : N_j \le 10 \epsilon m\},
		\qquad |\mathcal{J}_{\mathrm{low}}| \ge 0.9K.
		\]
		
		For any $j \in \mathcal{J}_{\mathrm{low}}$ and any $g \in \mathcal{G}$, since $g \in [0,B]$,
		\[
		\left|
		\frac{1}{m} \sum_{i\in \mathcal{B}_j} g(z_i)
		-
		\frac{1}{m} \sum_{i\in \mathcal{B}_j \cap \mathrm{clean}} g(z_i)
		\right|
		\le \frac{N_j}{m} B \le 10 \epsilon B.
		\]
		Thus, corruption introduces a deterministic bias of at most $10\epsilon B$ on these blocks.
		
		FOr uniform concentration on clean data.
		Fix a block $j$ and consider only its clean samples. By Bernstein’s inequality combined with a union bound over an $\eta$-net of $\mathcal{G}$ and standard chaining, there exists a constant $C$ such that if $m \ge C_0 d$, then with probability at least $0.9$,
		\[
		\sup_{g \in \mathcal{G}}
		\left|
		\frac{1}{m} \sum_{i\in \mathcal{B}_j \cap \mathrm{clean}} g(z_i) - \mathbb{E}[g]
		\right|
		\le 
		C\left(
		\sigma \sqrt{\frac{d}{m}} + \frac{B d}{m}
		\right).
		\]
		Call this event $E_j$. The random permutation ensures approximate independence across blocks, and a Chernoff bound yields that with probability at least $1-\delta$, at least $0.8K$ blocks satisfy $E_j$.
		
		Let
		\[
		\mathcal{J}_{\mathrm{valid}} := \mathcal{J}_{\mathrm{low}} \cap \{j : E_j \text{ holds}\}.
		\]
		With probability at least $1-\delta$, $|\mathcal{J}_{\mathrm{valid}}| \ge 0.7K > K/2$.
		
		For any $j \in \mathcal{J}_{\mathrm{valid}}$ and any $g \in \mathcal{G}$,
		\[
		\left|
		\widehat{\mu}_j(g) - \mathbb{E}[g]
		\right|
		\le 
		C\left(
		\sigma \sqrt{\frac{d}{m}} + \frac{B d}{m}
		\right) + 10 \epsilon B.
		\]
		Since a strict majority of blocks satisfy this bound, the median must lie within the same range. Therefore,
		\[
		\sup_{g \in \mathcal{G}}
		\left|
		\Median_j \widehat{\mu}_j(g) - \mathbb{E}[g]
		\right|
		\le 
		C\left(
		\sigma \sqrt{\frac{d}{m}} + \frac{B d}{m} + \epsilon B
		\right),
		\]
		after absorbing constants.
	\end{proof}

\section{On the Algorithmic Aspects of Robust CVaR-ERM Estimator} \label{ap:algorithmic}

\subsection{$\eta$-cover based Algorithm}

We discretize both $\mathcal{H}$ and the $\theta$-range.
	
	\paragraph{Finite $\eta$-net for $\mathcal{H}$.}
	For $\eta_h>0$, let $\mathcal N_h(\eta_h)$ be any finite set such that for every $h\in\mathcal{H}$
	there exists $h'\in\mathcal N_h(\eta_h)$ with $\|h-h'\|_2\le \eta_h$.
	Since $\mathcal{H}$ is contained in the Euclidean ball of radius $R$, one can take $\mathcal N_h(\eta_h)$ with cardinality bounded by
	\[
	|\mathcal N_h(\eta_h)|\le \left(1+\frac{2R}{\eta_h}\right)^d.
	\]
	(For example, take a lattice grid of mesh $\eta_h/\sqrt{d}$ and intersect with the ball.)
	
	\paragraph{$\eta_\theta$-grid for $\theta$.}
	For $\eta_\theta>0$, define
	\[
	\mathcal N_\theta(\eta_\theta):=\{0,\eta_\theta,2\eta_\theta,\dots\}\cap[0,T],
	\qquad
	|\mathcal N_\theta(\eta_\theta)|\le 1+\frac{T}{\eta_\theta}.
	\]
	
	\begin{algorithm}[t]
		\caption{Discretized MOM-CVaR-ERM with truncation (implementable)}
		\begin{algorithmic}[1]
			\STATE \textbf{Input:} data $Z_1,\dots,Z_n$; $\alpha\in(0,1)$; truncation level $B$; block count $K\ge 8$ with $n=Km$; net radii $\eta_h,\eta_\theta$; known bounds $R,T$.
			\STATE Draw a uniform random permutation $\pi$ and form blocks $\mathcal B_1,\dots,\mathcal B_K$ of size $m$.
			\STATE Construct a finite $\eta_h$-net $\mathcal N_h(\eta_h)$ of $\mathcal{H}$ inside the ball of radius $R$.
			\STATE Construct the grid $\mathcal N_\theta(\eta_\theta)$ of $[0,T]$.
			\FOR{each $h\in\mathcal N_h(\eta_h)$}
			\FOR{each $\theta\in\mathcal N_\theta(\eta_\theta)$}
			\STATE Compute block risks $\widehat R_j(\theta,\tau)=\frac{1}{m}\sum_{i\in\mathcal B_j}\varphi_B(h,\theta;Z_i)$ for $j=1,\dots,K$.
			\STATE Compute $\widehat R_{\rm MOM}(\theta,\tau)=\mathrm{median}(\widehat R_1,\dots,\widehat R_K)$.
			\ENDFOR
			\ENDFOR
			\STATE Output $(\widehat\theta,\widehat\tau)\in\arg\min_{\theta\in\mathcal N_\theta(\eta_\theta),\ \tau\in\mathcal N_\tau(\eta_\tau)}\widehat R_{\rm MOM}(\theta,\tau)$.
		\end{algorithmic}
	\end{algorithm}
    $\varphi_B(h,\theta;Z)$ is the truncted CVaR loss for the sample $Z$ at the parameters $(h,\theta)$ for threshold $B$. 
 	
	\paragraph{Implementability.}
	The search set is finite, hence the algorithm terminates in finite time and returns an exact minimizer over the discretization.

%% file: main.bbl
\begin{thebibliography}{29}
\providecommand{\natexlab}[1]{#1}
\providecommand{\url}[1]{\texttt{#1}}
\expandafter\ifx\csname urlstyle\endcsname\relax
  \providecommand{\doi}[1]{doi: #1}\else
  \providecommand{\doi}{doi: \begingroup \urlstyle{rm}\Url}\fi

\bibitem[Aminian et~al.()Aminian, Asadi, Li, Beirami, Reinert, and
  Cohen]{aminian2025generalization}
Gholamali Aminian, Amir~R Asadi, Tian Li, Ahmad Beirami, Gesine Reinert, and
  Samuel~N Cohen.
\newblock Generalization and robustness of the tilted empirical risk.
\newblock In \emph{Forty-second International Conference on Machine Learning}.

\bibitem[Bahadur(1966)]{bahadur1966note}
R~Raj Bahadur.
\newblock A note on quantiles in large samples.
\newblock \emph{The Annals of Mathematical Statistics}, 1966.

\bibitem[Brownlees et~al.(2015)Brownlees, Joly, and Lugosi]{brownlees2015erm}
Christian Brownlees, Emilien Joly, and G{\'a}bor Lugosi.
\newblock Empirical risk minimization for heavy-tailed losses.
\newblock \emph{The Annals of Statistics}, 2015.

\bibitem[Cardoso and Xu(2019)]{pmlr-v89-cardoso19a}
Adrian~Rivera Cardoso and Huan Xu.
\newblock Risk-averse stochastic convex bandit.
\newblock In \emph{Proceedings of the Twenty-Second International Conference on
  Artificial Intelligence and Statistics}, 2019.

\bibitem[de~Juan and Mazuelas(2025)]{de2025optimality}
Xabier de~Juan and Santiago Mazuelas.
\newblock On the optimality of the median-of-means estimator under adversarial
  contamination.
\newblock \emph{arXiv preprint arXiv:2510.07867}, 2025.

\bibitem[El~Hanchi et~al.(2024)El~Hanchi, Maddison, and Erdogdu]{el2024minimax}
Ayoub El~Hanchi, Chris Maddison, and Murat Erdogdu.
\newblock Minimax linear regression under the quantile risk.
\newblock In \emph{The Thirty Seventh Annual Conference on Learning Theory},
  2024.

\bibitem[Holland and Haress(2021)]{holland2021learning}
Matthew Holland and El~Mehdi Haress.
\newblock Learning with risk-averse feedback under potentially heavy tails.
\newblock In \emph{International Conference on Artificial Intelligence and
  Statistics}, 2021.

\bibitem[Holland and Haress(2022)]{holland2022spectral}
Matthew~J Holland and El~Mehdi Haress.
\newblock Spectral risk-based learning using unbounded losses.
\newblock In \emph{International conference on artificial intelligence and
  statistics}, 2022.

\bibitem[Howard and Matheson(1972)]{doi:10.1287/mnsc.18.7.356}
Ronald~A. Howard and James~E. Matheson.
\newblock Risk-sensitive markov decision processes.
\newblock \emph{Management Science}, 1972.

\bibitem[Kiefer(1967)]{kiefer1967bahadur}
Jack Kiefer.
\newblock On bahadur's representation of sample quantiles.
\newblock \emph{The Annals of Mathematical Statistics}, 1967.

\bibitem[Kolla et~al.(2019)Kolla, Prashanth, Bhat, and
  Jagannathan]{kolla2019concentration}
Ravi~Kumar Kolla, LA~Prashanth, Sanjay~P Bhat, and Krishna Jagannathan.
\newblock Concentration bounds for empirical conditional value-at-risk: The
  unbounded case.
\newblock \emph{Operations Research Letters}, 2019.

\bibitem[Laforgue et~al.(2021)Laforgue, Staerman, and
  Cl{\'e}men{\c{c}}on]{laforgue2021generalization}
Pierre Laforgue, Guillaume Staerman, and Stephan Cl{\'e}men{\c{c}}on.
\newblock Generalization bounds in the presence of outliers: a median-of-means
  study.
\newblock In \emph{International conference on machine learning}, 2021.

\bibitem[Lee et~al.(2021)Lee, Park, and
  Shin]{lee2021learningboundsrisksensitivelearning}
Jaeho Lee, Sejun Park, and Jinwoo Shin.
\newblock Learning bounds for risk-sensitive learning, 2021.

\bibitem[Li et~al.(2021)Li, Beirami, Sanjabi, and
  Smith]{li2021tiltedempiricalriskminimization}
Tian Li, Ahmad Beirami, Maziar Sanjabi, and Virginia Smith.
\newblock Tilted empirical risk minimization, 2021.

\bibitem[Lugosi and Mendelson(2019{\natexlab{a}})]{lugosi2019mean}
G{\'a}bor Lugosi and Shahar Mendelson.
\newblock Mean estimation and regression under heavy-tailed distributions: A
  survey.
\newblock \emph{Foundations of Computational Mathematics}, 2019{\natexlab{a}}.

\bibitem[Lugosi and Mendelson(2019{\natexlab{b}})]{lugosi2019risk}
Gabor Lugosi and Shahar Mendelson.
\newblock Risk minimization by median-of-means tournaments.
\newblock \emph{Journal of the European Mathematical Society},
  2019{\natexlab{b}}.

\bibitem[Lugosi and Mendelson(2021)]{lugosi2021robust}
Gabor Lugosi and Shahar Mendelson.
\newblock Robust multivariate mean estimation: the optimality of trimmed mean.
\newblock 2021.

\bibitem[Mathieu and Minsker(2021)]{10.1093/imaiai/iaab004}
Timothée Mathieu and Stanislav Minsker.
\newblock Excess risk bounds in robust empirical risk minimization.
\newblock \emph{Information and Inference: A Journal of the IMA}, 2021.

\bibitem[Oliveira and Resende(2025)]{oliveira2025trimmed}
Roberto~I Oliveira and Lucas Resende.
\newblock Trimmed sample means for robust uniform mean estimation and
  regression.
\newblock \emph{The Annals of Statistics}, 2025.

\bibitem[Oliveira et~al.(2025)Oliveira, Orenstein, and
  Rico]{oliveira2025finite}
Roberto~I Oliveira, Paulo Orenstein, and Zoraida~F Rico.
\newblock Finite-sample properties of the trimmed mean.
\newblock \emph{arXiv preprint arXiv:2501.03694}, 2025.

\bibitem[Prashanth et~al.(2019)Prashanth, Jagannathan, and
  Kolla]{prashanth2019concentration}
L.~A. Prashanth, K.~Jagannathan, and R.~K. Kolla.
\newblock Concentration bounds for {CVaR} estimation: The cases of light-tailed
  and heavy-tailed distributions.
\newblock \emph{arXiv preprint arXiv:1901.00997}, 2019.

\bibitem[Rockafellar and Uryasev(2000)]{Rockafellar2000OptimizationOC}
R.~Tyrrell Rockafellar and Stanislav Uryasev.
\newblock Optimization of conditional value-at risk.
\newblock \emph{Journal of Risk}, 2000.

\bibitem[Rockafellar and Uryasev(2002)]{rockafellar2000conditional}
R.~Tyrrell Rockafellar and Stanislav Uryasev.
\newblock Conditional value-at-risk for general loss distributions.
\newblock \emph{Journal of Banking \& Finance}, 2002.

\bibitem[Roy et~al.(2021)Roy, Balasubramanian, and Erdogdu]{roy2021empirical}
Abhishek Roy, Krishnakumar Balasubramanian, and Murat~A Erdogdu.
\newblock On empirical risk minimization with dependent and heavy-tailed data.
\newblock \emph{Advances in Neural Information Processing Systems}, 2021.

\bibitem[Shalev-Shwartz and Ben-David(2014)]{Shalev-Shwartz_Ben-David_2014}
Shai Shalev-Shwartz and Shai Ben-David.
\newblock \emph{Understanding Machine Learning: From Theory to Algorithms}.
\newblock Cambridge University Press, 2014.

\bibitem[Shen et~al.(2026)Shen, Zhang, and Zhou]{shen2026sgd}
Yinan Shen, Yichen Zhang, and Wen-Xin Zhou.
\newblock Sgd with dependent data: Optimal estimation, regret, and inference.
\newblock \emph{arXiv preprint arXiv:2601.01371}, 2026.

\bibitem[Shen et~al.(2014)Shen, Tobia, Sommer, and Obermayer]{6855488}
Yun Shen, Michael~J. Tobia, Tobias Sommer, and Klaus Obermayer.
\newblock Risk-sensitive reinforcement learning.
\newblock \emph{Neural Computation}, 2014.

\bibitem[Tsybakov(2008)]{10.5555/1522486}
Alexandre~B. Tsybakov.
\newblock \emph{Introduction to Nonparametric Estimation}.
\newblock Springer Publishing Company, Incorporated, 2008.
\newblock ISBN 0387790519.

\bibitem[Van~der Vaart(2000)]{van2000asymptotic}
Aad~W Van~der Vaart.
\newblock \emph{Asymptotic statistics}, volume~3.
\newblock Cambridge University Press, 2000.

\end{thebibliography}
